\documentclass{article} %
\usepackage[final]{colm2025_conference}

\usepackage{microtype}
\usepackage{hyperref}
\usepackage{url}
\usepackage{booktabs}
\usepackage{xcolor}
\usepackage{svg}
\usepackage{subcaption}
\usepackage{multirow}    %
\usepackage{array}       %
\usepackage{amsmath}     %
\usepackage{enumitem}
\usepackage{colortbl}
\usepackage{booktabs}
\usepackage{amsmath}
\usepackage{amsthm}
\usepackage{amssymb}
\usepackage{algorithm}
\usepackage{algorithmic}
\usepackage{mdframed}
\usepackage{tcolorbox}
\usepackage[font=small,labelfont=bf]{caption}
\usepackage{xspace}

\usepackage{wrapfig}

\newtheorem{theorem}{Theorem}

\newtheorem{definition}{Definition}

\usepackage{lineno}

\definecolor{darkblue}{rgb}{0, 0, 0.5}
\hypersetup{colorlinks=true, citecolor=darkblue, linkcolor=darkblue, urlcolor=darkblue}

\definecolor{algobackground}{RGB}{245, 247, 250}
\definecolor{algoborder}{RGB}{220, 225, 235}

\title{SAEs \emph{Can} Improve Unlearning: Dynamic Sparse Autoencoder Guardrails for Precision Unlearning in LLMs}

\author{
  Aashiq Muhamed\textsuperscript{\dag}, Jacopo Bonato\textsuperscript{\ddag}\thanks{Work performed at Leonardo Labs, now at Prometeia S.p.A.} , \textbf{Mona Diab}\textsuperscript{\dag}, \textbf{Virginia Smith}\textsuperscript{\dag} \\
  \{amuhamed,mdiab,smithv\}@andrew.cmu.edu, jacopo.bonato@prometeia.com\\
  \textsuperscript{\dag}Carnegie Mellon University,
  \textsuperscript{\ddag}Prometeia S.p.A.
}

\begin{document}

\ifcolmsubmission
\linenumbers
\fi

\maketitle

\begin{abstract}
\looseness=-1

Machine unlearning is a promising approach to improve LLM safety by removing unwanted knowledge from a trained model. However, prevailing gradient-based unlearning methods suffer from issues such as high computational costs, hyperparameter instability, poor sequential unlearning capability, vulnerability to relearning attacks, low data efficiency, and lack of interpretability. While Sparse Autoencoders are well-suited to improve these aspects by enabling targeted activation-based unlearning, prior approaches underperform gradient-based methods. This work demonstrates that, contrary to these earlier findings, SAEs can significantly improve unlearning when employed dynamically. We introduce \textbf{D}ynamic \textbf{S}AE \textbf{G}uardrails (DSG), a novel method for precision unlearning that leverages principled feature selection and a dynamic classifier. Our experiments show DSG substantially outperforms leading unlearning methods, achieving superior forget-utility trade-offs.
DSG addresses key drawbacks of gradient-based approaches for unlearning---offering enhanced computational efficiency and stability, robust performance in sequential unlearning, stronger resistance to relearning attacks, better data efficiency including zero-shot settings, and more interpretable unlearning.\smash{\footnotemark}
\end{abstract}
\footnotetext{Code is available at \url{https://github.com/aashiqmuhamed/DynamicSAEGuardrails}}

\section{Introduction}

Machine unlearning, the process of removing specific information from trained LLMs, is a promising tool for applications in safety, privacy, and model maintenance \citep{liu2025rethinking}. However, predominant gradient-based unlearning methods suffer from significant limitations \citep{barez2025open}. Existing methods struggle to precisely balance forgetting target data with preserving general utility~\citep{thaker2024position}, incur high \textbf{computational costs}~\citep{cha2024towards}, exhibit \textbf{hyperparameter instability}~\citep{bu2024unlearning}, degrade quickly under \textbf{sequential unlearning} requests~\citep{shi2024muse, gao2025largelanguagemodelcontinual}, are vulnerable to \textbf{relearning attacks}~\citep{hu2024jogging, deeb2024unlearning,lucki2024adversarial}, lack \textbf{data efficiency}~\citep{gao2024practical}, and offer little \textbf{interpretability}~\citep{xu2024machine}. While interventions using Sparse Autoencoders (SAEs) \citep{bricken2023towards} offer a potential path towards more targeted, activation-based unlearning, existing SAE approaches have typically underperformed gradient-based approaches due to imprecise interventions that cause unintended side effects~\citep{farrell2024applying}.

This paper demonstrates that, contrary to previous findings, \textbf{SAEs \textit{can} significantly improve unlearning} when employed dynamically. We introduce \textbf{Dynamic SAE Guardrails (DSG)}, a novel method that leverages SAEs for precise, efficient, and interpretable unlearning in LLMs. DSG integrates Fisher Information-based feature selection to identify features causally linked to the forget data, with a dynamic, input-dependent classifier that triggers targeted feature clamping only when necessary. This conditional intervention acts as a guardrail, preventing the model from accessing specific knowledge pathways for relevant inputs while leaving general capabilities intact. Through extensive experiments on standard benchmarks, we show that DSG not only achieves a superior balance between forgetting and utility preservation compared to leading gradient-based and static SAE methods, but also directly addresses their core limitations.
Our main contributions are:
\begin{enumerate}
    \item We introduce DSG, a new activation-based unlearning method featuring principled SAE feature selection and a dynamic classifier for precise, conditional intervention.
    \item We demonstrate empirically that DSG achieves a superior balance between forgetting and utility preservation compared to leading gradient-based and SAE-based unlearning approaches on multiple benchmarks.
    \item We show that DSG provides substantial benefits over gradient-based unlearning such as greater \textbf{hyperparameter stability}, improved \textbf{computational efficiency}, and \textbf{sequential unlearning capability}, enhanced \textbf{resistance against relearning attacks}, enhanced \textbf{data efficiency even in the zero-shot setting} and  \textbf{interpretable unlearning}.
\end{enumerate}

\section{Background and Related Work}
\label{sec:background}

\paragraph{Unlearning in Large Language Models.}

\looseness=-1
Machine unlearning aims to modify a trained target model $\mathcal{M}(D)$ to behave as if specific data, the forget set $D_{forget}$, had never been part of its training data $D$ \citep{bourtoule2021machine, cao2015towards}. The resulting model, $\mathcal{M}_{unlearn}$, should ideally be indistinguishable from a model trained only on the retain set $D_{retain} = D \setminus D_{forget}$.  As retraining LLMs from scratch is computationally prohibitive, research focuses on \emph{approximate unlearning} \citep{liu2025rethinking}. These methods face the core challenge of balancing knowledge removal (\textit{forget quality}) and maintaining general capabilities (\textit{utility preservation}) \citep{shi2024muse, maini2024tofu}.

\looseness=-1
The dominant approach for approximate unlearning involves gradient-based optimization \citep{wang2024rethinking}. Methods like Gradient Ascent (GA) \citep{jang2023knowledge}, Gradient Difference (GradDiff) \citep{liu2022backdoor}, Negative Preference Optimization (NPO) \citep{zhang2024negative}, and RMU \citep{li2024wmdp} finetune model weights to reduce the influence of $D_{forget}$, often using regularization (e.g., KL divergence) to protect utility \citep{maini2024tofu, yao2024machine}. However, these gradient-based techniques frequently suffer from significant limitations: high \textbf{computational cost} (requiring backward passes), instability under \textbf{hyperparameter tuning}, degraded performance under \textbf{sequential unlearning} requests \citep{gao2024practical}, vulnerability to \textbf{relearning attacks} \citep{hu2024jogging}, poor \textbf{data efficiency}, and a lack of \textbf{interpretability} \citep{barez2025open}. These widespread challenges motivate the exploration of alternative unlearning paradigms, such as the activation-based interventions explored in this work.

\vspace{-0.1in}
\paragraph{Sparse Autoencoders (SAEs).}

\looseness=-1
Modern DNNs operate in a regime of superposition, where multiple features or capabilities are encoded along the same dimensions of hidden activations \citep{elhage2022toy}. 
SAEs provide an unsupervised method for disentangling these superposed representations into interpretable features \citep{bricken2023towards, cunningham2023sparse}. Given activations $\mathbf{h} \in \mathbb{R}^{d_\text{model}}$ from a specific layer or component of an LLM, an SAE decomposes and reconstructs these activations using encoder and decoder functions: 
$\smash{f(\mathbf{h}) := \sigma(\mathbf{W}_\text{enc}\mathbf{h} + \mathbf{b}_\text{enc})}$ and $ \smash{\hat{\mathbf{h}}(f) := \mathbf{W}_\text{dec}f + \mathbf{b}_\text{dec}}$.
In other words, SAEs express model activations as a sparse linear combination of interpretable feature vectors: 
$\mathbf{h} = \hat{\mathbf{h}} + \varepsilon(\mathbf{h}) = \sum_{i=1}^{d_\text{SAE}} f_i(\mathbf{h})\mathbf{v}_i + \mathbf{b} + \varepsilon(\mathbf{h})$
where $f_i(\mathbf{h}) \in \mathbb{R}$ are scalar feature activations, $v_i \in \mathbb{R}^{d_\text{model}}$ are unit vector feature directions, $b \in \mathbb{R}^{d_\text{model}}$ is a bias term, and $\varepsilon(\mathbf{h}) \in \mathbb{R}^{d_\text{model}}$ is the SAE error term. Wider SAEs continue to improve feature granularity and reduce the error term.

\looseness=-1
SAEs are trained on activations collected from the model processing pretraining data where training is conducted separately for each layer or component of interest (e.g., residual stream, attention outputs) using an objective that minimizes reconstruction loss while enforcing sparsity: $
L = \|\mathbf{h} - \hat{\mathbf{h}}(f(\mathbf{h}))\|_2^2 + \lambda \|f(\mathbf{h})\|_0$. Here $\lambda$ is a sparsity penalty coefficient encouraging most feature activations to be zero for any given input. In this work, we use JumpReLU SAEs \citep{rajamanoharan2024jumping}, which enforce sparsity using a shifted Heaviside step function. 
The interpretability of SAE features stems from their sparse activation pattern—because features are only active for a small fraction of inputs, they must capture meaningful patterns to be useful for reconstruction.
The cost of training SAEs is amortized across multiple downstream applications such as identifying and removing spurious correlations in models \citep{marks2024sparse} and steering behavior \citep{o2024steering}.

\looseness=-1
\citet{farrell2024applying} developed an early approach using SAEs for unlearning
by identifying features that were frequently active on $D_{forget}$ and applying a \textit{static} intervention---clamping these identified features when active at a token irrespective of overall context. However, this produced substantial side-effects on general model utility and ultimately underperformed gradient-based methods like RMU. In this work, we show that SAE unlearning \textit{can} be effective via a context-dependent intervention strategy rather than simple static clamping.

\textit{Further background details are provided in Appendix~\ref{app:background_details}.}

\section{Dynamic SAE Guardrails (DSG)}
\label{sec:methodology}

\begin{wrapfigure}{r}{0.45\textwidth}
    \centering
     \vspace{-0.4in}
    \includegraphics[width=\linewidth]{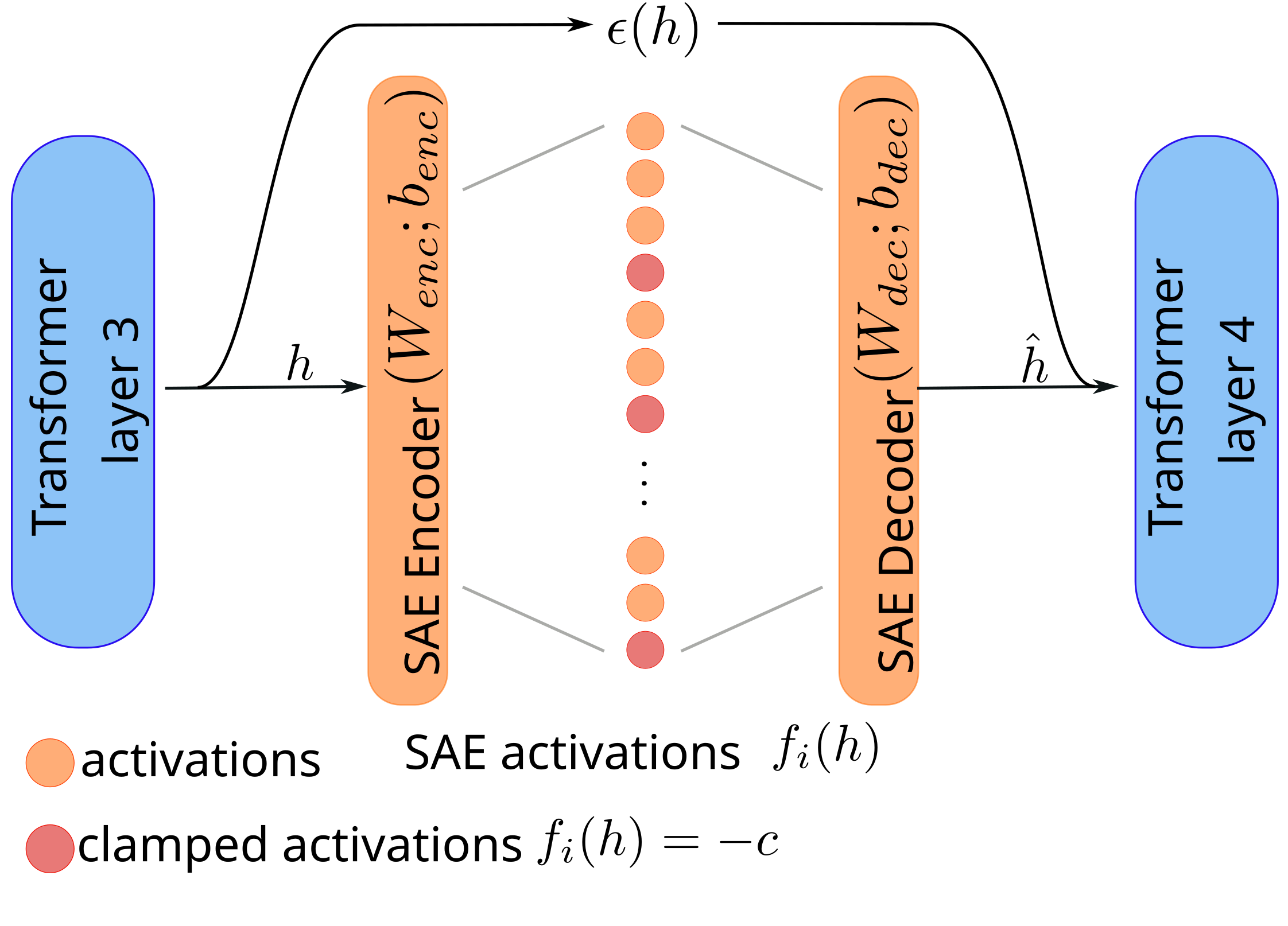}
    \vspace{-0.3in}
    \caption{\small An illustration of DSG}
    \label{fig:sae_plot}
    \vspace{-0.35in}
\end{wrapfigure}

DSG (\autoref{fig:sae_plot}, Algorithm \ref{alg:DSG}) is a targeted unlearning method for LLMs that leverages the interpretability of SAEs and combines: (1) a causal framing that motivates feature selection, (2) a theoretically justified feature importance scoring based on Fisher Information (FI), (3) a dynamic, input-dependent classification rule based on a threshold that is statistically optimal under a monotone likelihood-ratio assumption, and (4) a targeted clamping intervention to remove the influence of selected features.

\subsection{Causal Framework and Problem Setup}

We frame unlearning through a causal perspective where $\mathcal{D}_{\text{forget}}$ and $\mathcal{D}_{\text{retain}}$ influence model representations $\mathbf{E}$ and outputs $\mathbf{Y}$ through multiple pathways \citep{shen2024camu}. These include representation-mediated pathways ($\mathcal{D} \rightarrow \mathbf{E} \rightarrow \mathbf{Y}$), potential direct influence ($\mathcal{D} \rightarrow \mathbf{Y}$), and intertwined knowledge ($\mathcal{D}_{\text{forget}} \leftrightarrow \mathcal{D}_{\text{retain}}$) where conceptual overlap exists between the datasets.
SAE features $f_j$ derived from $\mathbf{E}$ serve as interpretable mediators  \citep{pearl2009causality} of information flow through these causal pathways. From this perspective, unlearning involves blocking pathways from $\mathcal{D}_{\text{forget}}$ to $\mathbf{Y}$ while preserving the pathways from $\mathcal{D}_{\text{retain}}$ to $\mathbf{Y}$. DSG implements this causal intervention $\mathrm{do}(f_j =-c)$ on forget features identified by analyzing SAE activation patterns across both datasets.

\subsection{Feature Selection: Identifying Causal Mediators via Fisher Information}
To identify which SAE features, $F_{\text{forget}}$ mediate the causal influence of $\mathcal{D}_{\text{forget}}$, 
we establish two key theoretical connections: first between FI and feature activation, and second between FI and causal influence. We then describe our percentile based feature selection.

\begin{theorem}[Fisher Information Approximation]
\label{thm:fisher_approx}
Fix the SAE encoder and consider the Gaussian decoder likelihood $\mathbf{R}\mid\mathbf{z}\sim\mathcal{N}(\hat{\mathbf{h}}(\mathbf{z}),\sigma^2\mathbf{I}_d)$, where $\mathbf{z}=f(\mathbf{h})$, $\hat{\mathbf{h}}$ is the linear decoder, and $\sigma>0$ is fixed. For unconstrained decoder weights and finite activation second moments, the Fisher block for feature $j$ satisfies
\[
\operatorname{tr}(\mathcal{I}_{jj})=\frac{d}{\sigma^2}\mathbb{E}_{\mathbf{h}}[f_j(\mathbf{h})^2].
\]
\end{theorem}

The proof of Theorem \ref{thm:fisher_approx} is provided in Appendix \ref{app:fisher_proof}. This establishes that squared feature activations are proportional to the Fisher Information of the corresponding decoder weights.

\begin{theorem}[Fisher Information as a Proxy for Causal Feature Importance]
\label{thm:fisher_causal}
For a fixed feature $j$, suppose $p(\mathbf{Y}\mid x,\mathrm{do}(f_j=a))=\mathcal{N}(\mathbf{b}_j(x)+a\mathbf{u}_j,\sigma^2\mathbf{I})$ over the range from zero to its observed activation, with $\mathbf{u}_j$ and $\sigma>0$ constant across contexts and datasets. Assuming finite activation second moments, the expected output KL divergence from zero-ablation is
\[
\mathrm{Influence}_{D}(j)=\frac{\|\mathbf{u}_j\|^2}{2\sigma^2}\mathbb{E}_{D}[f_j(\mathbf{h})^2].
\]
\end{theorem}

\looseness=-1
Proof of Theorem \ref{thm:fisher_causal} is in Appendix \ref{app:causal_connection}. 
Under these assumptions, squared activation is proportional to the zero-ablation influence of a fixed feature, with a feature-specific sensitivity factor.
Squared activation serves as a computationally tractable proxy for causal influence.

\paragraph{Importance Scores.}  
DSG obtains token-level SAE activations from each sequence in both $\mathcal{D}_{\text{forget}}$ and $\mathcal{D}_{\text{retain}}$, squares them, and aggregates the results into matrices $\smash{\mathbf{A}_{\text{forget}} \in \mathbb{R}^{n_F \times d_{\text{SAE}}}}$ and $\smash{\mathbf{A}_{\text{retain}} \in \mathbb{R}^{n_R \times d_{\text{SAE}}}}$ ($n_F$ and $n_R$ are the total numbers of tokens in the respective datasets). 
For each token $t$ in sequence $x$, we have the activation $\smash{f_j(\mathbf{h}_{x,t})}$ of feature $j$ on the hidden state $\smash{\mathbf{h}_{x,t}}$. Each entry of the activation matrices is thus $\smash{\mathbf{A}_{\text{forget}}[i,j]\approx\bigl[f_j(\mathbf{h}_{x,t})\bigr]^2}$ for a token $t$ in sequence $x \in \mathcal{D}_{\text{forget}}$ (and similarly for $\mathbf{A}_{\text{retain}}$). 
From these, we compute the average squared activation per feature as $\smash{\texttt{forget\_score}(j)=   1/n_F \sum_{x \in \mathcal{D}_{\text{forget}}}\sum_{t=1}^{|x|}\bigl[f_j(\mathbf{h}_{x,t})\bigr]^2}$ and $\smash{\texttt{retain\_score}(j)= 1/n_R \sum_{x \in \mathcal{D}_{\text{retain}}}\sum_{t=1}^{|x|}\bigl[f_j(\mathbf{h}_{x,t})\bigr]^2}$, and define the relative importance by $\texttt{imp\_ratio}(j)=\frac{\texttt{forget\_score}(j)}{\max\{\texttt{retain\_score}(j),\varepsilon\}}$, with $\varepsilon>0$ to avoid division by zero. Under the assumptions of Theorem~\ref{thm:fisher_causal}, this ratio equals the relative zero-ablation influence of feature $j$ when its sensitivity is nonzero and $\texttt{retain\_score}(j)\ge\varepsilon$.

\begin{algorithm}[htbp]
\caption{\small \textbf{Dynamic SAE Guardrails (DSG)}}
\label{alg:DSG}
\begin{tcolorbox}[
  colback=algobackground,
  colframe=algoborder,
  boxrule=0pt,
  arc=1pt,
  left=0pt,
  right=0pt,
  top=0pt,
  bottom=0pt,
  width=\textwidth,
  toptitle=0pt,
  bottomtitle=0pt
  
]
{\small
\begin{algorithmic}[0]
\REQUIRE LLM with SAE features $\{f_j\}$; datasets $\mathcal{D}_{\text{forget}}, \mathcal{D}_{\text{retain}}$; clamp strength $c$; percentiles $(p_{\text{ratio}}, p_{\text{dyn}})$; feature count $n_{\text{feats}}$
\STATE \textbf{Feature Selection:} 
\STATE \hspace{0.5em} Compute feature importance scores and threshold $\tau_{\text{ratio}}$ from percentiles
\STATE \hspace{0.5em} Identify $F_{\text{forget}} = \{j: \texttt{imp\_ratio}(j) \geq \tau_{\text{ratio}}\}$
\STATE \hspace{0.5em} Sort $F_{\text{forget}}$ by descending $\texttt{forget\_score}(j)$ and select top $n_{\text{feats}}$ features to form $S_{n_{\text{feats}}}$
\STATE \textbf{Dynamic Threshold Calibration:} 
\STATE \hspace{0.5em} Compute $\rho(x) = \frac{1}{|x|}\sum_{t} \mathbf{1}[\exists j\in S_{n_{\text{feats}}}: f_j(\mathbf{h}_t)>0]$ for each $x \in \mathcal{D}_{\text{retain}}$
\STATE \hspace{0.5em} Set threshold $\tau = \text{Percentile}(\{\rho(x)\}_{x\in \mathcal{D}_{\text{retain}}}, p_{\text{dyn}})$
\STATE \textbf{Inference-Time Intervention:} 
\STATE \hspace{0.5em} For input sequence $x$, compute $\rho(x)$ and classify as forget-relevant if $\rho(x) > \tau$
\STATE \hspace{0.5em} If forget-relevant: For each token $t$ and feature $j \in S_{n_{\text{feats}}}$, set $f_j'(\mathbf{h}_t) = -c$
\STATE \hspace{0.5em} Otherwise: Preserve all feature activations
\end{algorithmic}
}
\end{tcolorbox}
\end{algorithm}

\paragraph{Percentile-Based Feature Selection.}
To select features most causally relevant to $\mathcal{D}_{\text{forget}}$, we employ a percentile-based approach using $p_{\text{ratio}}$ to compute $\tau_{\text{ratio}}$ as $\smash{\text{Percentile}({\{\texttt{imp\_ratio}(j)\}_{j=1}^{d_{\text{SAE}}}}, p_{\text{ratio}})}$. $\text{Percentile}(S, p)$ returns the empirical $p$-th percentile of $S$. For example, with $p_{\text{ratio}} = 95$, $\tau_{\text{ratio}}$ is the 95th percentile of the feature importance ratios. We define the set of forget-mediating features as $F_{\text{forget}} = \{j : \texttt{imp\_ratio}(j)\ge\tau_{\text{ratio}}\}$. To filter out noisy features, we sort features in $F_{\text{forget}}$ by descending $\texttt{forget\_score}(j)$ and select the top $n_{\text{feats}}$ to form the final intervention set $S_{n_{\text{feats}}}$.

\subsection{Dynamic Sequence-Level Classification and Intervention}
\label{sec:dynamic}
DSG employs a dynamic input-dependent classification mechanism to minimize unintended side-effects on content unrelated to the forget knowledge. 
\begin{definition}[Forget-Set Activated Token]
A token $x_t$ is considered \emph{forget-set activated} if at least one feature $j \in S_{n_{\text{feats}}}$ has a positive activation: $\smash{f_j(\mathbf{h}_t) > 0}$.
\end{definition}

For an input sequence $\smash{x=(x_1,\dots,x_T)}$ of length $T$, we define the statistic $\smash{\rho(x) = \frac{1}{T}\sum_{t=1}^T \mathbf{1}[\exists\,j\in S_{n_{\text{feats}}}:\; f_j(\mathbf{h}_t)>0]}$, representing the percentage of forget-set activated tokens. A high $\rho(x)$ indicates that query $x$ frequently activates the selected forget features.

\vspace{-0.1in}

\paragraph{Threshold Selection and Classification.} We select a threshold $\tau \in [0, 1]$ based on the distribution of $\rho(x)$ on $\mathcal{D}_{\text{retain}}$ using $\tau = \text{Percentile}(\{\rho(x)\}_{x\in \mathcal{D}_{\text{retain}}}, p_{\text{dyn}})$ which controls the trade-off between unlearning effectiveness and performance preservation. 
\begin{wrapfigure}{r}{0.5\textwidth}
    \centering
    \includegraphics[width=\linewidth]{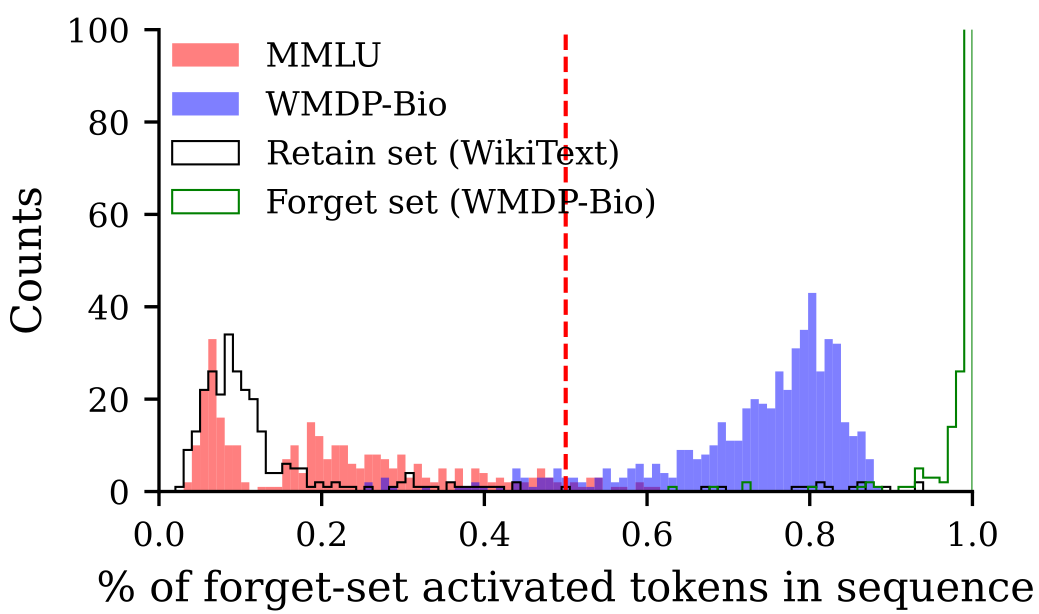}
    \vspace{-0.2in}
    \caption{\small Distribution of $\rho(x)$ for unlearning on WMDP-Bio. Threshold at 95th percentile (dashed red line) separates MMLU from WMDP.}
    \label{fig:bio_token}
\end{wrapfigure}
\looseness=-1
Empirically, we find $\rho(x)$ is typically larger on $\mathcal{D}_{\text{forget}}$ than on $\mathcal{D}_{\text{retain}}$, as seen in Figure~\ref{fig:bio_token}, which shows the distribution for both forget-domain queries (WMDP-Bio) and general knowledge queries (MMLU). $\tau$ is chosen to control the retain set's false-positive rate and separates forget-set queries effectively, achieving high recall on $\mathcal{D}_{\text{forget}}$. On this example, DSG successfully transfers from retain set (WikiText) and forget set to the test query set. We define classifier $C(x) = \mathbf{1}[\rho(x) > \tau]$, labeling inputs as either \emph{forget-relevant} or \emph{retain-relevant}. Under a monotone likelihood-ratio assumption, the Neyman-Pearson lemma establishes the optimality of thresholding:

\begin{theorem}[Neyman-Pearson Optimality]
\label{thm:neyman_pearson}
Suppose the sequence distributions have common support with $p_R>0$ and likelihood ratio $p_F(x)/p_R(x)=g(\rho(x))$ for a nonnegative, nondecreasing function $g$. If a threshold $\tau^*$ satisfies $\Pr_R[\rho(X)>\tau^*]=\alpha$, then $C^*(x)=\mathbf{1}[\rho(x)>\tau^*]$ maximizes the true-positive rate among all classifiers with false-positive rate at most $\alpha$.
\end{theorem}
\vspace{-0.1in}
\looseness=-1
The proof appears in Appendix \ref{app:neyman_pearson_proof} and states that under the stated likelihood-ratio assumption, thresholding $\rho(x)$ is the optimal classification approach for a given false-positive rate.

\paragraph{Conditional Clamping.} Our intervention is conditional on the classifier $C(x)$. When $C(x) = 1$ (forget-relevant), for each token $x_t$ and feature $j \in S_{n_{\text{feats}}} $, %
we set $\smash{f'_j(\mathbf{h}_t) = -c}$, where $-c$ is a large negative constant we call \emph{clamp strength}. This implements a targeted $\smash{do(f_j(\mathbf{h}_t) = -c)}$ operation, selectively severing the causal pathway \emph{only when the input query is deemed forget-relevant}. When $C(x) = 0$ (retain-relevant), we leave all features unchanged: $\smash{f'_j(\mathbf{h}_t) = f_j(\mathbf{h}_t)}$. This preserves the model's original behavior for queries classified as retain-relevant, minimizing side-effects and maintaining performance on $\mathcal{D}_{\text{retain}}$.

\looseness=-1
The dynamic clamping in DSG contrasts with static clamping methods \citep{farrell2024applying}, which intervene based only on feature activation without sequence-level classification, and risk inadequate coverage on $\mathcal{D}_{\text{forget}}$ or excessive side-effects on $\mathcal{D}_{\text{retain}}$. For the same feature set, if gating excludes retain inputs but no forget inputs, DSG preserves forget coverage while reducing the retain intervention rate (Theorem~\ref{appthm:dominance}, Appendix~\ref{app:dominance_proof}). Our experiments demonstrate a \textbf{superior unlearning-utility trade-off} over the evaluated static baselines.

\section{Experiments and Results}

\subsection{Unlearning on WMDP}
\label{sec:unlearningWMDP}

We evaluate DSG on the WMDP dataset \citep{li2024wmdp}, which benchmarks hazardous knowledge unlearning across multiple domains. We focus on  WMDP-Bio (biosecurity) and WMDP-Cyber (cybersecurity). For each domain, our unlearning setup uses domain-specific $\mathcal{D}_{\text{forget}}$—PubMed papers containing bio-weapon related content for WMDP-Bio and GitHub repositories for WMDP-Cyber—and WikiText \citep{merity2016pointer} as $\mathcal{D}_{\text{retain}}$. We evaluate unlearning effectiveness using the WMDP multiple-choice question test sets, which were not exposed to models during the unlearning process.

Following SAEBench \citep{karvonen2025saebench}, we evaluate unlearning only on questions the target model correctly answers across all 24 permutations of the 4 multiple-choice options. This yields 522/1273 questions for WMDP-Bio and 275/1987 questions for WMDP-Cyber.
For evaluating model utility, we similarly filter MMLU questions that the model answers correctly across all permutations. This yields 305 questions from history, computer science, geography, and human aging for WMDP-Bio. For WMDP-Cyber, we use 371 MMLU questions, replacing computer science with  biology.
Table~\ref{tab:bio_performance} reports the configuration that minimizes WMDP accuracy while maintaining at least 99\% of the target model MMLU accuracy along with with MT-Bench scores that measure general fluency.

\paragraph{Experimental Setup.} 
We implement DSG using \texttt{gemma-2-2b-it} model with \texttt{gemma-scope-2b-pt-res} SAE (width 16k) \citep{lieberum2024gemma} applied to layer 3 at $\ell_0$ 142. We use $P_{dyn}=95$ for both domains, and $P_{ratio}=95$ for WMDP-Bio and $P_{ratio}=90$ for WMDP-Cyber.
We compare DSG against several baselines across a broad hyperparameter sweep:
GA \citep{jang2023knowledge}, NPO \citep{zhang2024negative}, SSD \citep{foster2024fast}, SCRUB \citep{kurmanji2023towards},\citet{farrell2024applying} and RMU \citep{li2024wmdp}. Complete hyperparameter details are provided in \autoref{app:unlearning-wmdp}.

\begin{figure}[tb]
    \centering
    \begin{subfigure}{0.495\textwidth}
        \centering
        \includegraphics[width=\textwidth]{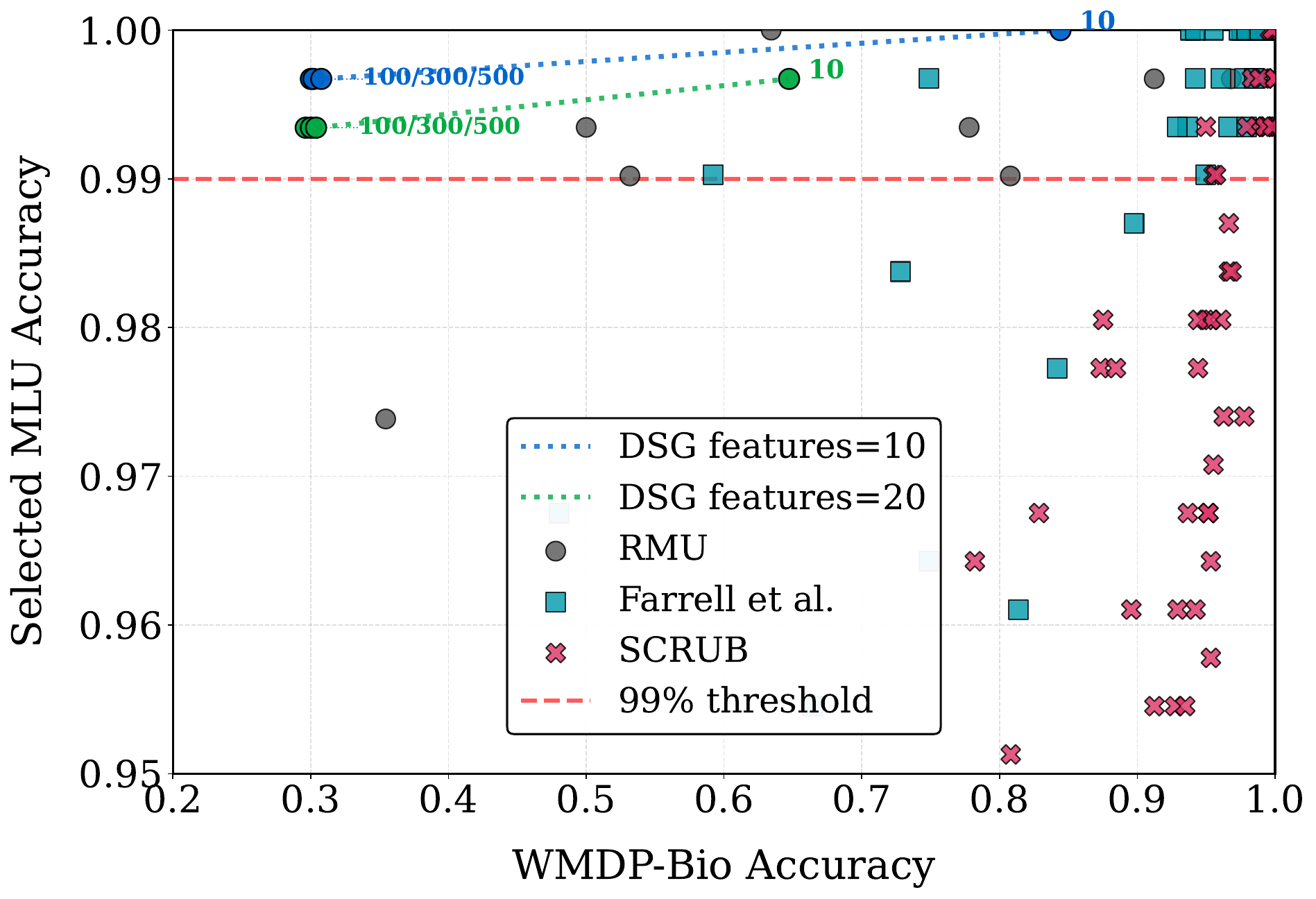}
        \label{fig:wmdp_bio_main}
    \end{subfigure}
    \hfill
    \begin{subfigure}{0.495\textwidth}
        \centering
        \includegraphics[width=\textwidth]{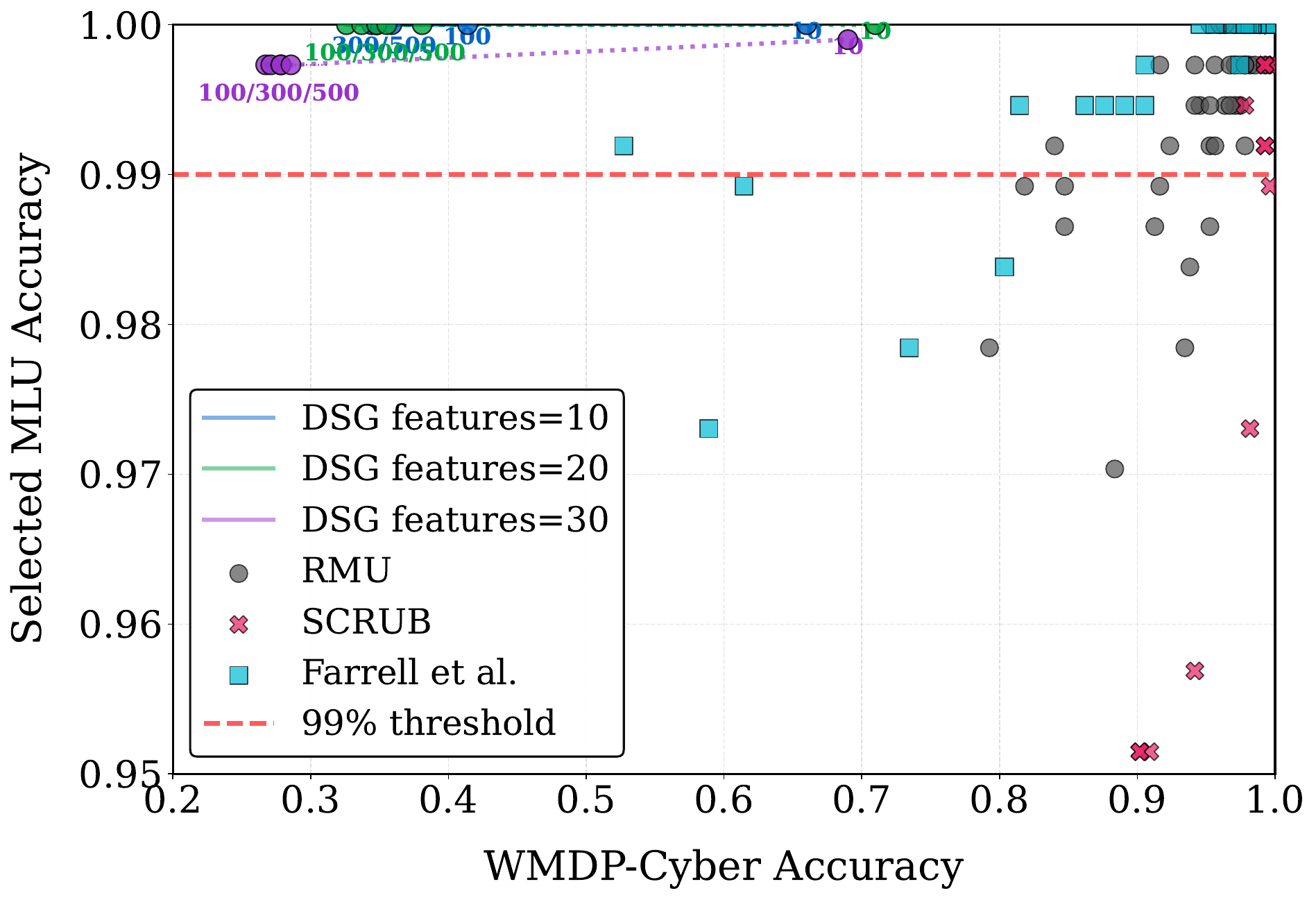}
        \label{fig:your_second_figure}
    \end{subfigure}
    \vspace{-0.3in}
    \caption{\small 
    Unlearning performance on WMDP-Bio (left) and WMDP-Cyber (right). Higher MMLU accuracy and lower WMDP accuracy is better. Clamp strengths ($c$) used for DSG points are shown as annotations. DSG Pareto-dominates the top four baseline methods (RMU, SCRUB, Farrell et al., SSD).}
    \label{fig:combined_figures}
\end{figure}

\begin{table}[tb]
\centering
\small
\begin{tabular}{l cccccc c}
\toprule
\multirow{2}{*}{Method} & \multirow{2}{*}{WMDP Bio ($\downarrow$)} & \multicolumn{5}{c}{MMLU ($\uparrow$)} & \multirow{2}{*}{MT ($\uparrow$)} \\
\cmidrule(lr){3-7} 
 & & HS Hist & C. CS & HS Geo & H. Aging & All & \\
\midrule
Target $\mathcal{M}$ & 100.00 & 100.00 & 100.00 & 100.00 & 100.00 & 100.00 & 7.36 \\
\midrule
GA & 99.44 & 98.18 & 100.00 & 100.00 & 100.00 & 99.35 & 7.44 \\
NPO & 97.95  & 99.99  & 88.88 & 100.00& 98.82 & 99.35 & 7.29 \\
SSD & 99.44 & 100.00 & 100.00 & 100.00 & 98.82 & 99.68 & 7.24 \\
SCRUB & 94.97 & 99.09 & 100.00 & 100.00 & 98.82 & 99.35 & 6.09 \\
Farrell et al. & 59.22 & 100.00 & 100.00 & 100.00 & 96.47 & 99.03 & 7.33 \\
RMU & 50.00 & 99.08 & 100.00 & 100.00 & 98.81 & 99.47 & 7.21 \\
\midrule
\textbf{DSG (Ours)} & \textbf{29.64} & 100.00 & 100.00 & 100.00 & 97.62 & 99.34 & \textbf{7.78} \\
\bottomrule
\end{tabular}
\caption{\small Unlearning performance on WMDP-Bio. All represents the average MMLU score. MT-Bench scores show 0.16 variance across 5 runs. DSG shows superior unlearning effectiveness compared to baselines while maintaining high MMLU performance. }

\label{tab:bio_performance}
\end{table}

\vspace{-0.15in}
\paragraph{Results.} 

As shown in Table~\ref{tab:bio_performance}, DSG significantly outperforms all baselines on the WMDP-Bio unlearning task, reducing accuracy to 29.64\% compared to the next best method RMU at 50.00\%. It maintains high MMLU performance (99.34\% average) and achieves the highest MT-Bench score (7.78), showing superior preservation of general model capabilities.
The results on WMDP-Cyber (Appendix~\ref{app:wmdp-cyber}) reinforce these findings. 
Figure~\ref{fig:combined_figures} provides a more comprehensive view of the unlearning-utility trade-off landscape, plotting all configurations with MMLU accuracy above 95\%. DSG Pareto-dominates all baseline methods: for any level of utility preservation (MMLU accuracy), DSG achieves more effective unlearning.

\looseness=-1
This superior performance is coupled with significant practical advantages over gradient-based methods in terms of \textbf{computational efficiency and hyperparameter stability}. Gradient-based approaches often exhibit \textbf{hyperparameter instability}, where slight tuning changes can drastically alter outcomes, risking poor unlearning or utility collapse. Furthermore, they require computationally costly backward passes through the LLM for optimization. In contrast, DSG shows greater \textbf{hyperparameter stability} (\autoref{fig:combined_figures}) and efficiency. It requires only forward passes: one to gather feature statistics initially, and then lightweight intervention during inference, completely avoiding expensive gradient calculations. This combination of efficiency and stability makes DSG particularly advantageous for large models and frequent unlearning where gradient computations impose substantial overhead.

\subsection{Unlearning on Muse}

\looseness=-1
We evaluate DSG on MUSE \citep{shi2024muse} comprising two corpora: NEWS and BOOKS, and focusing on six dimensions: verbatim memorization, knowledge memorization, privacy leakage, utility preservation, forget set scalability, and sequential unlearning. MUSE provides a challenging evaluation setting where forget and 
retain sets share substantial domain overlap: in NEWS, both sets are 
drawn from the same distribution of BBC articles, while in BOOKS, the 
forget set (Harry Potter books) and retain set (Harry Potter FanWiki) 
contain highly related content.

\begin{table}[htb]
\centering
\small
\setlength{\tabcolsep}{8pt}
\renewcommand{\arraystretch}{0.75}
\resizebox{\linewidth}{!}{ 
\begin{tabular}{lrlrlrlrl}
\toprule
 & \multicolumn{2}{c}{\small \textbf{C1. No Verbatim Mem.}} & \multicolumn{2}{c}{\small \textbf{C2. No Knowledge Mem.}} & \multicolumn{2}{c}{\small \textbf{C3. No Privacy Leak.}}  & \multicolumn{2}{c}{\small \textbf{C4. Utiltiy Preserv.}} \\
 & \multicolumn{2}{c}{\small \textsf{VerbMem} on $\mathcal{D}_{\text{forget}}$ ($\downarrow$)} & \multicolumn{2}{c}{\small \textsf{KnowMem} on $\mathcal{D}_{\text{forget}}$ ($\downarrow$)} & \multicolumn{2}{c}{\small \textsf{PrivLeak} ($\in [-5\%, 5\%]$)} & \multicolumn{2}{c}{\small \textsf{KnowMem} on $\mathcal{D}_{\text{retain}}$ ($\uparrow$)} \\
\midrule
\multicolumn{9}{c}{\cellcolor{gray!10}\textbf{\textsc{News}}} \\ 
Target $\mathcal{M}$  &
$21.15$& &
$29.51$& &
$-88.16$& &
$26.78$\\
\midrule
\textsf{GA}&
0.62& \colorbox{green!20}{$\downarrow97.1\%$}&
0.00& \colorbox{green!20}{$\downarrow100.0\%$}&
-8.16& \colorbox{red!20}{\small \textsf{under-unlearn}}&
0.09& \colorbox{red!20}{$\downarrow99.7\%$}\\
\textsf{GradDiff}&
2.81& \colorbox{green!20}{$\downarrow86.7\%$}&
0.71& \colorbox{green!20}{$\downarrow97.6\%$}&
93.10& \colorbox{red!20}{\small \textsf{over-unlearn}}&
7.76& \colorbox{red!20}{$\downarrow71.0\%$}\\
\textsf{NPO}&
20.98& \colorbox{green!20}{$\downarrow0.8\%$}&
25.14& \colorbox{green!20}{$\downarrow14.8\%$}&
-53.42& \colorbox{red!20}{\small \textsf{under-unlearn}}&
29.02& \colorbox{green!20}{$\uparrow8.4\%$}\\
\textsf{SimNPO}&
21.14& \colorbox{green!20}{$\downarrow0.0\%$}&
27.70& \colorbox{green!20}{$\downarrow6.1\%$}&
-89.84& \colorbox{red!20}{\small \textsf{under-unlearn}}&
30.59& \colorbox{green!20}{$\uparrow14.2\%$}\\
\textsf{RMU}&
9.60& \colorbox{green!20}{$\downarrow54.6\%$}&
26.63& \colorbox{green!20}{$\downarrow9.8\%$}& 
75.02& \colorbox{red!20}{\small \textsf{over-unlearn}}&
26.41& \colorbox{red!20}{$\downarrow1.4\%$}\\
\textsf{\textbf{DSG (Ours)}}&
11.80& \colorbox{green!20}{$\downarrow44.2\%$}&
0.44& \colorbox{green!20}{$\downarrow98.5\%$}& 
12.08& \colorbox{red!20}{\small \textsf{over-unlearn}}&
25.65& \colorbox{red!20}{$\downarrow4.2\%$}\\
\midrule
\multicolumn{9}{c}{\cellcolor{gray!10}\textbf{\textsc{Books}}} \\
Target $\mathcal{M}$ &
$15.80$& &
$33.90$& &
$-98.80$& &
$35.28$ \\
\midrule
\textsf{GA}&
2.61& \colorbox{green!20}{$\downarrow83.5\%$}&
0.17& \colorbox{green!20}{$\downarrow99.5\%$}&
-1.58& \colorbox{green!20}{\small \textsf{acceptable}}&
0.57& \colorbox{red!20}{$\downarrow98.4\%$}\\
\textsf{GradDiff}&
9.49& \colorbox{green!20}{$\downarrow39.9\%$}&
21.57& \colorbox{green!20}{$\downarrow36.4\%$}&
-10.30& \colorbox{red!20}{\small \textsf{under-unlearn}}&
23.66& \colorbox{red!20}{$\downarrow32.9\%$}\\
\textsf{NPO}&
14.41& \colorbox{green!20}{$\downarrow8.8\%$}&
28.21& \colorbox{green!20}{$\downarrow16.8\%$}&
-97.24& \colorbox{red!20}{\small \textsf{under-unlearn}}&
37.19& \colorbox{green!20}{$\uparrow5.4\%$}\\
\textsf{SimNPO}&
14.55& \colorbox{green!20}{$\downarrow7.9\%$}&
34.36& \colorbox{red!20}{$\uparrow1.4\%$}&
-96.40& \colorbox{red!20}{\small \textsf{under-unlearn}}&
36.62& \colorbox{green!20}{$\uparrow3.8\%$}\\
\textsf{RMU}&
14.89& \colorbox{green!20}{$\downarrow5.8\%$}&
32.59& \colorbox{green!20}{$\downarrow3.9\%$}& 
-97.58& \colorbox{red!20}{\small \textsf{under-unlearn}}&
37.13& \colorbox{green!20}{$\uparrow5.2\%$}\\
\textsf{\textbf{DSG (Ours)}}&
8.73& \colorbox{green!20}{$\downarrow44.7\%$}&
1.79& \colorbox{green!20}{$\downarrow94.7\%$}& 
-23.18& \colorbox{red!20}{\small \textsf{under-unlearn}}&
37.10& \colorbox{green!20}{$\uparrow5.2\%$}\\
\bottomrule 
\end{tabular}}
\caption{ \setlength{\fboxsep}{0pt} \small \looseness=-1 Unlearning performance on MUSE. We highlight in \colorbox{green!20}{ green}  if the method satisfies the criterion and \colorbox{red!20}{red} otherwise. For privacy
leakage, large positive values suggest \colorbox{red!20}{over-unlearning}, while large negative values suggest \colorbox{red!20}{under-unlearning}. DSG shows strong performance across all metrics, achieving substantial reductions in  verbatim and knowledge memorization while maintaining high utility.
}
\label{tbl:main}
\end{table}

\paragraph{Experimental Setup.} 
 We create separate target models for NEWS and BOOKS by finetuning \texttt{gemma-2-2b-it} on each corpus for 5 epochs using learning rate $10^{-5}$ and batch size 32. For each target model, we implement DSG using \texttt{gemma-scope-2b-pt-res} SAE (width 16k) applied to layer 3. For both domains, we use clamp strength 500, $p_{ratio}=95$ and $n_\text{feats}=20$. We use $p_{dyn}=90$ for NEWS and $p_{dyn}=95$ for BOOKS, with the lower threshold for NEWS enabling more effective verbatim memorization removal. For both scalability and sequential unlearning, we use the best NEWS hyperparameters. 

 \looseness=-1
 We compare DSG against: GA, GradDiff \citep{liu2022backdoor}, NPO, SimNPO \citep{fan2024simplicity}, and RMU. Following MUSE, we train for 10 epochs using AdamW with learning rate $10^{-5}$ and batch size 32,  selecting the last epoch checkpoint before utility falls below 90\% of the target model accuracy. Complete hyperparameters can be found in \autoref{app:unlearning-muse}.

\paragraph{Unlearning.} \autoref{tbl:main} shows that DSG outperforms existing baselnes. It is effective at verbatim memorization removal (C1) with 44.2\% reduction on NEWS and 44.7\% on BOOKS. On knowledge memorization (C2), DSG achieves near-complete removal with 98.5\% reduction on NEWS and 94.7\% reduction on BOOKS, outperforming most baselines. On privacy leakage (C3), while not within the ideal range, DSG performs better than the majority of baselines. For utility preservation (C4), DSG maintains 95.8\% of target model performance on NEWS and achieves a 5.2\% improvement on BOOKS compared to the target model.

\paragraph{Scalability.} Figure \ref{fig:knowmem_tradeoff}(a) shows \textbf{DSG is stable and robust when scaling to larger forget sets}. We evaluate performance across forget sets ranging from 0.8M to 3.3M tokens, and DSG maintains its position in the ideal region (high retain set knowledge, low forget set knowledge) even as the forget set size increases. In contrast, gradient-based methods exhibit substantial degradation, with increasingly poor tradeoffs between retaining general knowledge and forgetting targeted information.

\begin{wrapfigure}{r}{0.6\textwidth}
    \centering
    \includegraphics[width=0.6\textwidth]{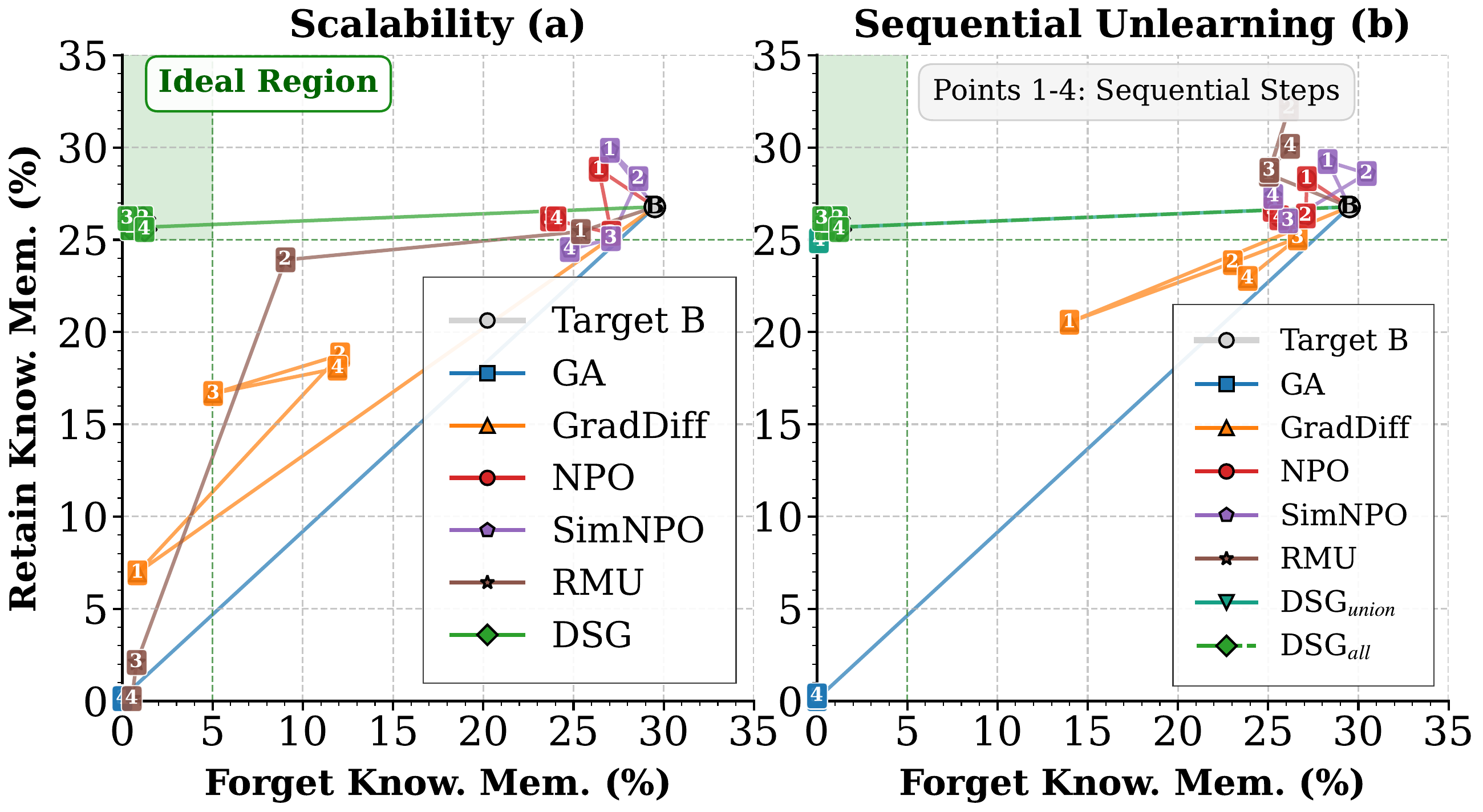}
    \vspace{-0.2in}
    \caption{\small 
    (a) Scalability: Performance across increasing forget set sizes. (b) Sequential Unlearning: Performance across sequential unlearning requests}
     \vspace{-0.1in}
    \label{fig:knowmem_tradeoff}
\end{wrapfigure}
\paragraph{Sequential Unlearning.} \autoref{fig:knowmem_tradeoff}(b) illustrates \textbf{DSG's effectiveness across sequential unlearning requests} on four disjoint NEWS folds. We implement two approaches: DSG$_{\text{all}}$, which cumulatively updates feature importance scores based on each new forget data request; and DSG$_{\text{union}}$, which takes the union of features selected independently at each step and uses this combined set to calculate $\rho(x)$ and threshold $\tau$ on $D_R$. Both approaches perform similarly well, consistently maintaining DSG in the ideal region where other methods rapidly degrade with each additional unlearning operation. Gradient-based methods suffer from catastrophic forgetting during sequential unlearning, where each update pushes the model further from its original performance distribution. (Details in \autoref{app:sequential_unlearning}.)

\subsection{Resistance to Relearning Attacks}
\label{sec:relearning}

\begin{figure}[htbp]
    \centering
    \begin{subfigure}{0.495\textwidth}
        \centering
        \includegraphics[width=\textwidth]{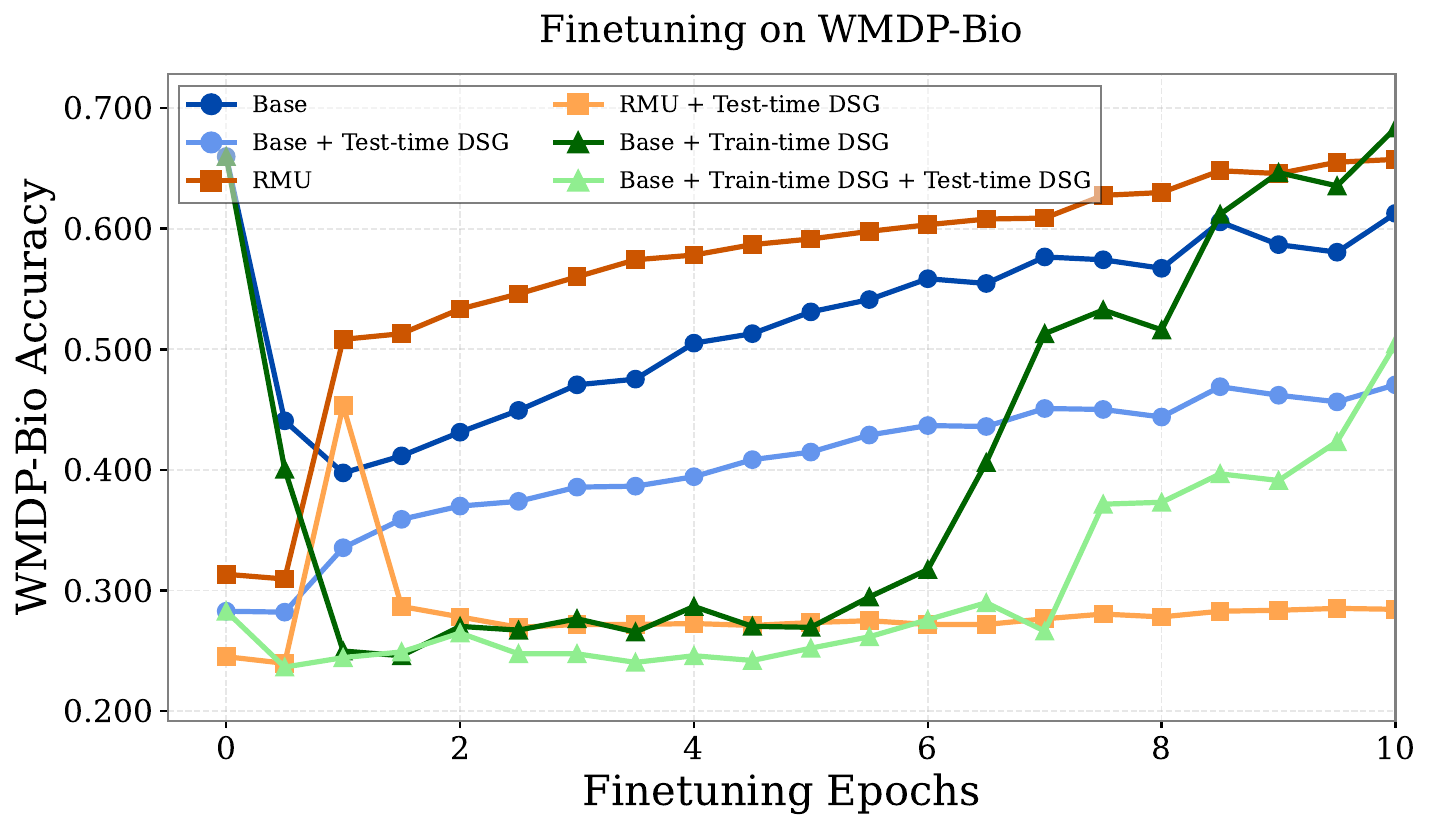}
    \end{subfigure}
    \hfill
    \begin{subfigure}{0.495\textwidth}
        \centering
        \includegraphics[width=\textwidth]{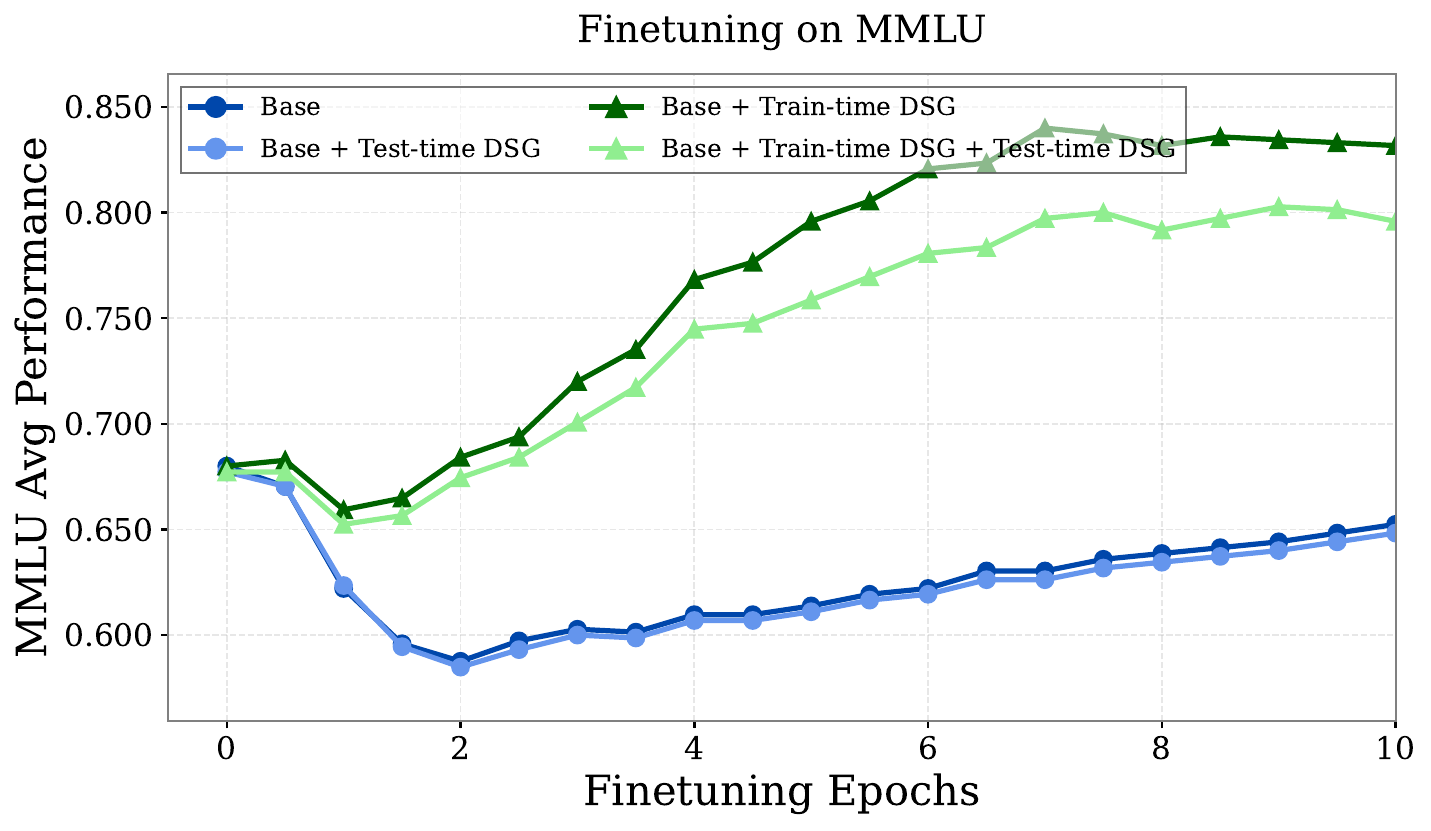}
    \end{subfigure}
    \vspace{-0.2in}
    \caption{ \small Relearning attack resistance across finetuning epochs. (a) DSG demonstrates superior resistance to relearning compared to RMU. (b) Test-time DSG preserves MMLU utility better than Train-time DSG while still providing significant protection.}
    \label{fig:relearning_main}
\end{figure}

We evaluate DSG's resistance to relearning attacks in API-based threat models where adversaries have query access but cannot directly manipulate model weights.
This resistance derives from the Superficial Alignment Hypothesis \citep{zhou2023lima}, which posits that a model's activation geometry stabilizes during pretraining and changes minimally during finetuning. Figure~\ref{fig:relearn_cosine} confirms this empirically, showing high cosine similarity between pre-finetuning and post-finetuning activation vectors, and activation magnitudes clustered around 1.0. By operating on these stable activation patterns rather than weights, DSG creates a more persistent defense. While obfuscation-based attacks have been proposed against activation-based interventions \citep{bailey2024obfuscated}, they are less effective in API-based black-box settings where attackers lack direct access to gradients and model representations.

\paragraph{Methodology.}
We evaluate two DSG defenses against relearning: (1) Test-time DSG, which applies intervention only at inference time after model finetuning, and (2) Train-time DSG, which integrates DSG during finetuning with frozen SAE parameters to filter gradients. We test six configurations with \texttt{google/gemma-2-2b-it} as base model: Base, Base+Test-time DSG, Base+Train-time DSG, Base+Train-time DSG+Test-time DSG, RMU (base model with RMU unlearning applied), and RMU+Test-time DSG. The relearning attack consists of finetuning each configuration on the WMDP-Bio test set for 10 epochs at learning rate 1e-5.

\paragraph{Results and Analysis.}
Figure~\ref{fig:relearning_main}(a) demonstrates clear differences in vulnerability to relearning attacks. Weight-based methods show high susceptibility, with RMU rapidly increasing in WMDP-Bio accuracy when finetuned, eventually exceeding the base model's finetuned performance. The base model itself shows an initial performance decrease before increasing, as the high learning rate temporarily undoes instruction tuning before relearning occurs. 

\textbf{Test-time DSG} provides substantial protection, with RMU+Test-time DSG maintaining near-random performance (25\%) throughout training. However, Base+Test-time DSG shows gradual vulnerability to relearning, with performance slowly increasing over finetuning epochs. This gradual protection erosion reveals a limitation of test-time intervention alone. 

\textbf{Train-time DSG} offers a distinct protective mechanism. Models finetuned with DSG active show immediate reduction to random-level performance that persists through approximately six epochs before gradually recovering. This delayed recovery pattern suggests DSG forces the model to develop entirely new processing circuits rather than simply reactivating suppressed knowledge. Figure~\ref{fig:train_loss_comparison} supports this interpretation, showing significantly higher training loss on WMDP-Bio compared to MMLU when finetuning with DSG active. 

\textbf{Combining both approaches} (Train-time DSG+Test-time DSG) extends resistance through epoch 7, demonstrating how these complementary mechanisms can be layered for enhanced protection. However, these approaches involve utility trade-offs. Figure~\ref{fig:relearning_main}(b) shows that while Base Finetuned and Base Finetuned+Test-time DSG maintain comparable MMLU performance, Train-time DSG exhibits moderate utility decline at higher epoch counts. 

DSG's \textbf{superior resistance to relearning attacks} stems from its activation-based intervention that leverages the stability of activation geometry during finetuning.

\subsection{Data Efficiency and Zero-shot Interpretable Unlearning}

\begin{figure}[htbp]
    \centering
    \includegraphics[width=\textwidth]{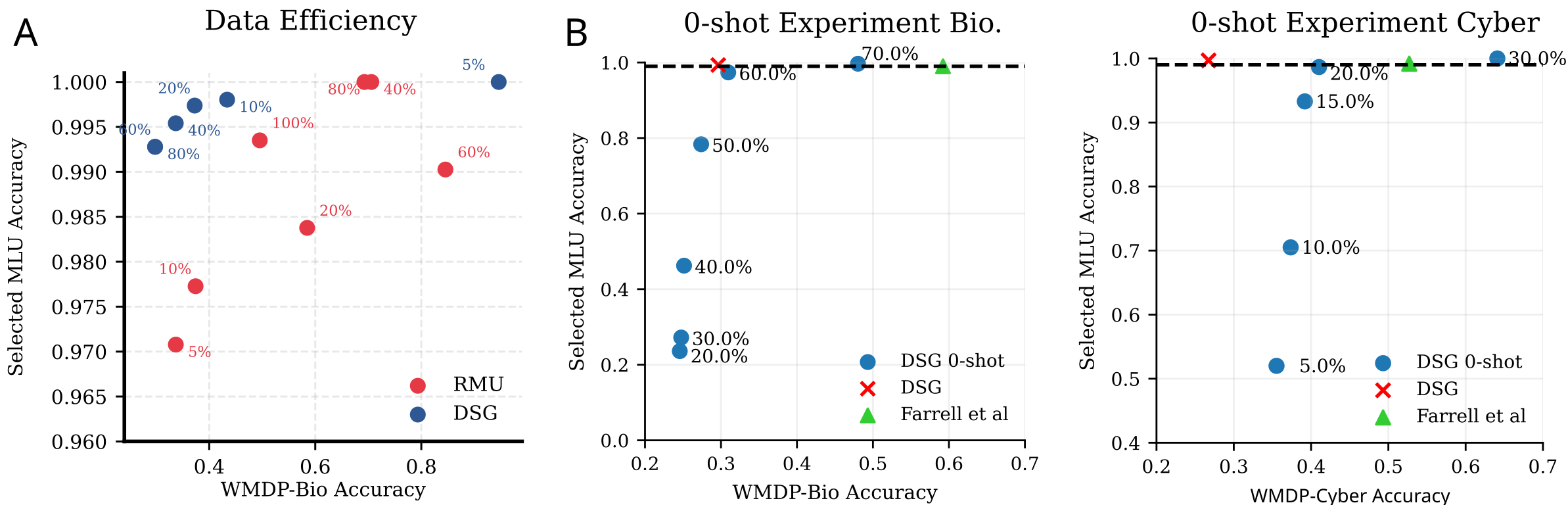}
    \vspace{-0.2in}
    \caption{\small Data efficiency analysis of DSG. (A) Performance across varying training data sizes compared to RMU. (B) Zero-shot performance on WMDP-Bio (left) and WMDP-Cyber (right) using 20 features selected via Neuropedia API with different $\tau$ thresholds (shown next to each data point).}
    \label{fig:zero_shot}
\end{figure}

\looseness=-1
We evaluate how DSG performs with limited forget and retain data on WMDP-Bio.
Figure \ref{fig:zero_shot}A shows DSG maintaining consistent performance when trained on 20-80\% of the original retain and forget datasets, preserving MMLU accuracy while keeping WMDP accuracy below 40\%. Only when dataset size falls below 20\% does effectiveness noticeably decline, with WMDP accuracy rising above 40\%. In contrast, RMU shows inconsistent results across different dataset sizes, indicating that gradient-based methods may be more \textbf{unstable to hyperparameter changes} when data is limited.

\looseness=-1
For \textbf{zero-shot evaluation}, we implement DSG without any domain-specific forget or retain data (Figure~\ref{fig:zero_shot}B), instead \textbf{leveraging the interpretability of SAEs}. We use Neuropedia \citep{neuronpedia} feature explanations to identify the forget set features by querying for concepts \textit{Biology} and \textit{Cybersecurity}, selecting the top 20 relevant features (details in \autoref{app:intepretability}). Both tasks use the \texttt{gemma-scope-2b-pt-res} SAE (width 16k) at layer 3 ($\ell_0$ 59). Since retain data is unavailable for dynamic threshold calibration, we sweep over static $\tau$ values, finding optimal settings ($\tau=60\%$ for WMDP-Bio, $\tau=20\%$ for WMDP-Cyber). Even with features selected purely based on their semantic descriptions and without dataset-specific tuning beyond $\tau$, these zero-shot DSG configurations outperform RMU and \citet{farrell2024applying}, demonstrating the potential for effective unlearning guided directly by feature interpretability.

\subsection{Ablations}

\begin{figure}[htbp]
    \centering
    \includegraphics[width=\textwidth]{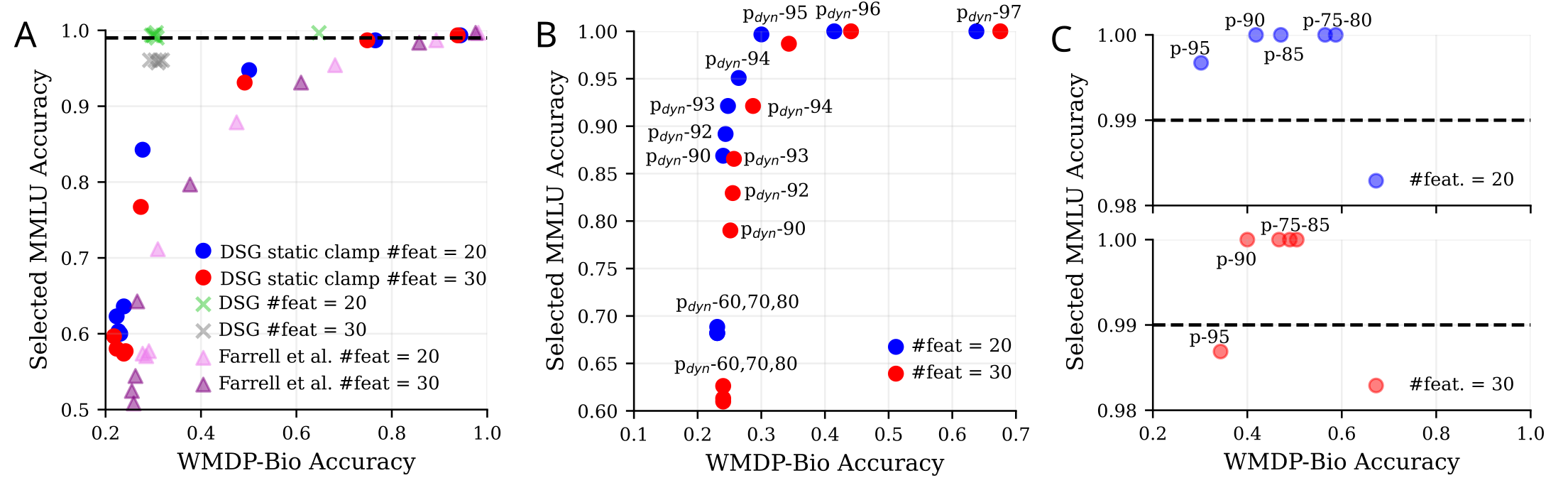} %
    \vspace{-0.2in}
    \caption{\small 
DSG Ablation studies (A) Static vs. dynamic clamping comparison with varying clamp strengths [10-500] for 20 and 30 features. (B) Effect of dynamic threshold percentile ($\smash{p_{\text{dyn}}}$) on performance (C) Impact of importance ratio threshold ($p_{\text{ratio}}$, range 75-95) for 20 and 30 features.}
    \label{fig:ablations}
\end{figure}

We evaluate each component of DSG by conducting ablation experiments on WMDP-Bio. 

\textbf{Dynamic Classification:} 
Figure~\ref{fig:ablations}A compares DSG with dynamic classification against DSG with static clamping from \citet{farrell2024applying} While static clamping effectively removes forget-set information at large clamp values ($c > 100$), it simultaneously reduces MMLU accuracy because it treats all inputs identically regardless of their forget-relevance. In contrast, our dynamic classifier only applies interventions when necessary based the statistical distribution of forget-feature activations. This conditional approach maintains higher MMLU accuracy ($>99\%$) while achieving comparable or better WMDP-Bio reduction.

\textbf{Percentile-Based Feature Selection:}  
DSG with static clamping leverages Fischer Information for feature selection instead of feature sparsity in \citet{farrell2024applying}. As shown in Figure~\ref{fig:ablations}A, across equivalent clamping strengths, this selection approach achieves 8\% lower WMDP-Bio accuracy on average while maintaining comparable MMLU performance, indicating more precise identification of forget-relevant features.

\textbf{Dynamic Threshold $p_{\text{dyn}}$:}  
\looseness=-1
Figure~\ref{fig:ablations}B shows the effect of $p_{\text{dyn}}$ on the forget-retain trade-off. Higher percentiles ($>95$) preserve more MMLU accuracy but allow more WMDP content to pass through undetected, while lower percentiles ($<90$) apply intervention more aggressively but with increased side effects on general knowledge. The optimal range 90-95 balances these considerations, removing targeted knowledge while minimizing side effects.

\textbf{Importance Ratio Threshold $p_{\text{ratio}}$:}  
As shown in Figure~\ref{fig:ablations}C, varying $p_{\text{ratio}}$ from 75-95 provides fine-grained control over feature selection. Higher values (95) select features with stronger forget-retain differentiation, yielding more targeted intervention, while lower values expand the feature set but may increase overlap with general knowledge features. Additionally we observed that the dynamic classifier can compensate for a lower $p_{\text{ratio}}$  maintaining effective forget-set filtering even when feature selection is less discriminative.

Additional ablations in \autoref{app:ablations} show that DSG is remarkably robust to clamp strength variations, and performs optimally with moderate feature counts. 
These findings highlight DSG's \textbf{practical hyperparameter stability.} Effective performance is maintained within reliable ranges for thresholds $p_{\text{dyn}}$/$p_{\text{ratio}}$ (90-95), feature counts (10-20), alongside notable robustness to clamp strength (100-500). Additionally, these hyperparameters transfer across datasets, simplifying deployment compared to gradient-based methods.

\vspace{-0.05in}
\section{Conclusion and Future Work}
\vspace{-0.05in}

In this work, we introduced DSG, demonstrating that SAEs with dynamic classification enable precise, activation-based unlearning that substantially outperforms gradient-based methods across multiple benchmarks. While our evaluation focuses on Gemma-2-2B where high-quality SAEs are available, extending DSG to larger models and diverse architectural families remains an important avenue for future work, as does investigating how performance scales with SAE width and training configurations.

\section*{Acknowledgements}
This work was supported in part by the National Science Foundation grants IIS2145670 and CCF2107024, and funding from Amazon, Apple, Google, Intel, Meta, the CyLab Security and Privacy Institute, and Leonardo Labs. Any opinions, findings and conclusions, or recommendations expressed in this material are those of the author(s) and do not necessarily reflect the views of any of these funding agencies. 
\vspace{-0.1in}

\section*{Ethics Statement}

This work introduces Dynamic SAE Guardrails (DSG), a method for targeted unlearning in large language models (LLMs). While designed to promote responsible AI by enabling the removal of unwanted knowledge, several ethical considerations arise:

\begin{itemize}[itemsep=2pt, topsep=1pt, leftmargin=12pt, rightmargin=1pt, labelsep=3pt, itemindent=5pt, parsep=2pt]
    \item \textbf{Potential for misuse:} While our focus is on removing hazardous or unwanted knowledge, the same technology could potentially be used to censor information or suppress viewpoints, leading to undesirable social consequences if deployed without careful oversight. The zero-shot capabilities, while advantageous for data-scarce scenarios, could be misused if the user-provided keywords are biased or used to target specific groups/content unfairly.
    \item \textbf{Over-reliance on interpretability:} Although SAEs offer improved interpretability compared to black-box models, feature interpretations are not always definitive or fully reliable. Misinterpreting feature roles or over-relying on imperfect interpretations could lead to unintended consequences, including the removal of valuable knowledge or the failure to remove harmful content. The quality of feature interpretation depends on the quality and representativeness of the data used to train and interpret the SAE.
    \item \textbf{Limitations of unlearning:} As with all approximate unlearning methods, DSG does not guarantee complete removal of targeted knowledge. As we show, it reduces the likelihood of the model generating outputs related to the forget set, but subtle traces or indirect influences might persist.  It is essential to acknowledge these limitations and avoid presenting DSG as a perfect solution for knowledge removal.
    \item \textbf{Dual-use concerns:} The techniques developed in this work for improving model control and safety could also be adapted by malicious actors to develop more sophisticated attacks or to create models that resist safety interventions.  We recognize this inherent dual-use nature and emphasize the need for responsible development and sharing of research findings.
    \item \textbf{Computational Cost of SAE Training:} The training of SAEs can be computationally demanding, raising environmental concerns. However there are several open-source SAEs, amortizing their cost, and the the inference-time efficiency of DSG offers some mitigation compared to gradient-based unlearning approaches.

\end{itemize}

We believe the benefits of precise, controllable unlearning for enhancing AI safety outweigh these risks, provided the technology is developed and deployed responsibly. We encourage future work to address these limitations and explore more robust evaluation methods for unlearning.

\section*{Reproducibility Statement}

To facilitate reproducibility, we have provided detailed descriptions of our experimental setups, including all relevant datasets, models, and hyperparameters.  Specifics for each set of experiments can be found as follows:

\begin{itemize}[itemsep=2pt, topsep=1pt, leftmargin=12pt, rightmargin=1pt, labelsep=3pt, itemindent=5pt, parsep=2pt]
    \item \textbf{WMDP Unlearning:} Complete hyperparameter settings for DSG and all baseline methods (GA, NPO, SSD, SCRUB, Farrell et al., RMU) are detailed in Appendix \ref{app:unlearning-wmdp-hyperparams}.  Model and SAE details are provided in Section \ref{sec:unlearningWMDP}.
    \item \textbf{MUSE Unlearning:} Hyperparameters for DSG and baseline methods (GA, GradDiff, NPO, SimNPO, RMU) are in Appendix \ref{app:unlearning-muse-hyperparams}, with model details in Section \ref{sec:methodology}.
    \item \textbf{Relearning Attacks (Section \ref{sec:relearning}):}  Model and hyperparameter configurations for the relearning experiments, including both train-time and test-time DSG interventions, are given in Appendix \ref{app:relearning_hyperparams}, along with the details in the main text.
    \item \textbf{Data Efficiency and Zero-shot Experiments:}  Model, SAE, and hyperparameter details for the data efficiency analysis and zero-shot evaluations are given in Appendix \ref{app:data-efficiency}.
    \item \textbf{Ablations:} All details related to the ablation studies, and chosen hyperparameters are in Appendix \ref{app:ablations}.
\end{itemize}

We have described the feature selection process, dynamic classification rule, and intervention mechanism in sufficient detail to allow for reimplementation (Algorithm \ref{alg:DSG} and Section \ref{sec:methodology}). We have also released the code and relevant scripts necessary to reproduce our results. 

\bibliography{colm2025_conference}

\begin{thebibliography}{46}
\providecommand{\natexlab}[1]{#1}
\providecommand{\url}[1]{\texttt{#1}}
\expandafter\ifx\csname urlstyle\endcsname\relax
  \providecommand{\doi}[1]{doi: #1}\else
  \providecommand{\doi}{doi: \begingroup \urlstyle{rm}\Url}\fi

\bibitem[Achiam et~al.(2023)Achiam, Adler, Agarwal, Ahmad, Akkaya, Aleman,
  Almeida, Altenschmidt, Altman, Anadkat, et~al.]{achiam2023gpt}
Josh Achiam, Steven Adler, Sandhini Agarwal, Lama Ahmad, Ilge Akkaya,
  Florencia~Leoni Aleman, Diogo Almeida, Janko Altenschmidt, Sam Altman,
  Shyamal Anadkat, et~al.
\newblock Gpt-4 technical report.
\newblock \emph{arXiv preprint arXiv:2303.08774}, 2023.

\bibitem[Bailey et~al.(2024)Bailey, Serrano, Sheshadri, Seleznyov, Taylor,
  Jenner, Hilton, Casper, Guestrin, and Emmons]{bailey2024obfuscated}
Luke Bailey, Alex Serrano, Abhay Sheshadri, Mikhail Seleznyov, Jordan Taylor,
  Erik Jenner, Jacob Hilton, Stephen Casper, Carlos Guestrin, and Scott Emmons.
\newblock Obfuscated activations bypass llm latent-space defenses.
\newblock \emph{arXiv preprint arXiv:2412.09565}, 2024.

\bibitem[Barez et~al.(2025)Barez, Fu, Prabhu, Casper, Sanyal, Bibi, O’Gara,
  Kirk, Bucknall, Fist, et~al.]{barez2025open}
Fazl Barez, Tingchen Fu, Ameya Prabhu, Stephen Casper, Amartya Sanyal, Adel
  Bibi, Aidan O’Gara, Robert Kirk, Ben Bucknall, Tim Fist, et~al.
\newblock Open problems in machine unlearning for ai safety, 2025.

\bibitem[Bourtoule et~al.(2021)Bourtoule, Chandrasekaran, Choquette-Choo, Jia,
  Travers, Zhang, Lie, and Papernot]{bourtoule2021machine}
Lucas Bourtoule, Varun Chandrasekaran, Christopher~A Choquette-Choo, Hengrui
  Jia, Adelin Travers, Baiwu Zhang, David Lie, and Nicolas Papernot.
\newblock Machine unlearning.
\newblock In \emph{2021 IEEE Symposium on Security and Privacy (SP)}, pp.\
  141--159. IEEE, 2021.

\bibitem[Bricken et~al.(2023)Bricken, Templeton, Batson, Chen, Jermyn, Conerly,
  Turner, Anil, Denison, Askell, Lasenby, Wu, Kravec, Schiefer, Maxwell,
  Joseph, Hatfield-Dodds, Tamkin, Nguyen, McLean, Burke, Hume, Carter,
  Henighan, and Olah]{bricken2023towards}
Trenton Bricken, Adly Templeton, Joshua Batson, Brian Chen, Adam Jermyn, Tom
  Conerly, Nick Turner, Cem Anil, Carson Denison, Amanda Askell, Robert
  Lasenby, Yifan Wu, Shauna Kravec, Nicholas Schiefer, Tim Maxwell, Nicholas
  Joseph, Zac Hatfield-Dodds, Alex Tamkin, Karina Nguyen, Brayden McLean,
  Josiah~E Burke, Tristan Hume, Shan Carter, Tom Henighan, and Christopher
  Olah.
\newblock Towards monosemanticity: Decomposing language models with dictionary
  learning.
\newblock \emph{Transformer Circuits Thread}, 2023.
\newblock
  https://transformer-circuits.pub/2023/monosemantic-features/index.html.

\bibitem[Bu et~al.(2024)Bu, Jin, Vinzamuri, Ramakrishna, Chang, Cevher, and
  Hong]{bu2024unlearning}
Zhiqi Bu, Xiaomeng Jin, Bhanukiran Vinzamuri, Anil Ramakrishna, Kai-Wei Chang,
  Volkan Cevher, and Mingyi Hong.
\newblock Unlearning as multi-task optimization: A normalized gradient
  difference approach with an adaptive learning rate.
\newblock \emph{arXiv preprint arXiv:2410.22086}, 2024.

\bibitem[Cao \& Yang(2015)Cao and Yang]{cao2015towards}
Yinzhi Cao and Junfeng Yang.
\newblock Towards making systems forget with machine unlearning.
\newblock In \emph{2015 IEEE symposium on security and privacy}, pp.\
  463--480. IEEE, 2015.

\bibitem[Cha et~al.(2024)Cha, Cho, Hwang, and Lee]{cha2024towards}
Sungmin Cha, Sungjun Cho, Dasol Hwang, and Moontae Lee.
\newblock Towards robust and cost-efficient knowledge unlearning for large
  language models.
\newblock \emph{arXiv preprint arXiv:2408.06621}, 2024.

\bibitem[Chvykov \& Hoel(2020)Chvykov and Hoel]{chvykov2020causal}
Pavel Chvykov and Erik Hoel.
\newblock Causal geometry.
\newblock \emph{Entropy}, 23\penalty0 (1):\penalty0 24, 2020.

\bibitem[Cunningham et~al.(2023)Cunningham, Ewart, Riggs, Huben, and
  Sharkey]{cunningham2023sparse}
Hoagy Cunningham, Aidan Ewart, Logan Riggs, Robert Huben, and Lee Sharkey.
\newblock Sparse autoencoders find highly interpretable features in language
  models.
\newblock \emph{arXiv preprint arXiv:2309.08600}, 2023.

\bibitem[Deeb \& Roger(2024)Deeb and Roger]{deeb2024unlearning}
Aghyad Deeb and Fabien Roger.
\newblock Do unlearning methods remove information from language model weights?
\newblock \emph{arXiv preprint arXiv:2410.08827}, 2024.

\bibitem[Elhage et~al.(2022)Elhage, Hume, Olsson, Schiefer, Henighan, Kravec,
  Hatfield-Dodds, Lasenby, Drain, Chen, et~al.]{elhage2022toy}
Nelson Elhage, Tristan Hume, Catherine Olsson, Nicholas Schiefer, Tom Henighan,
  Shauna Kravec, Zac Hatfield-Dodds, Robert Lasenby, Dawn Drain, Carol Chen,
  et~al.
\newblock Toy models of superposition.
\newblock \emph{arXiv preprint arXiv:2209.10652}, 2022.

\bibitem[Fan et~al.(2024)Fan, Liu, Lin, Jia, Zhang, Mei, and
  Liu]{fan2024simplicity}
Chongyu Fan, Jiancheng Liu, Licong Lin, Jinghan Jia, Ruiqi Zhang, Song Mei, and
  Sijia Liu.
\newblock Simplicity prevails: Rethinking negative preference optimization for
  llm unlearning.
\newblock \emph{arXiv preprint arXiv:2410.07163}, 2024.

\bibitem[Farrell et~al.(2024)Farrell, Lau, and Conmy]{farrell2024applying}
Eoin Farrell, Yeu-Tong Lau, and Arthur Conmy.
\newblock Applying sparse autoencoders to unlearn knowledge in language models.
\newblock In \emph{NeurIPS 2024 Safe Generative AI Workshop}, 2024.
\newblock URL \url{https://arxiv.org/abs/2410.19278}.

\bibitem[Foster et~al.(2024)Foster, Schoepf, and Brintrup]{foster2024fast}
Jack Foster, Stefan Schoepf, and Alexandra Brintrup.
\newblock Fast machine unlearning without retraining through selective synaptic
  dampening.
\newblock In \emph{Proceedings of the AAAI conference on artificial
  intelligence}, volume~38, pp.\  12043--12051, 2024.

\bibitem[Gao et~al.(2024)Gao, Wang, Weng, Wang, and Zhu]{gao2024practical}
Chongyang Gao, Lixu Wang, Chenkai Weng, Xiao Wang, and Qi~Zhu.
\newblock Practical unlearning for large language models.
\newblock \emph{arXiv preprint arXiv:2407.10223}, 2024.

\bibitem[Gao et~al.(2025)Gao, Wang, Ding, Weng, Wang, and
  Zhu]{gao2025largelanguagemodelcontinual}
Chongyang Gao, Lixu Wang, Kaize Ding, Chenkai Weng, Xiao Wang, and Qi~Zhu.
\newblock On large language model continual unlearning, 2025.
\newblock URL \url{https://arxiv.org/abs/2407.10223}.

\bibitem[Gokaslan et~al.(2019)Gokaslan, Cohen, Pavlick, and
  Tellex]{Gokaslan2019OpenWeb}
Aaron Gokaslan, Vanya Cohen, Ellie Pavlick, and Stefanie Tellex.
\newblock Openwebtext corpus.
\newblock \url{http://Skylion007.github.io/OpenWebTextCorpus}, 2019.

\bibitem[Hendrycks et~al.(2020)Hendrycks, Burns, Basart, Zou, Mazeika, Song,
  and Steinhardt]{hendrycks2020measuring}
Dan Hendrycks, Collin Burns, Steven Basart, Andy Zou, Mantas Mazeika, Dawn
  Song, and Jacob Steinhardt.
\newblock Measuring massive multitask language understanding.
\newblock \emph{arXiv preprint arXiv:2009.03300}, 2020.

\bibitem[Hu et~al.(2025)Hu, Fu, Wu, and Smith]{hu2024jogging}
Shengyuan Hu, Yiwei Fu, Zhiwei~Steven Wu, and Virginia Smith.
\newblock Jogging the memory of unlearned llms through targeted relearning
  attacks.
\newblock \emph{International Conference on Learning Representations}, 2025.

\bibitem[Jang et~al.(2023)Jang, Yoon, Yang, Cha, Lee, Logeswaran, and
  Seo]{jang2023knowledge}
Joel Jang, Dongkeun Yoon, Sohee Yang, Sungmin Cha, Moontae Lee, Lajanugen
  Logeswaran, and Minjoon Seo.
\newblock Knowledge unlearning for mitigating privacy risks in language models.
\newblock In \emph{Proceedings of the 61st Annual Meeting of the Association
  for Computational Linguistics (Volume 1: Long Papers)}, pp.\  14389--14408,
  2023.

\bibitem[Karvonen et~al.(2025)Karvonen, Rager, Lin, Tigges, Bloom, Chanin, Lau,
  Farrell, McDougall, Ayonrinde, et~al.]{karvonen2025saebench}
Adam Karvonen, Can Rager, Johnny Lin, Curt Tigges, Joseph Bloom, David Chanin,
  Yeu-Tong Lau, Eoin Farrell, Callum McDougall, Kola Ayonrinde, et~al.
\newblock Saebench: A comprehensive benchmark for sparse autoencoders in
  language model interpretability.
\newblock \emph{arXiv preprint arXiv:2503.09532}, 2025.

\bibitem[Kurmanji et~al.(2023)Kurmanji, Triantafillou, Hayes, and
  Triantafillou]{kurmanji2023towards}
Meghdad Kurmanji, Peter Triantafillou, Jamie Hayes, and Eleni Triantafillou.
\newblock Towards unbounded machine unlearning.
\newblock \emph{Advances in neural information processing systems},
  36:\penalty0 1957--1987, 2023.

\bibitem[Li et~al.(2024)Li, Pan, Gopal, Yue, Berrios, Gatti, Li, Dombrowski,
  Goel, Phan, et~al.]{li2024wmdp}
Nathaniel Li, Alexander Pan, Anjali Gopal, Summer Yue, Daniel Berrios, Alice
  Gatti, Justin~D. Li, Ann-Kathrin Dombrowski, Shashwat Goel, Long Phan, et~al.
\newblock The wmdp benchmark: Measuring and reducing malicious use with
  unlearning.
\newblock In \emph{International Conference on Machine Learning}, 2024.

\bibitem[Lieberum et~al.(2024)Lieberum, Rajamanoharan, Conmy, Smith, Sonnerat,
  Varma, Kram{\'a}r, Dragan, Shah, and Nanda]{lieberum2024gemma}
Tom Lieberum, Senthooran Rajamanoharan, Arthur Conmy, Lewis Smith, Nicolas
  Sonnerat, Vikrant Varma, J{\'a}nos Kram{\'a}r, Anca Dragan, Rohin Shah, and
  Neel Nanda.
\newblock Gemma scope: Open sparse autoencoders everywhere all at once on gemma
  2.
\newblock \emph{arXiv preprint arXiv:2408.05147}, 2024.

\bibitem[Lin(2023)]{neuronpedia}
Johnny Lin.
\newblock Neuronpedia: Interactive reference and tooling for analyzing neural
  networks, 2023.
\newblock URL \url{https://www.neuronpedia.org}.
\newblock Software available from neuronpedia.org.

\bibitem[Liu et~al.(2024)Liu, Yao, Jia, Casper, Baracaldo, Hase, Xu, Yao, Li,
  Varshney, et~al.]{wang2024rethinking}
Sijia Liu, Yuanshun Yao, Jinghan Jia, Stephen Casper, Nathalie Baracaldo, Peter
  Hase, Xiaojun Xu, Yuguang Yao, Hang Li, Kush~R. Varshney, et~al.
\newblock Rethinking machine unlearning for large language models.
\newblock In \emph{International Conference on Learning Representations},
  February 2024.
\newblock URL \url{https://openreview.net/forum?id=PLz7eW1K82}.

\bibitem[Liu et~al.(2025)Liu, Yao, Jia, Casper, Baracaldo, Hase, Yao, Liu, Xu,
  Li, et~al.]{liu2025rethinking}
Sijia Liu, Yuanshun Yao, Jinghan Jia, Stephen Casper, Nathalie Baracaldo, Peter
  Hase, Yuguang Yao, Chris~Yuhao Liu, Xiaojun Xu, Hang Li, et~al.
\newblock Rethinking machine unlearning for large language models.
\newblock \emph{Nature Machine Intelligence}, pp.\  1--14, 2025.

\bibitem[Liu et~al.(2022)Liu, Fan, Chen, Liu, Ma, Wang, and
  Ma]{liu2022backdoor}
Yang Liu, Mingyuan Fan, Cen Chen, Ximeng Liu, Zhuo Ma, Li~Wang, and Jianfeng
  Ma.
\newblock Backdoor defense with machine unlearning.
\newblock In \emph{IEEE INFOCOM 2022-IEEE conference on computer
  communications}, pp.\  280--289. IEEE, 2022.

\bibitem[{\L}ucki et~al.(2024){\L}ucki, Wei, Huang, Henderson, Tram{\`e}r, and
  Rando]{lucki2024adversarial}
Jakub {\L}ucki, Boyi Wei, Yangsibo Huang, Peter Henderson, Florian Tram{\`e}r,
  and Javier Rando.
\newblock An adversarial perspective on machine unlearning for ai safety.
\newblock \emph{arXiv preprint arXiv:2409.18025}, 2024.

\bibitem[Maini et~al.(2024)Maini, Feng, Schwarzschild, Lipton, and
  Kolter]{maini2024tofu}
Pratyush Maini, Zhili Feng, Avi Schwarzschild, Zachary~C Lipton, and J~Zico
  Kolter.
\newblock Tofu: A task of fictitious unlearning for llms, 2024.

\bibitem[Marks et~al.(2024)Marks, Rager, Michaud, Belinkov, Bau, and
  Mueller]{marks2024sparse}
Samuel Marks, Can Rager, Eric~J Michaud, Yonatan Belinkov, David Bau, and Aaron
  Mueller.
\newblock Sparse feature circuits: Discovering and editing interpretable causal
  graphs in language models.
\newblock \emph{arXiv preprint arXiv:2403.19647}, 2024.

\bibitem[Merity et~al.(2016)Merity, Xiong, Bradbury, and
  Socher]{merity2016pointer}
Stephen Merity, Caiming Xiong, James Bradbury, and Richard Socher.
\newblock Pointer sentinel mixture models.
\newblock \emph{arXiv preprint arXiv:1609.07843}, 2016.

\bibitem[O'Brien et~al.(2024)O'Brien, Majercak, Fernandes, Edgar, Chen, Nori,
  Carignan, Horvitz, and Poursabzi-Sangde]{o2024steering}
Kyle O'Brien, David Majercak, Xavier Fernandes, Richard Edgar, Jingya Chen,
  Harsha Nori, Dean Carignan, Eric Horvitz, and Forough Poursabzi-Sangde.
\newblock Steering language model refusal with sparse autoencoders.
\newblock \emph{arXiv preprint arXiv:2411.11296}, 2024.

\bibitem[Pearl(2009)]{pearl2009causality}
Judea Pearl.
\newblock \emph{Causality}.
\newblock Cambridge university press, 2009.

\bibitem[Rafailov et~al.(2023)Rafailov, Sharma, Mitchell, Manning, Ermon, and
  Finn]{rafailov2023direct}
Rafael Rafailov, Archit Sharma, Eric Mitchell, Christopher~D Manning, Stefano
  Ermon, and Chelsea Finn.
\newblock Direct preference optimization: Your language model is secretly a
  reward model.
\newblock \emph{Advances in Neural Information Processing Systems},
  36:\penalty0 53728--53741, 2023.

\bibitem[Rajamanoharan et~al.(2024)Rajamanoharan, Lieberum, Sonnerat, Conmy,
  Varma, Kram{\'a}r, and Nanda]{rajamanoharan2024jumping}
Senthooran Rajamanoharan, Tom Lieberum, Nicolas Sonnerat, Arthur Conmy, Vikrant
  Varma, J{\'a}nos Kram{\'a}r, and Neel Nanda.
\newblock Jumping ahead: Improving reconstruction fidelity with jumprelu sparse
  autoencoders.
\newblock \emph{arXiv preprint arXiv:2407.14435}, 2024.

\bibitem[Shen et~al.(2024)Shen, Zhang, Bialkowski, Chen, and Xu]{shen2024camu}
Shaofei Shen, Chenhao Zhang, Alina Bialkowski, Weitong Chen, and Miao Xu.
\newblock Camu: disentangling causal effects in deep model unlearning.
\newblock In \emph{Proceedings of the 2024 SIAM International Conference on
  Data Mining (SDM)}, pp.\  779--787. SIAM, 2024.

\bibitem[Shi et~al.(2024)Shi, Lee, Huang, Malladi, Zhao, Holtzman, Liu,
  Zettlemoyer, Smith, and Zhang]{shi2024muse}
Weijia Shi, Jaechan Lee, Yangsibo Huang, Sadhika Malladi, Jieyu Zhao, Ari
  Holtzman, Daogao Liu, Luke Zettlemoyer, Noah~A. Smith, and Chiyuan Zhang.
\newblock {MUSE}: Machine unlearning six-way evaluation for language models.
\newblock In \emph{Conference on Language Modeling Research}, July 2024.
\newblock URL \url{https://openreview.net/forum?id=gAfnQ8o9Cq}.

\bibitem[Shokri et~al.(2017)Shokri, Stronati, Song, and
  Shmatikov]{shokri2017membership}
Reza Shokri, Marco Stronati, Congzheng Song, and Vitaly Shmatikov.
\newblock Membership inference attacks against machine learning models.
\newblock In \emph{2017 IEEE symposium on security and privacy (SP)}, pp.\
  3--18. IEEE, 2017.

\bibitem[Thaker et~al.(2025)Thaker, Hu, Kale, Maurya, Wu, and
  Smith]{thaker2024position}
Pratiksha Thaker, Shengyuan Hu, Neil Kale, Yash Maurya, Zhiwei~Steven Wu, and
  Virginia Smith.
\newblock Position: Llm unlearning benchmarks are weak measures of progress.
\newblock \emph{IEEE Conference on Secure and Trustworthy Machine Learning},
  2025.

\bibitem[Xu et~al.(2024)Xu, Wu, Wang, and Jia]{xu2024machine}
Jie Xu, Zihan Wu, Cong Wang, and Xiaohua Jia.
\newblock Machine unlearning: Solutions and challenges.
\newblock \emph{IEEE Transactions on Emerging Topics in Computational
  Intelligence}, 2024.

\bibitem[Yao et~al.(2024)Yao, Chien, Du, Niu, Wang, Cheng, and
  Yue]{yao2024machine}
Jin Yao, Eli Chien, Minxin Du, Xinyao Niu, Tianhao Wang, Zezhou Cheng, and
  Xiang Yue.
\newblock Machine unlearning of pre-trained large language models, 2024.

\bibitem[Zhang et~al.(2024)Zhang, Lin, Bai, and Mei]{zhang2024negative}
Ruiqi Zhang, Licong Lin, Yu~Bai, and Song Mei.
\newblock Negative preference optimization: From catastrophic collapse to
  effective unlearning, 2024.

\bibitem[Zheng et~al.(2023)Zheng, Chiang, Sheng, Zhuang, Wu, Zhuang, Lin, Li,
  Li, Xing, et~al.]{zheng2023judging}
Lianmin Zheng, Wei-Lin Chiang, Ying Sheng, Siyuan Zhuang, Zhanghao Wu, Yonghao
  Zhuang, Zi~Lin, Zhuohan Li, Dacheng Li, Eric Xing, et~al.
\newblock Judging llm-as-a-judge with mt-bench and chatbot arena.
\newblock \emph{Advances in Neural Information Processing Systems},
  36:\penalty0 46595--46623, 2023.

\bibitem[Zhou et~al.(2023)Zhou, Liu, Xu, Iyer, Sun, Mao, Ma, Efrat, Yu, Yu,
  et~al.]{zhou2023lima}
Chunting Zhou, Pengfei Liu, Puxin Xu, Srinivasan Iyer, Jiao Sun, Yuning Mao,
  Xuezhe Ma, Avia Efrat, Ping Yu, Lili Yu, et~al.
\newblock Lima: Less is more for alignment.
\newblock \emph{Advances in Neural Information Processing Systems},
  36:\penalty0 55006--55021, 2023.

\end{thebibliography}
\bibliographystyle{colm2025_conference}

\appendix

\section{Additional Background and Related Work Details}
\label{app:background_details}

This appendix provides further details on concepts mentioned in the Background and Related Work (Section~\ref{sec:background}).

\subsection{Formal Goal of Unlearning}
\label{app:formal_unlearning}

As introduced in the main text, machine unlearning aims to transform a model $\mathcal{M}(D)$, initially trained on a dataset $D = D_{retain} \cup D_{forget}$, into an unlearned model $\mathcal{M}_{unlearn}$. The theoretical ideal of \emph{exact unlearning} requires that $\mathcal{M}_{unlearn}$ be computationally indistinguishable from a model $\mathcal{M}(D_{retain})$ that was trained exclusively on the retain set $D_{retain}$ from the beginning \citep{bourtoule2021machine, cao2015towards}. Due to the computational cost of retraining large language models from scratch, achieving exact unlearning is generally impractical. Therefore, the field primarily focuses on developing \emph{approximate unlearning} methods. These methods aim to satisfy specific criteria related to effectively removing the influence of $D_{forget}$ while preserving the model's performance on $D_{retain}$, without incurring the cost of full retraining \citep{liu2025rethinking}.

\subsection{Unlearning Evaluation Metrics and Benchmarks}
\label{app:eval_metrics}

The evaluation of approximate unlearning methods typically involves measuring two primary aspects: Forget Quality and Utility Preservation. \textbf{Forget Quality} quantifies the successful removal of information pertaining to the forget set $D_{forget}$. Common metrics include measuring the \emph{Forget Set Performance Degradation}, which involves observing reduced accuracy or increased loss on tasks specifically related to the content of $D_{forget}$ \citep{shi2024muse, maini2024tofu}. Another aspect is assessing \emph{Memorization Metrics}, which gauge the model's reduced ability to recall specific sequences or knowledge points verbatim from $D_{forget}$ \citep{shi2024muse}. Furthermore, \emph{Privacy Leakage Metrics} evaluate the decreased success rate of Membership Inference Attacks (MIAs) that try to infer whether a given data point was part of the original $D_{forget}$, often quantified using the Area Under the Curve (AUC) of the MIA classifier \citep{shokri2017membership, shi2024muse}.

Conversely, \textbf{Utility Preservation} assesses how well the unlearned model retains its general knowledge and capabilities on tasks unrelated to $D_{forget}$. This is commonly measured by evaluating \emph{Retain Set Performance Preservation}, which checks for maintained accuracy on standard academic or commonsense reasoning benchmarks such as MMLU \citep{hendrycks2020measuring}. Additionally, \emph{General Language Modeling Performance} is often assessed by ensuring minimal increase in perplexity or loss when the model processes large, general-purpose text corpora like OpenWebText \citep{Gokaslan2019OpenWeb} or WikiText \citep{merity2016pointer}. Finally, \emph{Fluency and Coherence} of the model's generated text are important, often evaluated through automated metrics, human judgment, or interaction with benchmark chatbots like MT-Bench \citep{zheng2023judging}. Standardized benchmarks like MUSE \citep{shi2024muse}, TOFU \citep{maini2024tofu}, WMDP \citep{li2024wmdp}, and SAEBench \citep{karvonen2025saebench} provide datasets, tasks, and evaluation protocols designed to measure performance across these diverse criteria.

\subsection{Gradient-Based Unlearning Methods}
\label{app:gradient_methods}

Gradient-based unlearning techniques directly modify the weights $\theta$ of the original model $\mathcal{M}(D)$ using optimization algorithms, typically variants of gradient descent or ascent. 

\textbf{Gradient Ascent (GA)} represents a basic approach where the optimization objective is to maximize the loss function (e.g., negative log-likelihood) on the forget set $D_{forget}$, thereby pushing the model parameters away from configurations that accurately represent this data \citep{jang2023knowledge}. This method, however, often suffers from catastrophic forgetting of useful knowledge if not carefully regularized.

\textbf{Gradient Difference (GradDiff or NegGrad)} attempts to balance forgetting and retention by computing gradients for both minimizing loss on $D_{retain}$ and maximizing loss on $D_{forget}$, then applying an update based on a combination (often a subtraction) of these gradients \citep{liu2022backdoor}. 

\textbf{Negative Preference Optimization (NPO)} leverages insights from preference-based finetuning methods like DPO \citep{rafailov2023direct}, reformulating unlearning as learning to disprefer outputs related to $D_{forget}$ relative to some reference, which could be outputs from the original model or data from $D_{retain}$ \citep{zhang2024negative}. Simplified variants like SimNPO aim to reduce the computational overhead \citep{fan2024simplicity}. 

\textbf{Representation Misdirection Unlearning (RMU)} operates by injecting noise or applying targeted shifts to the internal activations of the model at specific layers, but only when processing inputs related to $D_{forget}$, while simultaneously using a regularization term to keep activations on $D_{retain}$ close to those of the original model \citep{li2024wmdp}. 

\textbf{Selective Synaptic Dampening (SSD)} aims for more targeted weight modification by estimating the importance of individual parameters for both $D_{forget}$ and $D_{retain}$ (using approximations based on Fisher information) and then selectively reducing the magnitude of parameters found to be more critical for $D_{forget}$ than for $D_{retain}$ \citep{foster2024fast}. 

\textbf{SCRUB} employs a student-teacher knowledge distillation framework; it trains a copy of the original model (the student) to diverge from the original frozen model (the teacher) on $D_{forget}$ inputs (typically by maximizing KL divergence) while simultaneously encouraging the student to mimic the teacher on $D_{retain}$ inputs (by minimizing KL divergence) \citep{kurmanji2023towards}. 

Finally, many gradient-based methods incorporate explicit \textbf{Regularization Techniques} to counteract the tendency towards catastrophic forgetting. Common regularizers include minimizing the KL divergence between the probability distributions of the unlearned and original models when evaluated on $D_{retain}$ \citep{maini2024tofu}, or directly including a term in the loss function that minimizes the model's prediction error on $D_{retain}$ \citep{yao2024machine}.

\subsection{Prior SAE Unlearning Work}
\label{app:prior_sae}

The work by \citet{farrell2024applying} is an early exploration into using Sparse Autoencoders (SAEs) for machine unlearning. We describe their methodology here.

First, they computed the activation sparsity for each feature in the SAE dictionary, calculated separately over the forget dataset ($D_{forget}$) and the retain dataset ($D_{retain}$). Sparsity here refers to the fraction of input tokens for which a given feature has a non-zero activation. 

Second, to mitigate potential damage to the model's general capabilities, they filtered out any features whose activation sparsity on the retain set $D_{retain}$ exceeded a predetermined threshold (e.g., a feature active on more than 1\% of retain tokens might be excluded). 

Third, from the pool of features that passed the retain-sparsity filter, they selected the top-$N$ features that exhibited the highest activation sparsity when measured on the forget set $D_{forget}$. The assumption was that features frequently active on forget data are likely responsible for encoding the knowledge to be removed.

Fourth, they implemented a \emph{static} intervention mechanism during inference: whenever any of the top-$N$ selected features $f_j$ produced a positive activation ($f_j(\mathbf{h}_t) > 0$) for any token $t$, its activation was clamped to a fixed negative value (e.g., -c). This clamping was applied universally, regardless of the overall context of the input sequence. 

This combination of sparsity-based feature selection and static clamping ultimately proved limiting, leading to significant side effects on utility and performance inferior to contemporary gradient-based methods like RMU on benchmarks such as WMDP-Bio. Recognizing these limitations directly motivated our work (DSG), where we instead develop and apply principled feature selection and dynamic, context-aware interventions.

\subsection{Relearning Attacks}
\label{app:relearning_details}

Approximate unlearning methods, especially those modifying model weights, face another significant challenge: \textbf{relearning attacks} \citep{deeb2024unlearning}. In these attacks, an adversary finetunes the unlearned model $\mathcal{M}_{unlearn}$ to recover the supposedly forgotten information. Such recovery is sometimes possible even using only data tangentially related to the original forget set $D_{forget}$ \citep{hu2024jogging}. The success of relearning attacks suggests that gradient-based weight modifications may primarily suppress access to knowledge rather than truly erasing it from the parameter space; subsequent finetuning can often reverse these weight adjustments, particularly if it reinforces the target concepts.

The feasibility of relearning attacks strongly depends on the threat model. In an \textbf{API-based (black-box) setting}, where adversaries only have query access, mounting effective relearning attacks is more difficult, particularly if the provider restricts extensive finetuning or monitors queries. Activation-based intervention methods like DSG, which modify internal states rather than weights to control outputs for relevant inputs, may offer greater robustness in this black-box scenario compared to weight modification techniques.

Although sophisticated \textbf{obfuscation attacks} targeting activation-based defenses exist \citep{bailey2024obfuscated}, they typically require white-box access (e.g., gradients, internal states). Such access is unavailable in a pure API setting, limiting their threat against deployed systems focused on output safety via activation manipulation. DSG's potential resilience against relearning could stem from the relative stability of activation geometry during standard finetuning, a phenomenon related to the \textbf{Superficial Alignment Hypothesis} \citep{zhou2023lima}. If DSG reliably identifies features encoding the target knowledge based on these stable patterns and consistently applies interventions, the unlearning effect may prove more durable against finetuning-based relearning attacks compared to methods reliant on weight configurations.

\section{Fisher Information Approximation Proof}
\label{app:fisher_proof}

\setcounter{theorem}{0}
\begin{theorem}[Approximate Fisher Information from SAE Features]
\label{thm:fisher_detail_appendix}
Let a sparse autoencoder (SAE) with reconstruction $\hat{\mathbf{r}}(x) = \mathbf{z}(x)\mathbf{W}_{\mathrm{dec}}$ be applied to data $x \sim \mathcal{D}$, where $\mathbf{z}(x) \in \mathbb{R}^F$ represents latent activations and $\mathbf{W}_{\mathrm{dec}} \in \mathbb{R}^{F \times D}$ the decoder weights. Define the reconstruction loss as:
\begin{equation*}
\ell(x) = \frac{1}{2}\|\hat{\mathbf{r}}(x) - \mathbf{r}(x)\|^2
\end{equation*}

Fix the encoder and assume $\mathbb{E}[z_i^2]<\infty$. For the Gaussian decoder likelihood $\mathbf{R}\mid\mathbf{z}\sim\mathcal{N}(\mathbf{z}\mathbf{W}_{\mathrm{dec}},\sigma^2\mathbf{I}_D)$ with fixed $\sigma>0$ and unconstrained decoder weights, the Fisher block for row $\boldsymbol{\theta}_{i,\cdot}$ satisfies
\[
\operatorname{tr}\mathcal{I}(\boldsymbol{\theta}_{i,\cdot})=\frac{D}{\sigma^2}\mathbb{E}[z_i^2].
\]
\end{theorem}

\begin{proof}
We establish this result through careful analysis of the gradient structure in sparse autoencoders.

\paragraph{Computing the Gradient of Decoder Weights.}
By definition of the reconstruction loss:
\begin{align*}
\ell(x) &= \frac{1}{2}\|\hat{\mathbf{r}}(x) - \mathbf{r}(x)\|^2 \\
&= \frac{1}{2}\|\mathbf{z}(x)\mathbf{W}_{\mathrm{dec}} - \mathbf{r}(x)\|^2
\end{align*}

For row $i$ of $\mathbf{W}_{\mathrm{dec}}$, denoted $\boldsymbol{\theta}_{i,\cdot} \in \mathbb{R}^D$, we compute the gradient:
\begin{align*}
\nabla_{\boldsymbol{\theta}_{i,\cdot}}\ell(x) &= \nabla_{\boldsymbol{\theta}_{i,\cdot}}\left[\frac{1}{2}\|\mathbf{z}(x)\mathbf{W}_{\mathrm{dec}} - \mathbf{r}(x)\|^2\right]
\end{align*}

By the chain rule:
\begin{align*}
\nabla_{\boldsymbol{\theta}_{i,\cdot}}\ell(x) &= (\mathbf{z}(x)\mathbf{W}_{\mathrm{dec}} - \mathbf{r}(x)) \cdot \nabla_{\boldsymbol{\theta}_{i,\cdot}}(\mathbf{z}(x)\mathbf{W}_{\mathrm{dec}})
\end{align*}

Since $\mathbf{z}(x)\mathbf{W}_{\mathrm{dec}}$ is linear in $\boldsymbol{\theta}_{i,\cdot}$ with coefficient $z_i(x)$, we have:
\begin{align*}
\nabla_{\boldsymbol{\theta}_{i,\cdot}}(\mathbf{z}(x)\mathbf{W}_{\mathrm{dec}}) &= z_i(x) \cdot \mathbf{I}_D
\end{align*}
where $\mathbf{I}_D$ is the $D$-dimensional identity matrix. Therefore:
\begin{align*}
\nabla_{\boldsymbol{\theta}_{i,\cdot}}\ell(x) &= z_i(x)(\hat{\mathbf{r}}(x) - \mathbf{r}(x))
\end{align*}

\paragraph{Computing the Squared Gradient Norm.}
Taking the squared norm of this gradient:
\begin{align*}
\|\nabla_{\boldsymbol{\theta}_{i,\cdot}}\ell(x)\|^2 &= \|z_i(x)(\hat{\mathbf{r}}(x) - \mathbf{r}(x))\|^2 \\
&= z_i(x)^2\|\hat{\mathbf{r}}(x) - \mathbf{r}(x)\|^2
\end{align*}

Taking the expectation over the data distribution:
\begin{align*}
\mathbb{E}_{x \sim \mathcal{D}}[\|\nabla_{\boldsymbol{\theta}_{i,\cdot}}\ell(x)\|^2] &= \mathbb{E}_{x \sim \mathcal{D}}[z_i(x)^2\|\hat{\mathbf{r}}(x) - \mathbf{r}(x)\|^2]
\end{align*}

\paragraph{Analyzing the Small Error Regime.}
If $A=\{\|\hat{\mathbf{r}}-\mathbf{r}\|^2<\epsilon^2\}$ has probability at least $1-\delta$ and $\mathbb{E}[z_i^2\|\hat{\mathbf{r}}-\mathbf{r}\|^2\mathbf{1}_{A^c}]\le C\delta$, then
\begin{align*}
\mathbb{E}[z_i^2\|\hat{\mathbf{r}}-\mathbf{r}\|^2]
&\le \epsilon^2\mathbb{E}[z_i^2]+C\delta.
\end{align*}
This is an upper bound; small reconstruction error alone does not imply proportionality.

\paragraph{Connection to Fisher Information.}
For the Gaussian decoder likelihood in the theorem, the log-likelihood gradient is
\[
\nabla_{\boldsymbol{\theta}_{i,\cdot}}\log p(\mathbf{R}\mid\mathbf{z})
=\frac{z_i}{\sigma^2}(\mathbf{R}-\mathbf{z}\mathbf{W}_{\mathrm{dec}}).
\]
Taking the Fisher expectation over $\mathbf{z}$ and $\mathbf{R}\mid\mathbf{z}$, whose residual covariance is $\sigma^2\mathbf{I}_D$, gives
\begin{align*}
\mathcal{I}(\boldsymbol{\theta}_{i,\cdot})&=\frac{\mathbb{E}[z_i^2]}{\sigma^2}\mathbf{I}_D,\\
\operatorname{tr}\mathcal{I}(\boldsymbol{\theta}_{i,\cdot})&=\frac{D}{\sigma^2}\mathbb{E}[z_i^2].
\end{align*}

\paragraph{Interpretation.}
The above analysis shows that $(f_j(x))^2 = z_j(x)^2$ serves as a natural importance measure for feature $j$. Features with larger average squared activations have higher decoder Fisher Information under this model.
 This justifies our approach of using squared activations to identify features most strongly associated with specific knowledge domains.
\end{proof}

\section{Connecting Fisher Information to Causal Influence}
\label{app:causal_connection}

In this section, we establish how the Fisher Information associated with Sparse Autoencoder (SAE) features connects to their causal influence as mediators of information flow in language models. Drawing inspiration from causal geometry \citep{chvykov2020causal}, we provide a proof for why expected squared activation serves as a measure of feature importance.

\looseness=-1
\begin{theorem}[Fisher Information as a Proxy for Causal Feature Importance]
\label{thm:fisher_causal_influence_appendix}
Under the affine Gaussian intervention model specified below, for any fixed feature $f_j$, the expected squared activation $\mathbb{E}_{\mathcal{D}}[f_j(\mathbf{h})^2]$ is proportional to its zero-ablation influence on model outputs, with a feature-specific proportionality constant.
\end{theorem}

\begin{proof}
We build upon the result in Appendix~\ref{app:fisher_proof}, which showed that the expected squared activation $\mathbb{E}[f_j(\mathbf{h})^2]$ is proportional to the trace of the Fisher Information Matrix for the corresponding decoder weights.

\paragraph{Causal Model Setup.}
Consider a language model (LM) that produces hidden states $\mathbf{h}(x) \in \mathbb{R}^d$.  A Sparse Autoencoder (SAE) encodes $\mathbf{h}(x)$ into feature activations $\mathbf{z} = f(\mathbf{h}) \in \mathbb{R}^{d_{\text{SAE}}}$, i.e.\ each feature is $f_j\bigl(\mathbf{h}(x)\bigr)$. Let $\mathcal{D}_{\text{forget}}$ and $\mathcal{D}_{\text{retain}}$ be two subsets of the training data. We model the causal structure as:
\begin{align*}
  \text{Data} 
  \;\longrightarrow\; 
  \mathbf{h}(x) 
  \;\longrightarrow\; 
  \mathbf{z}=f(\mathbf{h}) 
  \;\longrightarrow\; 
  \mathbf{Y}\ (\text{model outputs})
\end{align*}
Here, $\mathbf{Y} \in  \mathbb{R}^{d_Y}$ represents the model's output vector (e.g., logits or embeddings). We measure inference-time interventions on inputs from each dataset, with model parameters held fixed.

\paragraph{Assumptions.}
\begin{enumerate} [itemsep=2pt, topsep=1pt, leftmargin=12pt, rightmargin=1pt, labelsep=3pt, itemindent=5pt, parsep=2pt]
\item \textbf{Affine Gaussian response.}
   For a fixed input $x$, $p(\mathbf{Y}\mid x,do(f_j=a))=\mathcal{N}(\mathbf{b}_j(x)+a\mathbf{u}_j,\sigma^2\mathbf{I})$ over the range from zero to the observed activation. The sensitivity $\mathbf{u}_j$ and $\sigma>0$ are constant across contexts and datasets, and the activation second moments are finite.

\item \textbf{Well-defined feature interventions.}
   We can perform $do\bigl(f_j=\alpha\bigr)$, meaning forcibly setting feature $j$ to $\alpha$ and thus severing its normal dependence on $\mathbf{h}$, while holding the other features and reconstruction residual fixed.

\item \textbf{Small SAE reconstruction error.}
   Writing $\hat{\mathbf{h}}(\mathbf{z})\approx \mathbf{W}\,\mathbf{z}$, we assume $\|\mathbf{h}-\hat{\mathbf{h}}(\mathbf{z})\|$ is small with high probability.

\item \textbf{Approximate feature independence.}
   Features $f_j(\mathbf{h})$ are sufficiently sparse or decorrelated that cross-terms can be neglected.
\end{enumerate}

\paragraph{Defining Causal Influence.}
We quantify the causal influence of feature $f_j$ by how much the model's output distribution $p(\mathbf{Y})$ changes when we intervene to set $f_j$ to its normal value $f_j(\mathbf{h})$ vs.\ forcing it to zero:
\begin{align*}
  \text{Influence}(f_j)
  &= \mathbb{E}_{x\sim\mathcal{D}}\Bigl[
    D_{\mathrm{KL}}\!\Bigl(\,
       p\bigl(\mathbf{Y} \,\mid\, x, do(f_j = f_j(\mathbf{h}))\bigr)
       \;\big\|\;
       p\bigl(\mathbf{Y} \,\mid\, x, do(f_j = 0)\bigr)
    \Bigr)
  \Bigr]
\end{align*}
A large KL means toggling $f_j$ from 0 to its actual value drastically shifts $p(\mathbf{Y})$, so $f_j$ is a strong mediator for $\mathcal{D}$.

\paragraph{Expanding KL Divergence.}
For a fixed input $x$, let $\mathbf{g}_{x,j}(\alpha)=\mathbf{b}_j(x)+\alpha\mathbf{u}_j$ describe the output mean under intervention. By the Gaussian response assumption,
\begin{align*}
   p\!\bigl(\mathbf{Y} \mid x, do(f_j=\alpha)\bigr)
   &= \mathcal{N}\bigl(\mathbf{g}_{x,j}(\alpha),\;\sigma^2\mathbf{I}\bigr)
\end{align*}

For two different interventions $do(f_j = \alpha)$ and $do(f_j = \beta)$, we can now derive the KL divergence between the resulting output distributions. Using the standard formula for KL divergence between multivariate Gaussians with the same covariance matrix:
\begin{align*}
D_{\mathrm{KL}}(\mathcal{N}(\boldsymbol{\mu}_1, \boldsymbol{\Sigma}) \| \mathcal{N}(\boldsymbol{\mu}_2, \boldsymbol{\Sigma})) &= \frac{1}{2}(\boldsymbol{\mu}_1 - \boldsymbol{\mu}_2)^T \boldsymbol{\Sigma}^{-1} (\boldsymbol{\mu}_1 - \boldsymbol{\mu}_2)
\end{align*}

Therefore:
\begin{align*}
  &D_{\mathrm{KL}}\!\Bigl(
    p\!\bigl(\mathbf{Y}\!\mid x, do(f_j=\alpha)\bigr)
    \,\big\|\,
    p\!\bigl(\mathbf{Y}\!\mid x, do(f_j=\beta)\bigr)
  \Bigr)\\
  &= D_{\mathrm{KL}}(\mathcal{N}(\mathbf{g}_{x,j}(\alpha), \sigma^2\mathbf{I}) \| \mathcal{N}(\mathbf{g}_{x,j}(\beta), \sigma^2\mathbf{I}))\\
  &= \frac{1}{2}(\mathbf{g}_{x,j}(\alpha) - \mathbf{g}_{x,j}(\beta))^T (\sigma^2\mathbf{I})^{-1} (\mathbf{g}_{x,j}(\alpha) - \mathbf{g}_{x,j}(\beta))\\
  &= \frac{1}{2\sigma^2}(\mathbf{g}_{x,j}(\alpha) - \mathbf{g}_{x,j}(\beta))^T (\mathbf{g}_{x,j}(\alpha) - \mathbf{g}_{x,j}(\beta))\\
  &= \frac{1}{2\sigma^2}\,\bigl\|\,\mathbf{g}_{x,j}(\alpha)-\mathbf{g}_{x,j}(\beta)\bigr\|^2
\end{align*}

\paragraph{Affine Response.}
By assumption, $\mathbf{g}_{x,j}(\alpha)-\mathbf{g}_{x,j}(0)=\alpha\mathbf{u}_j$. Setting $\beta=0$ therefore gives
\begin{align*}
   D_{\mathrm{KL}}\bigl(p(\mathbf{Y}|x,do(f_j=\alpha))\|p(\mathbf{Y}|x,do(f_j=0))\bigr)
   &=\frac{\alpha^2}{2\sigma^2}\|\mathbf{u}_j\|^2.
\end{align*}

\paragraph{Expected Causal Influence.}
Now we can compute the expected causal influence by substituting $\alpha = f_j(\mathbf{h})$ and taking the expectation over $\mathbf{h} \sim \mathcal{D}$:
\begin{align*}
  \text{Influence}(f_j)
  &= \mathbb{E}_{x\sim\mathcal{D}}\,\Bigl[
    D_{\mathrm{KL}}\bigl(p(\mathbf{Y}|x,do(f_j = f_j(\mathbf{h}))) \| p(\mathbf{Y}|x,do(f_j = 0))\bigr)
  \Bigr]\\
  &= \mathbb{E}_{x\sim\mathcal{D}}\,\Bigl[
    \frac{1}{2\sigma^2}\,f_j(\mathbf{h})^2\,\bigl\|\mathbf{u}_j\bigr\|^2
  \Bigr]\\
  &= \frac{\|\mathbf{u}_j\|^2}{2\sigma^2}\;
  \mathbb{E}_{x\sim\mathcal{D}}\!\Bigl[f_j(\mathbf{h})^2\Bigr]
\end{align*}

Thus, the expected causal influence of feature $j$ as a mediator of information from dataset $\mathcal{D}$ is directly proportional to the expected squared activation $\,\mathbb{E}_{\mathcal{D}}\bigl[f_j(\mathbf{h})^2\bigr]$, with a proportionality constant $\frac{\|\mathbf{u}_j\|^2}{2\sigma^2}$ that depends on the sensitivity of the model's outputs to changes in feature $j$.

\paragraph{Connection to Fisher Information.}
The decoder Fisher Information in Appendix~\ref{app:fisher_proof} and the zero-ablation influence derived here are both proportional to $\mathbb{E}[f_j(\mathbf{h})^2]$ for a fixed feature under their respective models. The influence coefficient $\|\mathbf{u}_j\|^2/(2\sigma^2)$ can differ across features, so activation magnitude alone need not rank different features by causal influence.

\paragraph{Implications for Feature Selection.}
By identifying features with high squared activations on $\mathcal{D}_{\text{forget}}$ but low activations on $\mathcal{D}_{\text{retain}}$, we target mediators with high relative forget-set influence under this model. For DSG's negative clamp, extending the affine model to $-c$ gives squared intervention magnitude $(f_j+c)^2$ in place of $f_j^2$.

\paragraph{Comparisons Across Datasets.}
For $\,\mathcal{D}_{\text{forget}}$ vs.\ $\,\mathcal{D}_{\text{retain}}$, we earlier defined:
\begin{align*}
  \texttt{forget\_score}(j) &= \mathbb{E}_{\mathcal{D}_{\text{forget}}}\!\bigl[f_j(\mathbf{h})^2\bigr]\\
  \texttt{retain\_score}(j) &= \mathbb{E}_{\mathcal{D}_{\text{retain}}}\!\bigl[f_j(\mathbf{h})^2\bigr]
\end{align*}

Write $\text{Influence}_{\mathcal{D}}(f_j)$ for the expected influence on dataset $\mathcal{D}$. For $\|\mathbf{u}_j\|>0$ and a positive retain second moment, the ratio is:
\begin{align*}
  \frac{\text{Influence}_{\mathcal{D}_{\text{forget}}}(f_j)}{
        \text{Influence}_{\mathcal{D}_{\text{retain}}}(f_j)}
  &= \frac{\frac{\|\mathbf{u}_j\|^2}{2\sigma^2} \cdot \mathbb{E}_{\mathcal{D}_{\text{forget}}}\![f_j(\mathbf{h})^2]}{\frac{\|\mathbf{u}_j\|^2}{2\sigma^2} \cdot \mathbb{E}_{\mathcal{D}_{\text{retain}}}\![f_j(\mathbf{h})^2]}\\
  &= \frac{\mathbb{E}_{\mathcal{D}_{\text{forget}}}[f_j(\mathbf{h})^2]}{
        \mathbb{E}_{\mathcal{D}_{\text{retain}}}[f_j(\mathbf{h})^2]}\\
  &= \frac{\texttt{forget\_score}(j)}{\texttt{retain\_score}(j)}
\end{align*}

Thus, under the stated model, $\texttt{forget\_score}(j)/\texttt{retain\_score}(j)$ equals the relative zero-ablation influence of feature $j$ across datasets. This is the importance ratio in Section~\ref{sec:methodology} when the denominator is not clipped by $\varepsilon$.
\end{proof}

\section{Proof of Neyman-Pearson Optimality}
\label{app:neyman_pearson_proof}

\begin{theorem}[Neyman-Pearson Optimality]
\label{thm:neyman_pearson_appendix}
Let $p_F$ and $p_R$ be the sequence densities for $\mathcal{D}_{\text{forget}}$ and $\mathcal{D}_{\text{retain}}$ on a common support with $p_R>0$. Suppose $p_F(x)/p_R(x)=g(\rho(x))$ for a nonnegative, nondecreasing function $g$. If a threshold $\tau^*$ satisfies $\Pr_R[\rho(X)>\tau^*]=\alpha$, then the threshold test $C^*(x)=\mathbf{1}[\rho(x)>\tau^*]$ maximizes the true-positive rate among all classifiers with false-positive rate at most $\alpha$.
\end{theorem}

\begin{proof}
We adapt the classical Neyman-Pearson Lemma to our unlearning context. Our goal is to find the optimal decision rule for classifying inputs as either forget-relevant or retain-relevant.

Consider the class of all decision rules $a: \mathcal{X} \rightarrow \{\text{clamp}, \text{no-clamp}\}$ with false-positive rate at most $\alpha$. That is, all rules $a$ such that:
\begin{equation*}
\Pr_{X \sim \mathcal{D}_{\text{retain}}}[a(X) = \text{clamp}] \leq \alpha
\end{equation*}

For each decision rule $a$, define its acceptance region $A = \{x \in \mathcal{X} : a(x) = \text{clamp}\}$. The constraint on false-positive rate translates to $\Pr_{X \sim \mathcal{D}_{\text{retain}}}[A] \leq \alpha$.

Now, define the threshold-based decision rule $a^*$ as:
\begin{equation*}
a^*(x) = \mathbf{1}[\rho(x) > \tau^*]
\end{equation*}
where $\tau^*$ is chosen such that $\Pr_{X \sim \mathcal{D}_{\text{retain}}}[\rho(X) > \tau^*] = \alpha$. The acceptance region for this rule is $A^* = \{x : \rho(x) > \tau^*\}$.

We need to prove that $a^*$ maximizes the true-positive rate among all rules with false-positive rate at most $\alpha$. In other words, for any rule $a$ with $\Pr_{X \sim \mathcal{D}_{\text{retain}}}[a(X) = \text{clamp}] \leq \alpha$, we must show:
\begin{equation*}
\Pr_{X \sim \mathcal{D}_{\text{forget}}}[a(X) = \text{clamp}] \leq \Pr_{X \sim \mathcal{D}_{\text{forget}}}[a^*(X) = \text{clamp}]
\end{equation*}

Set $k=g(\tau^*)$. By the likelihood-ratio assumption, $p_F-kp_R\ge0$ on $A^*$ and $p_F-kp_R\le0$ outside $A^*$. Hence, for any competing acceptance region $A$,
\[
(\mathbf{1}_{A^*}-\mathbf{1}_A)(p_F-kp_R)\ge0.
\]
Integrating and using $\Pr_R[A]\le\alpha=\Pr_R[A^*]$ gives
\[
\Pr_F[A^*]-\Pr_F[A]\ge k(\Pr_R[A^*]-\Pr_R[A])\ge0.
\]

Therefore, for any decision rule $a$ with false-positive rate at most $\alpha$:
\begin{equation*}
\Pr_{X \sim \mathcal{D}_{\text{forget}}}[a(X) = \text{clamp}] = \Pr_{X \sim \mathcal{D}_{\text{forget}}}[A] \leq \Pr_{X \sim \mathcal{D}_{\text{forget}}}[A^*] = \Pr_{X \sim \mathcal{D}_{\text{forget}}}[a^*(X) = \text{clamp}]
\end{equation*}

This establishes that the threshold test $a^*(x) = \mathbf{1}[\rho(x) > \tau^*]$ maximizes the true-positive rate among all tests with false-positive rate at most $\alpha$.
\end{proof}

\textbf{Practical Implications:} Under the stated likelihood-ratio assumption, this theorem establishes the statistical optimality of our thresholding approach for making the binary decision of whether to apply an intervention. In particular, it shows that our dynamic classification rule maximizes coverage on forget-set queries while maintaining a controlled false-positive rate on retain-set queries. This assumption is stronger than stochastic dominance. If $\rho$ has atoms, attaining an exact false-positive rate can require boundary randomization. DSG uses an empirical retain-set percentile to estimate the population threshold; population and transfer false-positive rates require evaluation.

\section{Proof of Dynamic Clamping Dominance}
\label{app:dominance_proof}

\begin{theorem}[Dominance of Dynamic Clamping]
\label{appthm:dominance}
Let $S_{n_{\text{feats}}}$ be a fixed subset of features identified as forget-relevant. Consider the static approach $a_{\text{static}}(x)$ that clamps features in $S_{n_{\text{feats}}}$ whenever they activate, and the dynamic approach $a_{\text{dynamic}}(x)$ that first classifies input $x$ using $C(x) = \mathbf{1}[\rho(x) > \tau]$ and only then applies clamping. Compute $\rho$ from the same pre-intervention activations for both methods. If there exists $\tau^*\in[0,1]$ such that $\Pr_F[0<\rho(X)\le\tau^*]=0$ and $\Pr_R[0<\rho(X)\le\tau^*]>0$, then dynamic clamping has equal forget-set coverage and a strictly lower retain-set intervention rate than static clamping.
\end{theorem}

\begin{proof}
We begin by formalizing the metrics used to evaluate both approaches and precisely defining their operation.

\paragraph{Preliminaries and Definitions.}
Let $\mathcal{X}$ be the space of possible input sequences. For a sequence $x = (x_1, \ldots, x_T)$ and its corresponding hidden states $\mathbf{h}_t$, we define:

\begin{itemize}
    \item A token $t$ is triggered by $S_{n_{\text{feats}}}$ if $\exists j \in S_{n_{\text{feats}}}$ such that $f_j(\mathbf{h}_t) > 0$
    \item The fraction of triggered tokens in a sequence: $\rho(x) = \frac{1}{T}\sum_{t=1}^T \mathbf{1}[\exists j \in S_{n_{\text{feats}}}: f_j(\mathbf{h}_t) > 0]$
\end{itemize}

We consider two distributions: $\mathcal{D}_{\text{forget}}$: The distribution of forget-relevant queries, and $\mathcal{D}_{\text{retain}}$: The distribution of retain-relevant queries.

\paragraph{The Two Approaches.}
For both approaches, we define a \emph{clamp set} $B_{\text{method}} \subseteq \mathcal{X}$ as the set of inputs where the method applies some clamping.

\quad 1. \textbf{Static Approach} ($a_{\text{static}}$): Clamps features in $S_{n_{\text{feats}}}$ whenever they activate on any token. here the Clamp set is $B_{\text{stat}} = \{x : \exists t, j \in S_{n_{\text{feats}}} \text{ such that } f_j(\mathbf{h}_t) > 0\}$.

\quad  2. \textbf{Dynamic Approach} ($a_{\text{dynamic}}$):  Computes $\rho(x)$ and applies a threshold test $\rho(x) > \tau$. Only clamps if the sequence passes this test. The Clamp set for threshold $\tau$: $B_{\text{dyn}}(\tau) = \{x : \rho(x) > \tau\}$.

\paragraph{Performance Metrics.}
We define:
\begin{itemize}
   \item \textbf{Coverage}: The probability that clamping occurs on forget-set queries
   $$\text{Coverage(method)} = \Pr_{x \sim \mathcal{D}_{\text{forget}}}[x \in B_{\text{method}}]$$
   \item \textbf{Side-effect}: The probability that clamping occurs on retain-set queries
   $$\text{SideEffect(method)} = \Pr_{x \sim \mathcal{D}_{\text{retain}}}[x \in B_{\text{method}}]$$
\end{itemize}

\paragraph{Step 1: Find the side-effect of the static approach.}
The static approach clamps whenever any token has an activating feature in $S_{n_{\text{feats}}}$. Therefore:

$$\text{SideEffect}(a_{\text{static}}) = \Pr_{x \sim \mathcal{D}_{\text{retain}}}[x \in B_{\text{stat}}] = \alpha$$

\paragraph{Step 2: Compare the intervention sets.}
Since $\tau^*\ge0$, $B_{\text{dyn}}(\tau^*)\subseteq B_{\text{stat}}$, and the excluded inputs are exactly $\{x:0<\rho(x)\le\tau^*\}$. Thus,
\begin{align*}
\text{SideEffect}(a_{\text{dynamic}}(\tau^*))
&=\alpha-\Pr_R[0<\rho(X)\le\tau^*]<\alpha.
\end{align*}

\paragraph{Step 3: Compare forget-set coverage.}
By the same set identity and the assumed zero forget probability of the excluded region,
\begin{align*}
\text{Coverage}(a_{\text{dynamic}}(\tau^*))
&=\text{Coverage}(a_{\text{static}})-\Pr_F[0<\rho(X)\le\tau^*]\\
&=\text{Coverage}(a_{\text{static}}).
\end{align*}
This proves dominance in the stated coverage--intervention-rate metrics under the separation condition. Without this condition, gating can reduce forget coverage as well as retain interventions.
\end{proof}

\section{Distribution of token activations on WMDP-Cyber}

\autoref{fig:cyber_token} plots the distribution of forget-set activated tokens on WMDP-Cyber. The threshold is chosen to control the retain set’s false
positive rate and we find that $p_{dyn}=95$ typically separates forget-set queries effectively achieving high recall on $\smash{\mathcal{D}_{forget}}$.
On WMDP-Cyber, DSG successfully transfers from the retain set (WikiText) and forget set to the test query set. 

\begin{figure}[htbp]
    \centering
    \includegraphics[width=0.55\textwidth]{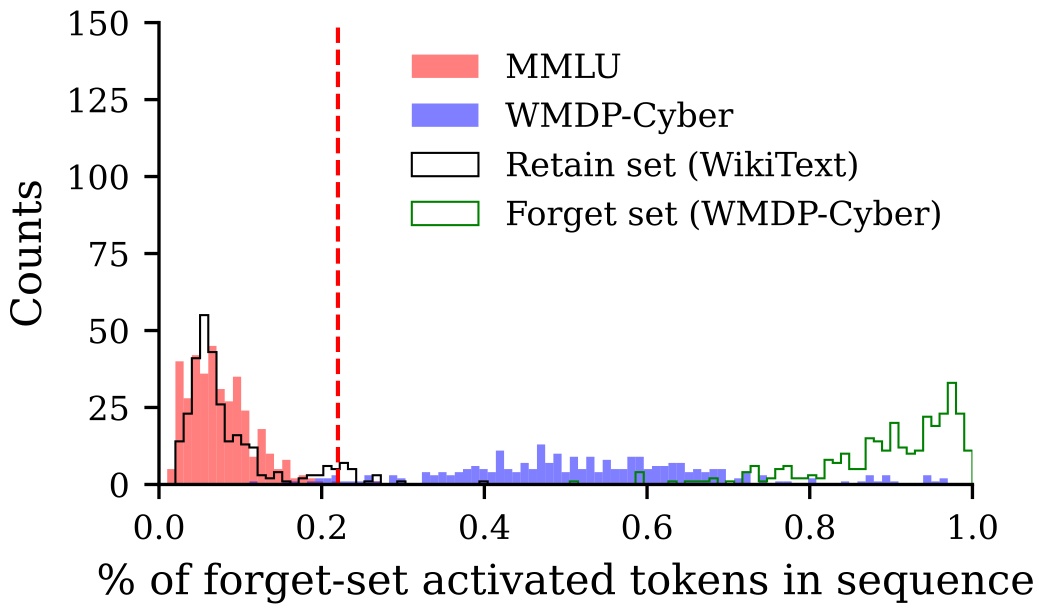} %
    \caption{\small Distribution of forget-set activated tokens for WMDP-Cyber. Threshold at the 95th percentile (dashed red line) effectively separates MMLU from WMDP. 
}
    \label{fig:cyber_token}
\end{figure}

\section{Unlearning on WMDP}
\label{app:unlearning-wmdp}

\subsection{Hyperparameter Details and Model Descriptions for Baselines}
\label{app:unlearning-wmdp-hyperparams}

To ensure a comprehensive and fair comparison of unlearning methods, we conducted extensive hyperparameter sweeps for each baseline, optimizing for both the effectiveness of knowledge removal and the preservation of model utility. For all gradient-based methods, we experimented with updating parameters in layers 3, 7, and 11 (as recommended in \citep{li2024wmdp}), as well as all layers. Unless otherwise specified, all experiments used the \texttt{google/gemma-2-2b-it} \citep{lieberum2024gemma} model.

\paragraph{Dynamic SAE Guardrails (DSG).}
Our proposed method, DSG, is a non-gradient-based intervention method that selectively removes hazardous knowledge by manipulating SAE feature activations. DSG first identifies a subset of SAE features strongly indicative of the knowledge to be forgotten, based on their differential activation patterns on forget and retain datasets. During inference, DSG employs a dynamic classifier to assess the relevance of input sequences. If a sequence is classified as forget-relevant based on the aggregate activation of selected features, DSG dynamically clamps these features to a negative value. This conditional, sequence-level clamping ensures that intervention is applied only when necessary, minimizing side effects on benign inputs and preserving model utility.

We employed the \texttt{gemma-scope-2b-pt-res} SAE (width 16k) applied to layer 3 ($\ell_0$ 142) \citep{lieberum2024gemma}. The dynamic threshold percentile ($p_{\text{dyn}}$) was fixed at 95. We swept the importance ratio percentile ($p_{\text{ratio}}$), number of selected features, and clamp strength ($c$):

\begin{table}[h]
\centering
\small
\begin{tabular}{ll}
    \toprule
    \textbf{Hyperparameter} & \textbf{Values Tested} \\
    \midrule
    Importance Ratio Percentile ($p_{\text{ratio}}$) & 90, 95 \\
    Number of Features & 10, 20, 30 \\
    Clamp Strength ($c$) & 10, 25, 50, 100, 200, 300, 400, 500 \\
    \bottomrule
\end{tabular}
\caption{\small Hyperparameter sweep for Dynamic SAE Guardrails (DSG). Fixed values: $p_{\text{dyn}}=95$.}
\label{tab:dsg_hyperparams}
\end{table}

The best configurations were: WMDP-Bio ($p_{\text{ratio}}=95$, features=20, $c=500$) and WMDP-Cyber ($p_{\text{ratio}}=90$, features=30, $c=500$).

\paragraph{Representation Misdirection for Unlearning (RMU).}

RMU \citep{li2024wmdp} is a gradient-based finetuning method that minimizes a composite loss function to achieve targeted forgetting while preserving model utility. This loss combines a forget loss and a retain loss. The forget loss acts on the model's activations on the forget dataset, increasing their norm in specific directions and making it difficult for later layers to process this information effectively. Simultaneously, the retain loss regularizes the updated model's activations on the retain dataset, encouraging activations to stay close to the original model's activations on benign data.

Key hyperparameters include the steering coefficient, which controls how much the activations are amplified on hazardous data, and the alpha parameter ($\alpha$), which balances utility preservation against knowledge removal. We focus
unlearning only on the MLPs, as recommended in \citet{li2024wmdp}.

\begin{table}[h]
\centering
\small
\begin{tabular}{ll}
    \toprule
    \textbf{Hyperparameter} & \textbf{Values Tested} \\
    \midrule
    Steering Coefficient & 1, 5, 10, 20, 100, 200, 400, 500, 800, 1000 \\
    Alpha ($\alpha$) & 0.01, 0.1, 1, 10, 100, 300, 500 \\
    Batch Size & 4, 8 \\
    Steps & 400, 800 \\
    \bottomrule
\end{tabular}
\caption{\small Hyperparameter sweep for RMU. Fixed values: Monitoring Layer ID=3, Learning Rate=5e-6.}
\label{tab:rmu_hyperparams}
\end{table}

The best configuration for WMDP-Bio used steering coefficient 400, alpha 100, monitoring layer 3, learning rate 5e-6, batch size 8, and 400 steps. For WMDP-Cyber, we used steering coefficient 500, alpha 10, monitoring layer 3, and batch size 8 with 400 steps.

\paragraph{Scalable Remembering and Unlearning unBound (SCRUB).}
SCRUB \citep{kurmanji2023towards} employs a student-teacher framework for knowledge distillation-based unlearning. It trains a student model, a clone of the original model, to forget hazardous knowledge under the guidance of the original, frozen teacher model. During forget epochs, SCRUB maximizes the KL divergence between student and teacher logits on the forget dataset. In retain epochs, it minimizes this divergence on the retain dataset, guiding the student to mimic the teacher on benign data. 

We swept across values of beta ($\beta$), a weighting factor balancing knowledge distillation and task-specific loss, while fixing alpha ($\alpha$) and gamma ($\gamma$) at 1.0:

\begin{table}[h]
\centering
\small
\begin{tabular}{ll}
    \toprule
    \textbf{Hyperparameter} & \textbf{Values Tested} \\
    \midrule
    Beta ($\beta$) & 0.0001, 0.001, 0.01, 0.1, 1, 10 \\
    Learning Rate (lr) & 1e-4, 1e-5, 5e-6 \\
    Batch Size & 4, 8 \\
    Steps & 400, 800 \\
    \bottomrule
\end{tabular}
\caption{\small Hyperparameter sweep for SCRUB. Fixed values: $\alpha=1.0$, $\gamma=1.0$, KL Temperature=2.0.}
\label{tab:scrubs_hyperparams}
\end{table}

The best configuration for WMDP-Bio used beta 0.01, learning rate 5e-6, batch size 8, and 400 maximum batches. For WMDP-Cyber, we used beta 0.1, learning rate 1e-5, batch size 8, and 400 maximum batches.

\paragraph{Selective Synaptic Dampening (SSD).}
SSD \citep{foster2024fast} identifies and dampens parameters more important for the forget set than the retain set. It adapts a method originally developed for image classification to language modeling by modifying the loss function to use log-perplexity. SSD calculates parameter importance scores based on gradients observed for both forget and retain datasets, then applies a dampening factor to parameters with higher importance for the forget dataset.

We performed a grid search spanning dampening thresholds and constants:
\begin{table}[H]
\small
\centering
\begin{tabular}{ll}
    \toprule
    \textbf{Hyperparameter} & \textbf{Values Tested} \\
    \midrule
    Threshold & 0.1, 0.25, 0.5, 1, 2.5, 5 \\
    Dampening Constant & 1e-5, 1e-4, 1e-3, 1e-2, 1e-1, 1 \\
    \bottomrule
\end{tabular}
\caption{\small Hyperparameter sweep for Selective Synaptic Dampening (SSD).}
\label{tab:ssd_hyperparams}
\end{table}

The optimal configuration for WMDP-Bio used threshold 0.5 and dampening constant 1e-3. For WMDP-Cyber, we used threshold 1.0 and dampening constant 1e-2.

\paragraph{Static SAE Clamping (Farrell et al.).}
This non-gradient-based approach \citep{farrell2024applying} identifies salient SAE features and statically clamps their activations during inference to remove unwanted knowledge. Unlike our dynamic approach, this method applies feature clamping universally to all inputs whenever a selected feature activates, rather than conditionally based on sequence-level classification.

We varied the retain threshold, multiplier (clamp value), and number of features:

\begin{table}[h]
\centering
\small
\begin{tabular}{ll}
    \toprule
    \textbf{Hyperparameter} & \textbf{Values Tested} \\
    \midrule
    Retain Threshold & 0.01, 0.001, 0.005, 0.1, 1 \\
    Multiplier (Clamp Value) & 10, 25, 50, 100, 200, 500 \\
    Number of Features & 5, 10, 20, 30, 50 \\
    \bottomrule
\end{tabular}
\caption{\small Hyperparameter sweep for Static SAE Clamping. Fixed value: Sequence Length=1024.}
\label{tab:farrell_hyperparams}
\end{table}

The best configurations were: WMDP-Bio (retain threshold=0.01, multiplier=200, features=5) and WMDP-Cyber (retain threshold=0.005, multiplier=500, features=10).

\paragraph{Gradient Ascent (GA).}
GA \citep{jang2023knowledge} is a finetuning-based unlearning method that directly minimizes the likelihood of correct predictions on the forget dataset using gradient ascent. In contrast to standard finetuning which employs gradient descent, GA utilizes gradient ascent to maximize the cross-entropy loss on the forget dataset, pushing parameters in directions that increase prediction error on the targeted data.

We varied the learning rate and beta ($\beta$), the retain loss weight:

\begin{table}[h]
\centering
\small
\begin{tabular}{ll}
    \toprule
    \textbf{Hyperparameter} & \textbf{Values Tested} \\
    \midrule
    Learning Rate (lr) & 1e-5, 5e-5 \\
    Beta (Retain Loss Weight) & 0.01, 0.1, 1.0, 5.0, 10.0 \\
    \bottomrule
\end{tabular}
\caption{\small Hyperparameter sweep for Gradient Ascent. Fixed values: Gamma=1.0, Batch Size=8, Steps=400.}
\label{tab:ga_hyperparams}
\end{table}

We explored both with and without retain data configurations. The best setting for WMDP-Bio used learning rate 1e-5 with beta 1.0. For WMDP-Cyber, we used learning rate 1e-5 with beta 0.1.

\paragraph{Negative Preference Optimization (NPO).}
NPO \citep{zhang2024negative} adapts preference optimization techniques to treat the forget set as negative examples. It reframes unlearning as preference learning, optimizing the model to assign lower likelihood to the forget set. The beta parameter controls the extent to which the unlearned model's output distribution can diverge from the original model. To mitigate utility degradation and preserve performance on benign data, NPO can be regularized using two distinct retain loss types: Negative Log-Likelihood (NLL) and Kullback-Leibler (KL) divergence. NLL minimization directly encourages the model to maintain high probabilities for correct tokens in the retain set, calculated as the negative sum of log probabilities assigned to ground truth tokens. KL divergence minimization encourages the probability distribution of the unlearned model to remain close to that of the original model on retain set inputs, measured as the information lost when approximating the original model's distribution with the unlearned model's distribution. 

We tested NPO with various configurations:

\begin{table}[h]
\centering
\small
\begin{tabular}{ll}
    \toprule
    \textbf{Hyperparameter} & \textbf{Values Tested} \\
    \midrule
    Alpha (Retain Loss Weight) & 0.01, 0.1, 1.0 \\
    Beta (Temperature Parameter) & 0.1, 1.0 \\
    Retain Loss Type & NLL, KL \\
    \bottomrule
\end{tabular}
\caption{\small Hyperparameter sweep for NPO. Fixed values: Gamma=1.0, Learning Rate=1e-5, Batch Size=8, Steps=400.}
\label{tab:npo_hyperparams}
\end{table}

The optimal settings for WMDP-Bio used alpha 0.1, beta 0.1, and KL divergence as the retain loss type and WMDP-Cyber used alpha 1.0, beta 0.1, and KL divergence as the retain loss type.

For all methods, we selected configurations that minimized WMDP accuracy while maintaining at least 99\% of the original model's MMLU accuracy.

\paragraph{Compute.} All finetuning and inference was performed on 4 A6000 GPUs in under a day.

\subsection{Results on WMDP-Cyber}
\autoref{tab:cyber_performance} shows the performance of various unlearning baselines on WMDP-Cyber dataset. RMU is less effective on WMDP-Cyber (88.00\%), likely due to the data inefficiency of gradient-based methods on the smaller cyber forget set.

\label{app:wmdp-cyber}
\begin{table}[h!]
\centering
\small
\begin{tabular}{l cccccc c}
\toprule
\multirow{2}{*}{Method} & \multirow{2}{*}{WMDP Cyber ($\downarrow$)} & \multicolumn{5}{c}{MMLU ($\uparrow$)} & \multirow{2}{*}{MT ($\uparrow$)} \\
\cmidrule(lr){3-7} 
 & & HS Hist & C. Bio & HS Geo & H. Aging & All & \\
\midrule
Target $\mathcal{M}$  & 100.00 & 100.00 & 100.00 & 100.00 & 100.00 & 100.00 & 7.36 \\
\midrule
GA &  98.91 & 98.15 & 100.0 & 100.0 & 100.0 & 99.46 & 7.39 \\
NPO & 96.36 & 100.0  & 100.0  & 100.0  & 100.0  & 100.0 & 7.18 \\
SSD & 98.91 & 100.00 & 100.00 & 98.08 & 98.81 & 99.19 & 7.25 \\
SCRUB & 97.82 & 99.07 & 100.00 & 100.00 & 98.81 & 99.46 & 6.51 \\
Farrell et al. & 52.73 & 99.07 & 100.00 & 100.00 & 97.62 & 99.19 & 7.39 \\
RMU & 88.00 & 99.07 & 100.00 & 99.04 & 98.81 & 99.19 & 7.28 \\
\midrule
\textbf{DSG (Ours)} & 26.74 & 99.07 & 100.00 & 100.00 & 100.00 & 99.73 & 7.66 \\
\bottomrule
\end{tabular}

\caption{\small Unlearning performance on WMDP-Cyber. All represents the average MMLU score. MT-Bench scores show 0.13 variance across 5 runs. DSG shows superior unlearning effectiveness compared to other baselines while maintaining high MMLU performance. }
\label{tab:cyber_performance}
\end{table}

\subsection{MT-Bench Evaluation Details}
\label{app:mt_bench_details}

To measure the impact of unlearning on the model's general conversational abilities and fluency, we utilized the MT-Bench benchmark \citep{zheng2023judging}. Specifically, we report the average score across two conversational turns (the two-turn average score), which provides a measure of multi-turn conversational quality. Following standard MT-Bench protocol, evaluations were conducted using GPT-4~\citep{achiam2023gpt} as the judge to score the model's responses. To ensure the robustness of these fluency assessments, each model configuration reported in Section \ref{sec:unlearningWMDP} was evaluated 5 times using MT-Bench. The mean scores presented in Table \ref{tab:bio_performance} and Table \ref{tab:cyber_performance} reflect the average performance across these runs, and the standard deviation across the 5 runs is noted in the respective table captions (0.16 for WMDP-Bio results, 0.13 for WMDP-Cyber results). Higher MT-Bench scores indicate better preservation of general conversational capabilities after the unlearning procedure.

\section{Unlearning on MUSE}
\label{app:unlearning-muse}

\subsection{Hyperparameter Details and Model Descriptions for Baselines}
\label{app:unlearning-muse-hyperparams}

We provide implementation details for the baseline unlearning methods evaluated in our experiments.

\vspace{-0.1in}
\paragraph{Gradient Ascent (GA).} GA maximizes the loss on the forget set, directly opposing the standard training objective to push the model away from the forget data's distribution \citep{jang2023knowledge}. While straightforward, it often leads to catastrophic forgetting of general knowledge.

\vspace{-0.1in}
\paragraph{Gradient Difference (GradDiff).} GradDiff balances competing objectives by maximizing the loss on the forget set while minimizing the loss on the retain set \citep{liu2022backdoor}. Despite this approach, GradDiff struggles to find an optimal trade-off, resulting in either over- or under-unlearning.

\vspace{-0.1in}
\paragraph{Negative Preference Optimization (NPO).} NPO reframes unlearning within a preference learning framework, treating the forget set as negative preference data by adapting the Direct Preference Optimization objective \citep{zhang2024negative}. We use NPO with KL Divergence Minimization that augments NPO with a KL divergence term to preserve utility by minimizing distributional shift on benign data.

\vspace{-0.1in}
\paragraph{Simplified NPO (SimNPO).} A computationally efficient variant of NPO that simplifies the optimization process while retaining core principles of negative preference learning \citep{fan2024simplicity}. SimNPO trades some unlearning effectiveness for faster processing.

\vspace{-0.1in}

\paragraph{Representation Misdirection for Unlearning (RMU).} RMU injects targeted noise into specific layers to disrupt the model's ability to process information related to the forget set \citep{li2024wmdp}. Its effectiveness depends heavily on precise noise targeting and hyperparameter tuning. We injected noise in the 7th layer for both News and Books.

For all finetuning-based baselines (GA, GradDiff, NPO, SimNPO, RMU), we used AdamW optimizer with a learning rate of 1e-5 and batch size of 32. We finetuned all parameters in the model.
The optimal checkpoint for each method was determined by selecting the first epoch (within 10 epochs) where the unlearned model's utility on the retain set fell below $90\%$ that of the target model. Table \ref{tab:muse_hyperparams} summarizes the optimal epochs or $\alpha$ values for each method on both datasets.

\begin{table}[h]
\centering
\small
\begin{tabular}{lcc}
\toprule
\textbf{Unlearning Method} & \textbf{NEWS} & \textbf{BOOKS} \\
\midrule
GA & epoch 1 & epoch 1 \\
GradDiff & epoch 2 & epoch 3 \\
NPO & epoch 8 & epoch 10 \\
SimNPO & epoch 10 & epoch 10 \\
RMU & epoch 9 & epoch 10 \\
\bottomrule
\end{tabular}
\caption{\small Optimal epochs for baseline unlearning methods on MUSE benchmark, determined by utility-based stopping criteria.}
\label{tab:muse_hyperparams}
\end{table}

All finetuning and inference was performed on 4 A6000 GPUs in under a day.

\subsection{Sequential Unlearning Strategies for DSG}
\label{app:sequential_unlearning}

In real-world scenarios, unlearning requests often arrive sequentially over time. An effective unlearning method must be able to handle multiple, successive requests without significant degradation in performance or utility. In Section~\ref{sec:methodology}, we evaluated DSG's performance under sequential unlearning using the MUSE benchmark \citep{shi2024muse} with four disjoint folds of the NEWS corpus. We implemented and compared two strategies for adapting DSG to this sequential setting, referred to as DSG$_{\text{all}}$ and DSG$_{\text{union}}$. Both strategies leverage the core DSG mechanisms of feature selection and dynamic thresholding but differ in how they aggregate information across multiple unlearning requests.

\paragraph{Setup.}
Let $k = 1, 2, \dots, K$ index the sequential unlearning requests. Each request $k$ introduces a new forget dataset $D_{F,k}$. We assume the retain dataset $D_R$ remains constant throughout the process. The goal at step $k$ is to produce an unlearned model that effectively forgets the cumulative forget data $\mathcal{D}_{F,k}^{\text{cumul}} = \cup_{i=1}^k D_{F,i}$ while preserving utility evaluated on $D_R$. Let $n_{\text{feats}}$ be the desired number of features to select at each relevant stage.

\vspace{-0.05in}
\paragraph{Strategy 1: DSG$_{\text{all}}$ (Cumulative Score Update)}

This strategy treats the sequential unlearning problem as equivalent to unlearning a single, growing forget set $\mathcal{D}_{F,k}^{\text{cumul}}$ at each step $k$. It maintains cumulative statistics required for calculating the feature importance scores.

\vspace{-0.05in}
\begin{itemize}[itemsep=1pt, topsep=0pt, leftmargin=15pt, rightmargin=0pt, labelsep=2pt, parsep=0pt]
    \item \textbf{Cumulative Statistics:} At step $k$, we need the aggregate sum of squared activations and the total number of tokens for all forget data seen so far. Let $A^2_{F,i}(j) = \sum_{x \in D_{F,i}} \sum_{t=1}^{|x|} [f_j(\mathbf{h}_{x,t})]^2$ be the sum of squared activations for feature $j$ on dataset $D_{F,i}$, and $N_{F,i} = \sum_{x \in D_{F,i}} |x|$ be the total number of tokens in $D_{F,i}$. The cumulative sums at step $k$ are:
    \vspace{-0.2in}
        \begin{align*}
            \Sigma^2_{F,k}(j) &= \sum_{i=1}^k A^2_{F,i}(j) \\
            N_{F,k}^{\text{cumul}} &= \sum_{i=1}^k N_{F,i}
        \end{align*}
    These sums can be updated incrementally as each new $D_{F,k}$ arrives, without needing to store all previous datasets. The retain set statistics ($A^2_R(j)$ and $N_R$) are computed once from $D_R$.

    \item \textbf{Score Calculation:} The importance scores are calculated using the cumulative statistics:
        \begin{align*}
            \texttt{forget\_score}_{\text{all},k}(j) &= \frac{\Sigma^2_{F,k}(j)}{N_{F,k}^{\text{cumul}}} \\
            \texttt{retain\_score}(j) &= \frac{A^2_R(j)}{N_R} \quad (\text{constant across } k)\\
            \texttt{imp\_ratio}_{\text{all},k}(j) &= \frac{\texttt{forget\_score}_{\text{all},k}(j)}{\max\{\texttt{retain\_score}(j), \varepsilon\}}
        \end{align*}

    \item \textbf{Feature Selection:} Using $\texttt{imp\_ratio}_{\text{all},k}(j)$ and $\texttt{forget\_score}_{\text{all},k}(j)$, select the feature set $S_{\text{all},k}$ containing the top $n_{\text{feats}}$ features, following the procedure in Algorithm~\ref{alg:DSG} (filtering by percentile $p_{\text{ratio}}$ and ranking by forget score).

    \item \textbf{Dynamic Threshold and Intervention:} Calculate the activation statistic $\rho_{\text{all},k}(x)$ as 
    $\rho_{\text{all},k}(x) = ({1}/{|x|})\sum_{t} \mathbf{1}[ \exists j \in S_{\text{all},k}: f_j(\mathbf{h}_t)>0]$. Calibrate the dynamic threshold $\tau_{\text{all},k} = \text{Percentile}(\{\rho_{\text{all},k}(x)\}_{x\in D_R}, p_{\text{dyn}})$. Apply conditional clamping using $S_{\text{all},k}$ and $\tau_{\text{all},k}$ during inference.
\end{itemize}

DSG$_{\text{all}}$ aims for the most accurate representation of feature importance with respect to all forgotten data combined.

\paragraph{Strategy 2: DSG$_{\text{union}}$ (Union of Feature Sets)}

This strategy selects features based on each individual forget request $D_{F,k}$ and then uses the union of these feature sets for intervention.

\begin{itemize}[itemsep=1pt, topsep=0pt, leftmargin=15pt, rightmargin=0pt, labelsep=2pt, parsep=0pt]
    \item \textbf{Independent Score Calculation:} At step $k$, calculate importance scores using only the current forget set $D_{F,k}$ and the retain set $D_R$:
        \begin{align*}
            \texttt{forget\_score}_{\text{indep},k}(j) &= \frac{A^2_{F,k}(j)}{N_{F,k}} \\
            \texttt{retain\_score}(j) &= \frac{A^2_R(j)}{N_R} \\
            \texttt{imp\_ratio}_{\text{indep},k}(j) &= \frac{\texttt{forget\_score}_{\text{indep},k}(j)}{\max\{\texttt{retain\_score}(j), \varepsilon\}}
        \end{align*}

    \item \textbf{Independent Feature Selection:} Select the feature set $S_{\text{indep}, k}$ containing the top $n_{\text{feats}}$ features based on $\texttt{imp\_ratio}_{\text{indep},k}(j)$ and $\texttt{forget\_score}_{\text{indep},k}(j)$.

    \item \textbf{Union Set Formation:} Maintain the cumulative union of feature sets identified at each step:
        \begin{equation*}
            S_{\text{union}, k} = S_{\text{union}, k-1} \cup S_{\text{indep}, k} \quad (\text{with } S_{\text{union}, 0} = \emptyset)
        \end{equation*}
    The size of $S_{\text{union}, k}$ may grow beyond $n_{\text{feats}}$.

    \item \textbf{Dynamic Threshold and Intervention:} Calculate the activation statistic $\rho_{\text{union},k}(x) = \frac{1}{|x|}\sum_{t} \mathbf{1}[\exists j\in S_{\text{union},k}: f_j(\mathbf{h}_t)>0]$. Calibrate the dynamic threshold $\tau_{\text{union},k} = \text{Percentile}(\{\rho_{\text{union},k}(x)\}_{x\in D_R}, p_{\text{dyn}})$. Apply conditional clamping using $S_{\text{union}, k}$ and $\tau_{\text{union},k}$ during inference.
\end{itemize}

DSG$_{\text{union}}$ ensures that features deemed important for any past forget request are considered for intervention, potentially capturing a broader range of forget-related concepts but possibly leading to a larger intervention set over time.

\paragraph{Result.} As reported in the main text (\autoref{fig:knowmem_tradeoff}(b)), both DSG$_{\text{all}}$ and DSG$_{\text{union}}$ demonstrated strong and stable performance across the four sequential unlearning requests on the MUSE benchmark, significantly outperforming gradient-based methods which showed rapid degradation.

\section{Relearning attack}

\subsection{Superficial Alignment Hypothesis}
\label{app:alignment_hypothesis}

\begin{figure}[htbp]
    \centering
    \begin{subfigure}[b]{0.49\textwidth}
        \centering
        \includegraphics[width=\textwidth]{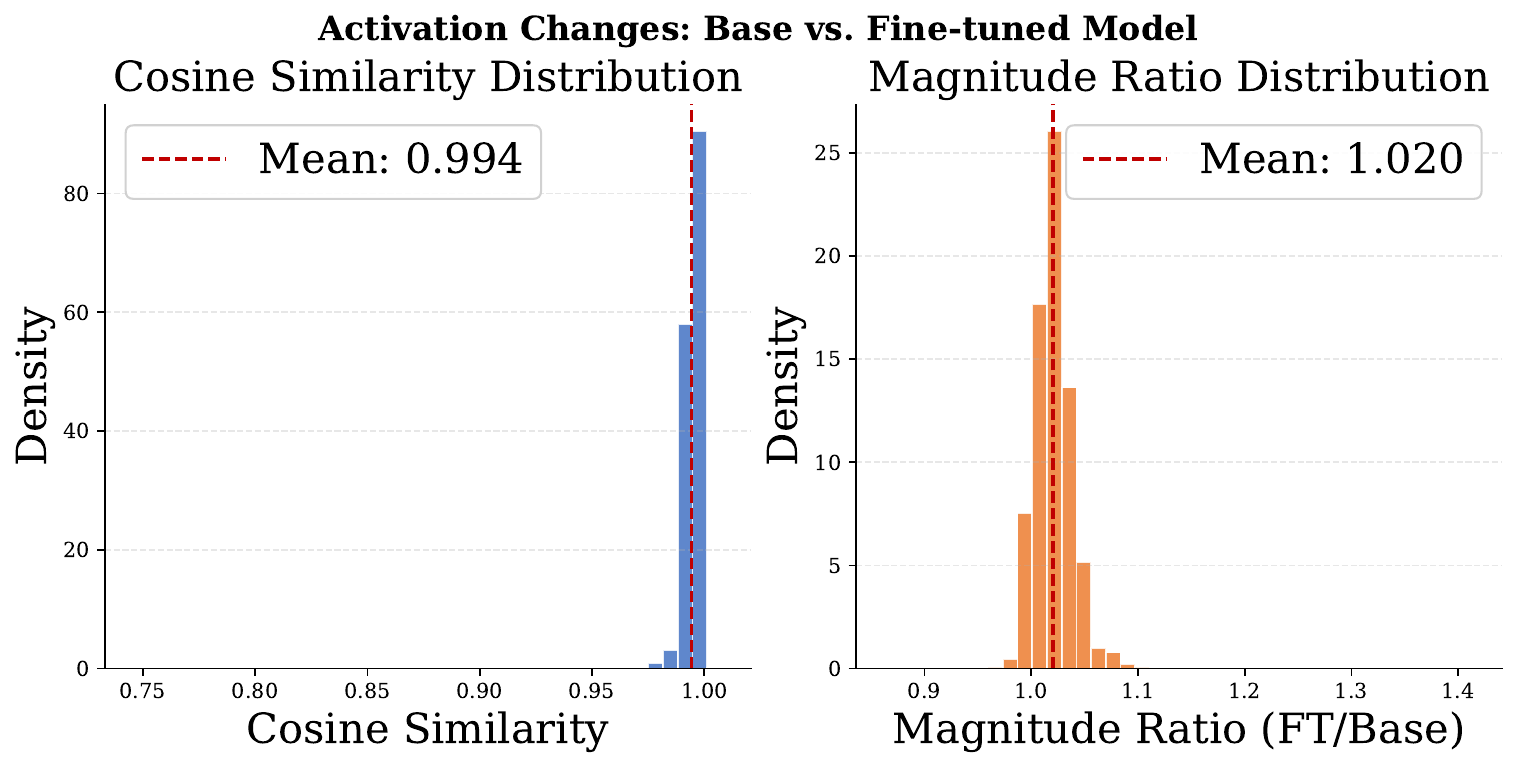}
        \caption{}
        \label{fig:relearn_cosine}
    \end{subfigure}
    \hfill
    \begin{subfigure}[b]{0.49\textwidth}
        \centering
        \includegraphics[width=\textwidth]{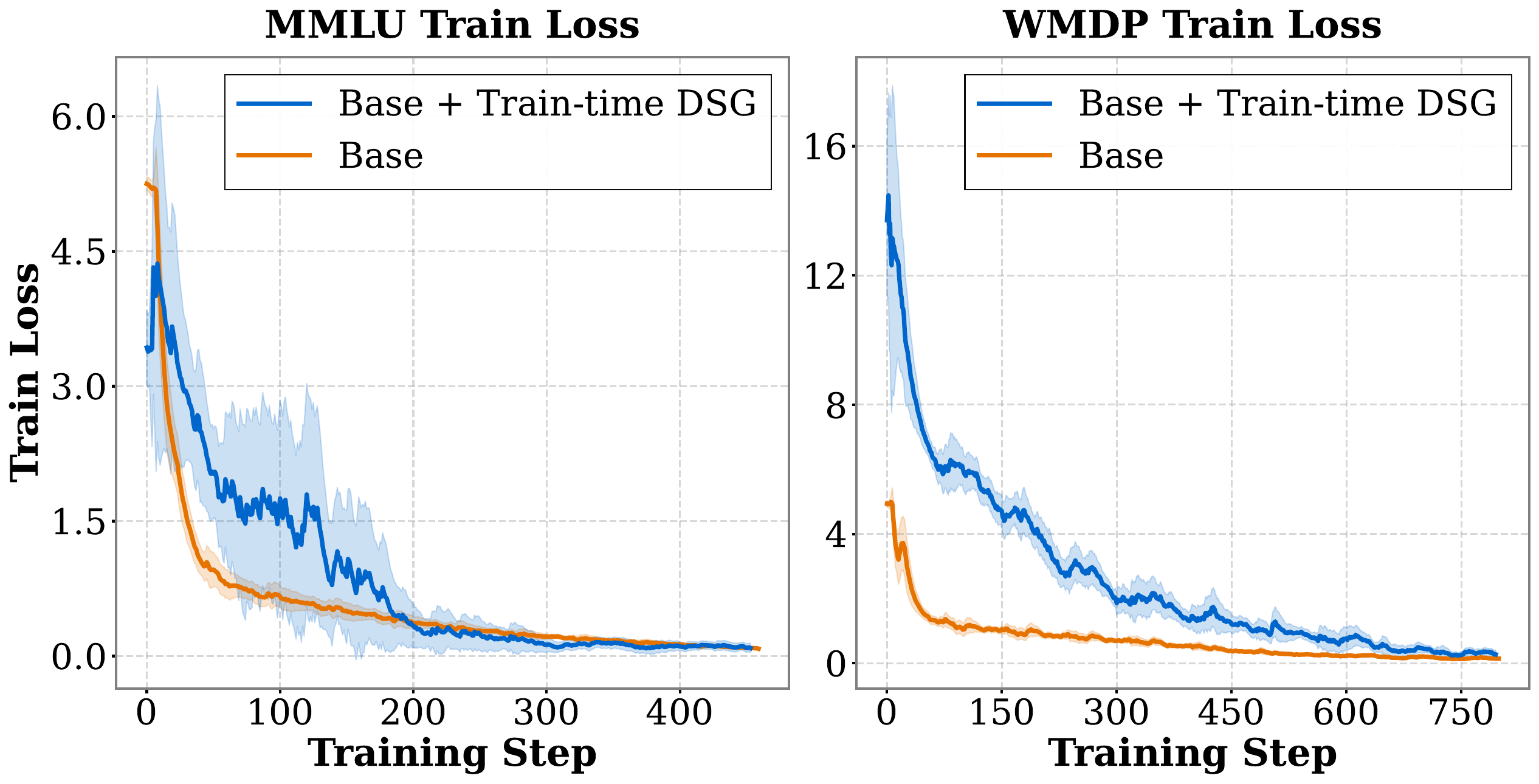}
        \caption{}
        \label{fig:train_loss_comparison}
    \end{subfigure}
    \caption{\small (a) Distribution of activation cosine similarity and activation magnitude ratio between Base and Finetuned models. Finetuning does not significantly change the underlying activation space. (b) Train loss when finetuning Base model and Base+SAE model on WMDP and MMLU. Loss on WMDP for the BASE+SAE model is significantly higher than on MMLU. }
    \label{fig:combined_analysis}
\end{figure}

The resistance of DSG to relearning attacks can be understood through the lens of the Superficial Alignment Hypothesis \citep{zhou2023lima}, which posits that a model's activation geometry is established during pretraining and remains relatively stable during subsequent finetuning. We provide empirical evidence supporting this hypothesis in Figure~\ref{fig:relearn_cosine}, which presents the distribution of activation cosine similarities and magnitude ratios between the base and finetuned models.

The concentration of cosine similarity values near 1.0 indicates that finetuning preserves the directional information in the activation space, with minimal rotational changes. Similarly, the activation magnitude ratios cluster tightly around 1.0, demonstrating that the scale of activations remains largely unchanged during finetuning. These findings align with previous research suggesting that while weights may change substantially during finetuning, the underlying activation patterns and geometry remain remarkably stable.

This stability of activation geometry is the basis for DSG's effectiveness against relearning attacks. By operating directly on these stable activation patterns rather than weights, DSG establishes a more durable defense mechanism that persists even when adversaries attempt to modify the model's weights through finetuning.

\subsection{Train-time DSG Details}
\label{app:train_time_dsg} %

Beyond applying DSG only at inference (Test-time DSG), we explore integrating it directly into the finetuning process itself to further enhance resistance against relearning attacks. This approach, termed \textbf{Train-time DSG}, applies the standard DSG logic during each forward pass of the finetuning/relearning phase.

Specifically, during finetuning on a potentially adversarial dataset (like the forget set itself in a relearning attack scenario), Train-time DSG operates as follows:
\begin{enumerate}[itemsep=1pt, topsep=0pt, leftmargin=*, rightmargin=0pt, labelsep=2pt, parsep=0pt]
    \item For each input sequence $x$ in a training batch, compute the hidden states $h_t$ and corresponding SAE feature activations $f_j(h_t)$.
    \item Calculate the statistic $\rho(x)$ based on the pre-selected forget feature set $S_{\text{nfeats}}$.
    \item Classify the sequence using the dynamic threshold $\tau$: $C(x) = \mathbf{1}[\rho(x) > \tau]$.
    \item \textbf{Conditional Clamping:} If $C(x) = 1$ (forget-relevant), modify the activations for features $j \in S_{\text{nfeats}}$ by setting $f'_j(h_t) = -c$ for all tokens $t$. Otherwise, $f'_j(h_t) = f_j(h_t)$. These potentially modified activations $f'$ are then used for the reconstruction $\hat{h}'$ and subsequent layers of the LLM.
    \item The final loss (e.g., cross-entropy on the relearning task) is computed based on the LLM's output derived from these potentially clamped activations.
    \item \textbf{Gradient Blocking:} During the backward pass, gradients flow back through the model as usual. However, for any feature activation $f_j(h_t)$ that was clamped to $-c$, the gradient of the loss with respect to the upstream components (that produced $h_t$) through that specific feature pathway is effectively blocked. Setting the activation to a constant $-c$ detaches it from the upstream computations for the purpose of gradient calculation via that feature's contribution path. This is conceptually akin to applying a \texttt{stop\_gradient} operation specifically on the clamped feature activations.
\end{enumerate}

During this finetuning process, the parameters of the SAE itself (encoder $W_{\text{enc}}, b_{\text{enc}}$ and decoder $W_{\text{dec}}, b_{\text{dec}}$) are kept \textbf{frozen}. This prevents the SAE from adapting to circumvent the clamping intervention.

This dynamic gradient blocking prevents the finetuning process from easily undoing the unlearning effect by simply adjusting weights to reactivate the specific features in $S_{\text{nfeats}}$ that carry the forget-set information. When the model attempts to minimize loss on forget-set examples by utilizing these features, Train-time DSG clamps them and blocks the relevant gradient signal. This forces the model, if it attempts to relearn, to find potentially much less direct or alternative pathways through other features or model components. This difficulty in relearning via the original pathways contributes to the significantly higher training loss observed on WMDP-Bio when finetuning with Train-time DSG active, as seen in Figure~\ref{fig:train_loss_comparison}.

\subsection{Tamper-Resistant Safeguards}
\label{app:tamper_resistance}

DSG functions as a tamper-resistant safeguard during finetuning by effectively filtering gradients that would otherwise enable the model to relearn forgotten knowledge. Figure~\ref{fig:train_loss_comparison} demonstrates this mechanism quantitatively, showing the training loss profiles when finetuning the base model and the base model with DSG active on both WMDP-Bio (forget set) and MMLU (retain set) datasets.

When DSG is active during finetuning, we observe significantly elevated training loss values on WMDP-Bio compared to MMLU. This marked difference in loss profiles indicates that DSG selectively impedes the model from reducing loss on forget set content while allowing normal optimization on retain set content. This selective gradient filtering creates an effective barrier against relearning targeted information.

The mechanism works because during finetuning, DSG constantly monitors activations and applies clamping whenever forget-relevant features are activated above the dynamic threshold. This intervention disrupts the gradient flow for targeted concepts, requiring the model to develop entirely new processing pathways rather than simply recovering previously established connections. This \emph{rewiring} requirement explains the delayed recovery pattern observed in the main relearning experiments, where performance remains near random for approximately six epochs before beginning to increase.

\subsection{Relearning Attack at Learning Rate 1e-6}

\label{app:relearning_attack_lr}
\begin{figure}[h!]
    \centering
    \includegraphics[width=0.8\textwidth]{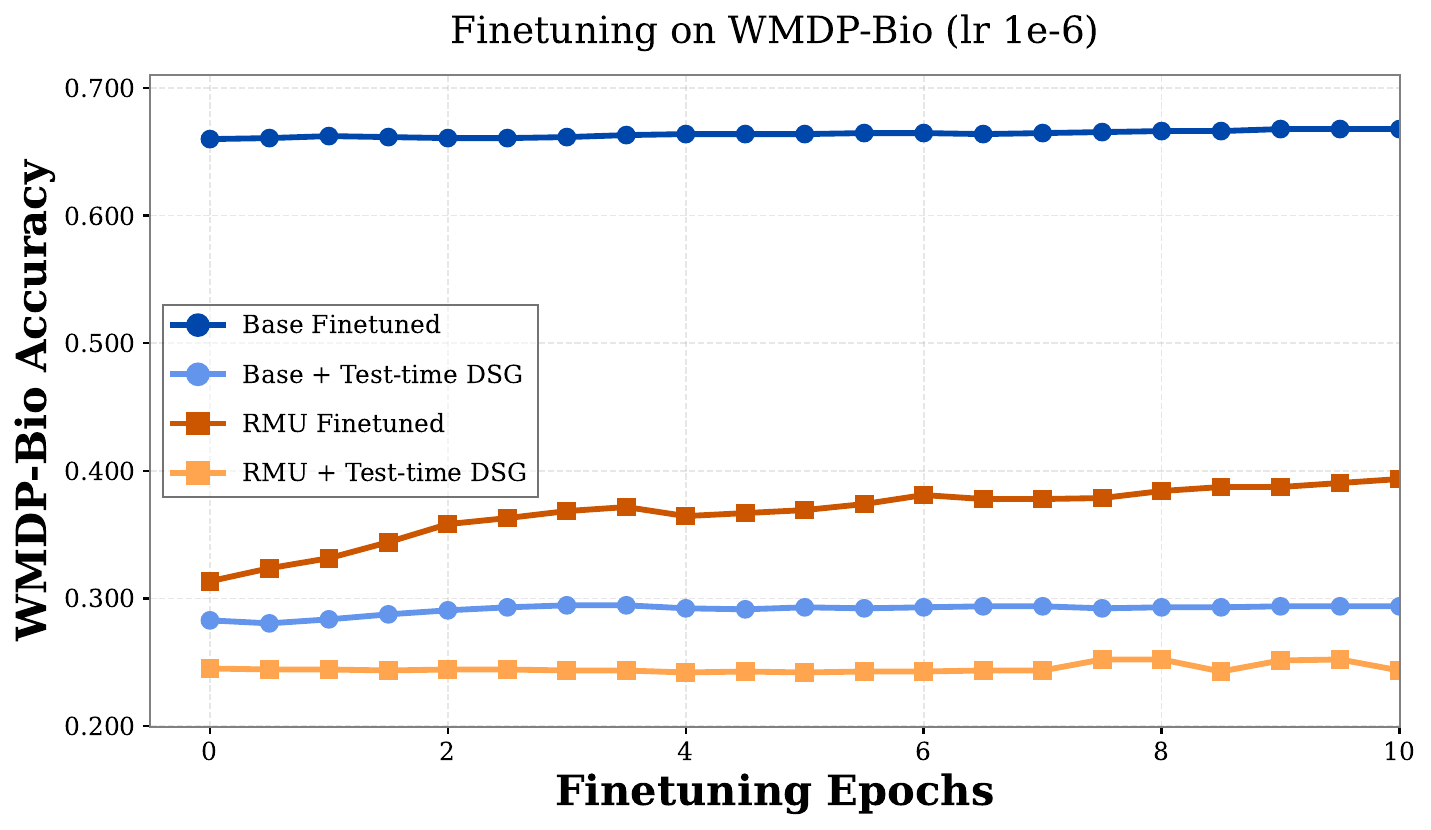} %
    \caption{\small Relearning attack performance with reduced learning rate (1e-6). All configurations show minimal performance changes across finetuning epochs, demonstrating that relearning attack efficacy is strongly dependent on learning rate.}
    \label{fig:relearning_1e-6}
\end{figure}

To investigate the impact of learning rate on relearning attack efficacy, we conducted a supplementary analysis using a reduced learning rate of 1e-6 (compared to 1e-5 in the main experiments). Figure~\ref{fig:relearning_1e-6} presents WMDP-Bio accuracy across finetuning epochs for all configurations under this reduced learning rate condition.

The results demonstrate minimal performance changes across all configurations throughout the finetuning process. This stability indicates that relearning attack efficacy is strongly dependent on learning rate, with lower rates substantially limiting the model's ability to recover forgotten knowledge. This finding has important implications for practical deployment scenarios, suggesting that implementing learning rate constraints on model access APIs could serve as an additional defense layer against relearning attacks.

\subsection{Relearning Hyperparameters}
\label{app:relearning_hyperparams}

For the relearning experiments, we used the RMU unlearned model as described in Section~\ref{sec:unlearningWMDP}, with RMU hyperparameters set to steering coefficient 400, alpha 100, monitoring layer 3, AdamW optimizer, learning rate 5e-6, batch size 8, and 400 steps. For DSG configurations, we employed the optimal parameters identified in our WMDP-Bio experiments: importance ratio percentile ($p_{\text{ratio}}$) of 95, feature count of 20, and clamp strength ($c$) of 500 for both test-time and train-time DSG interventions. The dynamic threshold percentile ($p_{\text{dyn}}$) was maintained at 95, consistent with our main experiments.

All finetuning and inference was performed on 2 A100 GPUs in under a day.

\section{Data Efficiency and Zero-shot Capabilities}
\label{app:data-efficiency}

\subsection{Hyperparameters}

\paragraph{Data Efficiency} 
For the data efficiency experiments, we maintain consistent hyperparameter settings across all data subsets to isolate the impact of dataset size. We use the optimal DSG configuration identified for WMDP-Bio with 100\% data, as shown in Table~\ref{tab:dsg-efficiency-hyperparams}.

\begin{table}[h]
    \centering
    \small
    \begin{tabular}{ll}
        \toprule
        \textbf{Parameter} & \textbf{Value} \\
        \midrule
        SAE & \texttt{gemma-scope-2b-pt-res} SAE (width 16k) \\
        SAE layer & Layer 3, $\ell_0$ 142 \\
        Importance ratio percentile ($p_{\text{ratio}}$) & 95 \\
        Dynamic threshold percentile ($p_{\text{dyn}}$) & 95 \\
        Number of selected features & 20 \\
        Clamp strength ($c$) & 500 \\
        \bottomrule
    \end{tabular}
        \caption{\small DSG hyperparameters for data efficiency experiments}
    \label{tab:dsg-efficiency-hyperparams}
\end{table}

For RMU comparisons, we evaluate two approaches: (1) maintaining the same number of training steps (400) across all data subsets, and (2) completing one full epoch over each dataset subset. Maintaining the same number of training steps produced a better Pareto front. We select the model with lower WMDP accuracy for each subset. The base RMU configuration for WMDP-Bio is presented in Table~\ref{tab:rmu-hyperparams}.

\begin{table}[h]
    \centering
    \small
    \begin{tabular}{ll}
        \toprule
        \textbf{Parameter} & \textbf{Value} \\
        \midrule
        Steering coefficient & 400 \\
        Alpha ($\alpha$) & 100 \\
        Monitoring layer & 3 \\
        Learning rate & 5e-6 \\
        Parameter subset & MLP layers only \\
        \bottomrule
    \end{tabular}
    \caption{\small RMU hyperparameters for data efficiency experiments}
    \label{tab:rmu-hyperparams}
\end{table}

\paragraph{Zero-shot} 
For zero-shot experiments, we vary only the dynamic threshold $\tau$ (as no retain set is available for calibration) while keeping all other hyperparameters fixed at their optimal values for each task, as shown in Table~\ref{tab:zero-shot-hyperparams}.

\begin{table}[h]
    \centering
    \small
    \begin{tabular}{lll}
        \toprule
        \textbf{Parameter} & \textbf{WMDP-Bio} & \textbf{WMDP-Cyber} \\
        \midrule
        SAE layer & Layer 3, $\ell_0$ 59 & Layer 3, $\ell_0$ 59 \\
        Importance ratio percentile ($p_{\text{ratio}}$) & 95 & 90 \\
        Feature selection & 20 features  & 20 features via  \\
        Clamp strength ($c$) & 500 & 500 \\
        $\tau$ range tested & 0.1 to 0.9 (increments of 0.1) & 0.1 to 0.9 (increments of 0.1) \\
        Optimal $\tau$ & 0.6 & 0.2 \\
        \bottomrule
    \end{tabular}
        \caption{\small Hyperparameters for zero-shot experiments}
    \label{tab:zero-shot-hyperparams}
\end{table}

The optimal thresholds were determined to be $\tau = 0.6$ for WMDP-Bio and $\tau = 0.2$ for WMDP-Cyber, as shown in Figure~\ref{fig:zero_shot}B.

\section{Ablations}
\label{app:ablations}
This appendix provides comprehensive details on our ablation studies for DSG. We analyze each component's contribution to overall performance and explore sensitivity to various hyperparameters.

\subsection{Additional Ablations}

\paragraph{DSG Clamp Strength \( c \).}  
The clamping parameter \( c \) determines the magnitude of intervention applied to selected SAE features. As shown in Figure~\ref{fig:multiplier}, WMDP-Bio accuracy drops significantly at modest clamp values (\( c=25 \)), reaching near-optimal unlearning performance, while MMLU accuracy remains above 99\% for configurations with 10-20 features. For these optimal feature counts, performance remains remarkably stable across a wide range of clamp strengths (\( 100 \leq c \leq 500 \)), demonstrating DSG's robustness to this parameter.
By contrast \citet{farrell2024applying} exhibit greater sensitivity to clamp values, as seen in Figure~\ref{fig:ablations}A.

\begin{figure}[h!]
    \centering
    \includegraphics[width=0.5\textwidth]{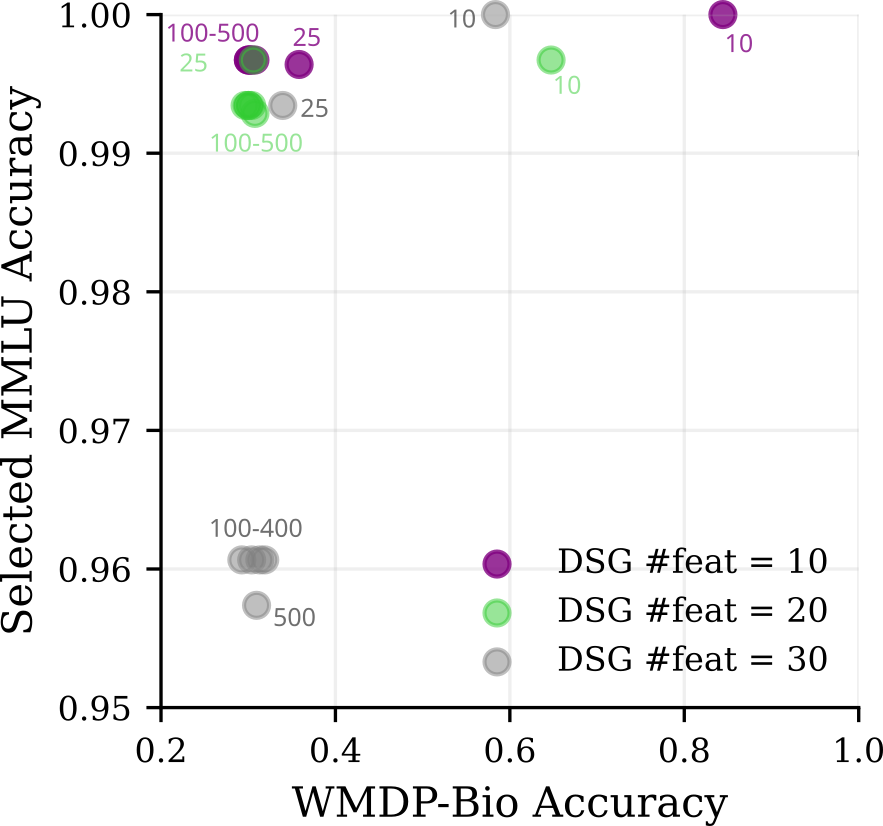}
\caption{\small Effect of clamp strength $c$ on DSG performance across different feature counts. MMLU accuracy (solid lines) remains consistently high ($>99\%$) for 10-20 features across all clamp values, while WMDP-Bio accuracy (dashed lines) drops sharply even at modest clamp strengths ($c=25$). This demonstrates DSG's ability to effectively remove targeted knowledge while preserving general model capabilities with minimal parameter sensitivity.}
    \label{fig:multiplier}
\end{figure}

\paragraph{DSG Number of Features.}  

Across experiments, the number of features selected during percentile-based feature selection represents a critical balance between coverage and precision. Selecting too few features may result in insufficient removal of forget-set information, as some forget-set inputs might not activate the limited feature set strongly enough to trigger intervention. Conversely, selecting too many features increases the risk of including noisy features selected using importance scoring or less discriminative features that activate on retain-set samples, potentially causing false positive detections and reducing model utility.

Our experiments consistently show that 20 features provides an optimal balance for both WMDP-Bio and WMDP-Cyber domains. Configurations with 10 features occasionally show reduced unlearning effectiveness despite good utility preservation, while 30-feature configurations begin to impact retain-set performance at higher clamp strengths. The precise optimal feature count may vary by domain and dataset characteristics as well as SAE width, but the overall pattern of diminishing returns with increased feature counts remains consistent.

\paragraph{Choice of Activation Statistic: Percentage vs. Raw Count.}
DSG's dynamic classification uses a sequence-level statistic derived from forget-feature ($S_{n_{\text{feats}}}$) activations. We compared two statistics:
(1) Percentage-based ($\rho$), the fraction of tokens where any $j \in S_{n_{\text{feats}}}$ activates ($f_j(\mathbf{h}_t) > 0$):
\begin{equation*}
    \rho(x) = \frac{1}{|x|} \sum_{t=1}^{|x|} \mathbf{1}[\exists j \in S_{n_{\text{feats}}}: f_j(\mathbf{h}_t) > 0]
\end{equation*}
and (2) Raw count-based ($\rho_{\text{raw}}$), the absolute number of such tokens:
\begin{equation*}
    \rho_{\text{raw}}(x) = \sum_{t=1}^{|x|} \mathbf{1}[\exists j \in S_{n_{\text{feats}}}: f_j(\mathbf{h}_t) > 0]
\end{equation*}
We compare score distributions using Total Variation Distance (TVD) to assess retain-set alignment (WikiText vs. MMLU) and retain/forget separation (WikiText vs. WMDP; Figure~\ref{fig:distances}).

Empirically, $\rho$ shows better alignment and separation by these TVD measures. For WMDP-Bio:
(1) Retain alignment (WikiText vs. MMLU): $\text{TVD}(\rho)=0.38 \pm 0.03$ vs. $\text{TVD}(\rho_{\text{raw}})=0.88 \pm 0.01$, indicating closer retain score distributions.
(2) Retain/Forget separation (WikiText vs. WMDP-Bio): $\text{TVD}(\rho)=0.90 \pm 0.02$ vs. $\text{TVD}(\rho_{\text{raw}})=0.41 \pm 0.03$, showing greater distributional separation for $\rho$. Similar results hold for WMDP-Cyber (Figure~\ref{fig:distances}).

The percentage-based statistic $\rho$ normalizes raw counts by sequence length, providing a length-normalized measure of activation density. Its distribution can still depend on length: for $T$ independent token indicators with common activation probability $p$, $\mathbb{E}[\rho]=p$ and $\operatorname{Var}(\rho)=p(1-p)/T$.

For score distributions $P,Q$, $|P(\rho>\tau)-Q(\rho>\tau)|\le\operatorname{TVD}(P,Q)$. Thus, small retain-set TVD bounds the change in intervention probability at a fixed threshold. Large retain/forget TVD describes overall separation; performance at the chosen threshold requires the corresponding forget and retain tail probabilities, or an ROC comparison across thresholds.

\begin{figure}[h!]
    \centering
    \includegraphics[width=0.8\textwidth]{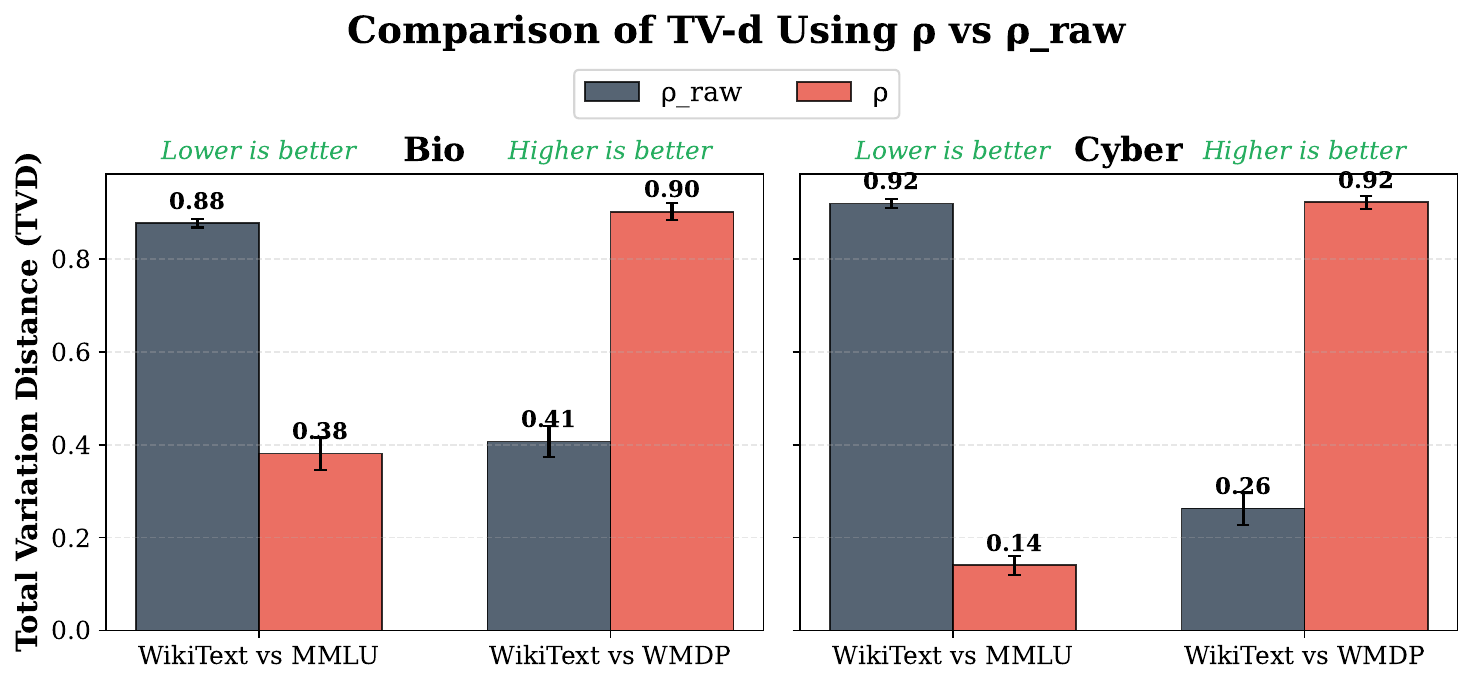}
    \caption{\small Total Variation Distance (TVD) between WikiText and benchmark datasets using percentage-based (\( \rho \)) vs. raw count-based (\( \rho_{\text{raw}} \)) metrics. Lower TVD between WikiText and MMLU indicates better alignment of retain sets, while higher TVD between WikiText and WMDP indicates better separation between retain and forget distributions. Percentage-based metrics consistently outperform raw counts on both measures across all benchmarks.}
    \label{fig:distances} %
\end{figure}

\subsection{Ablations Hyperparameter Details}

Our ablation studies used the following hyperparameter configurations:

\paragraph{Clamp Strength and Feature Count.} We evaluated DSG performance with feature counts of 10, 20, and 30, across clamp strengths \( c \in \{10, 25, 50, 100, 200, 300, 400, 500\} \) and \( p_{\text{ratio}}=95 \).

\paragraph{Feature Selection Comparison.} To compare our percentile-based approach with \citet{farrell2024applying}, we tested both methods using 20 and 30 features, with clamp values in the range [10-500]. We set \( p_{\text{ratio}}=95 \) for DSG and used the recommended threshold of 0.01 for \citet{farrell2024applying}.

\paragraph{Dynamic Threshold.} We varied \( p_{\text{dyn}} \) from 60 to 97 using 20 and 30 features with \( c = 500 \) to examine the impact of threshold selection on the forget-retain trade-off.

\paragraph{Importance Ratio Threshold.} We tested \( p_{\text{ratio}} \) values from 75 to 95 using 20 and 30 features with \( c=500 \) to assess feature selection stringency effects.

\paragraph{Activation Metrics.} For comparing percentage vs. raw count metrics, we applied bootstrap resampling with 1000 iterations, using Kernel Density Estimation to compute robust TVD estimates between WikiText and test set distributions.

\section{Computational Cost (Inference Latency)}
\label{app:latency}

A practical consideration for deploying unlearning methods is their impact on inference speed. We evaluated the latency introduced by DSG compared to the original model and a static clamping baseline \citep{farrell2024applying}.

Interventions using SAEs inherently introduce some latency compared to the original LLM without the SAE. This overhead stems from two main sources: \textbf{(1) The baseline cost of the SAE's forward pass}, which involves matrix multiplications for both encoding ($\mathbf{z} = \sigma(\mathbf{W}_{\text{enc}} \mathbf{h} + \mathbf{b}_{\text{enc}})$) and decoding ($\hat{\mathbf{h}} = \mathbf{W}_{\text{dec}} \mathbf{z} + \mathbf{b}_{\text{dec}}$), scaling with the SAE's width ($d_{\text{sae}}$); and \textbf{(2) The cost of the specific intervention logic} applied to the SAE features.

For DSG, this intervention logic involves two main steps beyond the standard SAE pass: (a) calculating the $\rho(x)$ statistic (fraction of forget-activated tokens) across the sequence's activations, and (b) conditionally applying the clamping intervention based on the $\rho(x) > \tau$ comparison.

Table \ref{tab:latency_comparison} presents the mean inference times (in seconds) and standard deviations over 100 samples (batch size 1) for processing sequences of varying lengths (256, 512, 1024 tokens). These measurements were performed using the \texttt{google/gemma-2-2b-it} model \citep{lieberum2024gemma} and with the \texttt{gemma-scope-2b-pt-res} SAE (width 16k, applied at layer 3, $\ell_0$ 142) \citep{lieberum2024gemma} on a single A6000 GPU.

As shown in Table \ref{tab:latency_comparison}, the \textbf{total combined overhead} (SAE matrix multiplications + intervention logic) introduced by both static clamping and DSG is \textbf{minimal}. Specifically, DSG increases latency by about 5\% (ranging from approximately 3.6\% to 7.3\%) over the original model across the tested sequence lengths. Importantly, the additional overhead incurred by DSG's dynamic classification logic (calculating $\rho(x)$ and thresholding) compared to simple static clamping is negligible indicating that the primary source of the observed latency increase relative to the base LLM is the SAE's own forward pass.

\begin{table}[h]
    \centering
    \small
    \begin{tabular}{lccc}
        \toprule
        \textbf{Seq Length} & \textbf{Original Model (s)} & \textbf{Static Clamping (s)} & \textbf{Dynamic Clamping (DSG) (s)} \\
        \midrule
        256 tokens  & $0.0872 \pm 0.0098$ & $0.0933 \pm 0.0091$ (+7.0\%) & $0.0936 \pm 0.0090$ (+7.3\%) \\
        512 tokens  & $0.1618 \pm 0.0061$ & $0.1659 \pm 0.0029$ (+2.5\%) & $0.1676 \pm 0.0047$ (+3.6\%) \\
        1024 tokens & $0.3300 \pm 0.0081$ & $0.3403 \pm 0.0083$ (+3.1\%) & $0.3420 \pm 0.0081$ (+3.6\%) \\
        \bottomrule
    \end{tabular}%
    \caption{\small Comparison of Inference Latency Across Sequence Lengths for \texttt{gemma-2-2b-it} with \texttt{gemma-scope-2b-pt-res} SAE. Data reported as mean $\pm$ std over 100 samples on a single A6000 GPU. Percentage increase relative to the Original Model shown in parentheses.}
    \label{tab:latency_comparison}
\end{table}

While DSG introduces this slight inference overhead, it is important to consider the broader computational context. Gradient-based unlearning methods require computationally intensive finetuning processes involving backward passes through the model for each unlearning request. In contrast, DSG's unlearning cost primarily involves a one-time computation of activation statistics (which can be amortized across many uses) and the \textbf{minimal}, constant inference-time overhead detailed above.

Therefore, DSG offers a highly efficient alternative for unlearning, achieving state-of-the-art forgetting effectiveness and utility preservation with only a marginal increase in inference latency. This makes it particularly attractive for scenarios requiring frequent or sequential unlearning operations where the cost of repeated gradient-based finetuning would be prohibitive.

\clearpage

\section{Feature Interpretability}
\label{app:intepretability}

A key strength of Dynamic SAE Guardrails (DSG) is \textbf{interpretable unlearning}, especially in zero-shot scenarios where domain-specific data is absent.  To demonstrate this, we used Neuronpedia API's \emph{search by SAE} \citep{neuronpedia} to directly identify Sparse Autoencoder (SAE) features relevant to biosecurity and cybersecurity hazards.  For WMDP-Bio and WMDP-Cyber, ``Biology" and ``Cybersecurity" queries retrieved the top 20 feature IDs from the \texttt{gemma-scope-2b-pt-res} SAE (width 16k) \citep{lieberum2024gemma} applied to \texttt{gemma-2-2b-it} layer 3 ($\ell_0$ 59).

Table \ref{tab:bio_cyber} shows the \textbf{semantic alignment} of these zero-shot features with the targeted knowledge.  Listing the top 20 SAE feature IDs for both domains, alongside Neuronpedia interpretations, the table shows features for ``Biology" consistently described with terms like ``biological processes", ``cellular functions", and ``genetics"—core concepts of biosecurity risks.  Similarly, ``Cybersecurity" features are linked to ``cyber threats", ``digital security", and ``encryption," reflecting cybersecurity risks in WMDP-Cyber. This highlights SAEs' ability to extract topically precise features, even without task-specific data.

Figure \ref{fig:interpretability} further illustrates this, visualizing activations on WMDP-Bio and WMDP-Cyber forget set sequences.  Figure \ref{fig:interpretability}A (WMDP-Bio) shows activations for IDs 373 and 10933 clustering around biological terms like ``bacteria", ``cellular", and ``infection" while Figure \ref{fig:interpretability}B (WMDP-Cyber, IDs 15286 and 2905) shows clusters around cybersecurity terms like ``encryption", ``data", and "security."

These examples and Table \ref{tab:bio_cyber} show that \textbf{zero-shot SAE feature selection captures semantically rich, domain-relevant concepts associated with hazardous knowledge}. This interpretability is \textbf{prescriptive for unlearning}: by targeting these topically coherent features, DSG achieves \textbf{zero-shot interpretable unlearning}. This is a key practical advantage over gradient-based methods, which require task-specific data and lack inherent interpretability, making DSG a uniquely transparent and data-efficient solution for mitigating hazardous knowledge, especially in data-scarce or zero-shot deployment.

\begin{figure}[htbp]
    \centering
    \includegraphics[width=\textwidth]{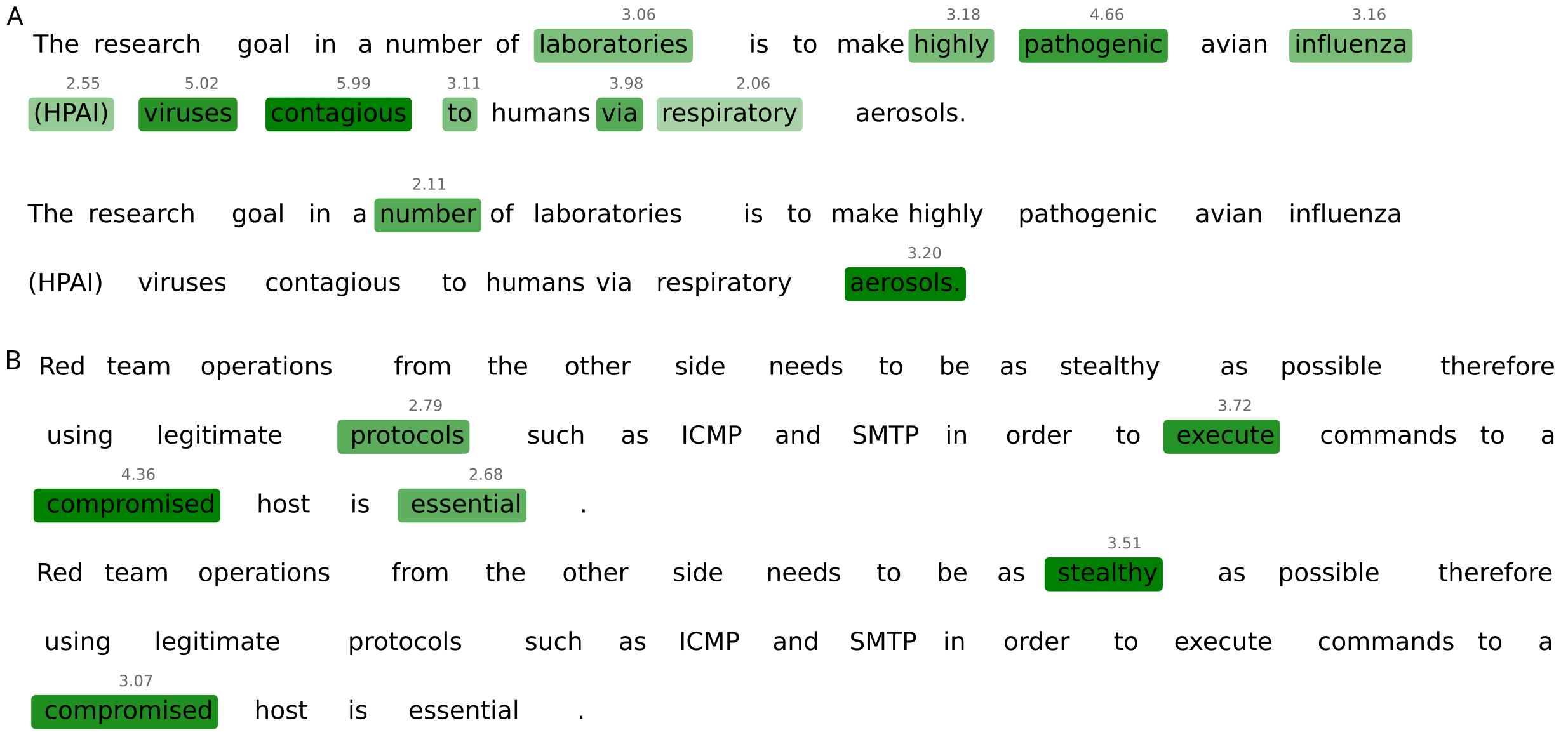} %
 \caption{ \small \textbf{\looseness=-1 Feature Activations on Example Sequences from Forget Sets.} (A) WMDP-Bio sequence with words highlighted in green indicating activation values $>0$ for feature ID (top) 373 and  (bottom) 10933. (B) WMDP-Cyber sequence with words highlighted in green indicating activation values $>0$ for feature ID (top) 15286 and  (bottom) 2905. Activation magnitudes are reported above the words in grey.}
    \label{fig:interpretability}
\end{figure}

\begin{table}[h]
    \centering
    \small
    \begin{tabular}{c|p{12cm}}

        \multicolumn{2}{c}{\textbf{Biology}} \\
        \hline
        \textbf{ID} & \textbf{Sentence} \\
        \hline
        12382 & Terms related to biological processes and structures in living organisms \\
        9722  & Concepts related to biological processes and systems \\
        343   & Terms related to biological processes and laboratory techniques \\
        373   & Scientific terminology related to biological processes and cellular functions \\
        11    & Scientific terms and concepts related to biology \\
        15969 & Terms related to biotechnology and bio-related fields \\
        12117 & Concepts related to biological or cellular processes and conditions, particularly focusing on requirements, limitations, and energy dynamics \\
        5877  & Terms related to biological processes and molecular interactions \\
        968   & Terms related to biological or medical processes and conditions, especially those involving cellular or molecular biology \\
        622   & Scientific terminology related to cellular processes and functions \\
        5231  & Specific terminology related to biological processes and gene expression \\
        10546 & Biological and genetic terms or sequences \\
        12037 & Medical terms and technical jargon related to genetic and biological research \\
        6150  & Elements related to scientific terminology, particularly in genetics and molecular biology \\
        5704  & Scientific terms and jargon related to biological research \\
        14747 & Technical terminology and references related to biotechnology and medical research \\
        8786  & Scientific terminology related to molecular biology and laboratory procedures \\
        10933 & Terms related to biological research and medical methodologies \\
        140   & Technical terms and concepts related to biology and bioengineering \\
        13527 & Terms related to biological or medical research, particularly focusing on specific conditions and associated microorganisms \\
        \hline
        \multicolumn{2}{c}{\textbf{Cybersecurity}} \\
        \hline
        \textbf{ID} & \textbf{Sentence} \\
        \hline
        15331  & Terms related to cyber threats and cybersecurity issues \\
        2060   & Explicit mentions of digital security concerns \\
        15286  & Concepts and terms related to digital security and data integrity \\
        11015  & Terms related to security and the act of securing something \\
        364    & References to security and related terms \\
        4836   & Concepts related to secure web connections and cryptocurrency surplus \\
        2905   & Terms related to data security and encryption \\
        10931  & References to national security and related governmental positions or actions \\
        11716  & Technical terms and language related to coding and software functionality, specifically focusing on vulnerabilities \\
        16160  & Discussions related to technology and computer systems \\
        6309   & References to technology and its applications across various sectors \\
        10543  & Keywords related to safety and security measures in various contexts \\
        11513  & Terms related to computing and data centers \\
        1803   & References to Common Weakness Enumeration (CWE) identifiers \\
        12681  & Keywords related to safety and security \\
        11520  & References to information technology and IT-related concepts \\
        11323  & Key concepts related to digital citizenship and its implications in various contexts \\
        10415  & Key components of data processing and communication in systems, particularly focusing on the details of data packet headers and their significance for routing and interpreting data \\
        3943   & References to computing systems and technologies \\
        4686   & References to technology and tech-related topics \\
        \hline
    \end{tabular}
 \caption{\small \textbf{Top 20 SAE Features for Biology and Cybersecurity in Zero-Shot Setting.}  List of the top 20 SAE feature IDs identified by querying Neuronpedia with ``Biology" and ``Cybersecurity", alongside their corresponding Neuronpedia-provided interpretations, showing the semantic relevance of the selected features to the targeted knowledge domains.}
    \label{tab:bio_cyber}
\end{table}

\end{document}